\documentclass[12pt]{article}

\usepackage{arxiv}
\usepackage{xargs}
\usepackage{xcolor}
\usepackage[utf8]{inputenc}
\usepackage[T1]{fontenc}    
\usepackage{hyperref}       
\usepackage{url}            
\usepackage{booktabs}       
\usepackage{amsfonts}       
\usepackage{nicefrac}       
\usepackage{microtype}      
\usepackage{xcolor}         
\usepackage{times}
\usepackage{lipsum}
\usepackage{graphicx}
\usepackage{mathrsfs}
\usepackage{xargs}
\usepackage{enumitem}
\usepackage{aliascnt}
\usepackage{appendix}
\usepackage[english]{babel}
\usepackage{amsmath}
\usepackage{amsthm}
\usepackage{amssymb}
\usepackage{caption}
\usepackage{cancel}
\usepackage{array}
\usepackage{framed}
\usepackage{multirow}
\usepackage{bbm}
\usepackage[ruled,vlined]{algorithm2e}

\usepackage[english]{babel}
\usepackage[sorting=nty,style=authoryear-comp,natbib=true,maxbibnames=99,maxcitenames=2,uniquelist=false,giveninits]{biblatex}

\usepackage{hyperref}
\hypersetup{colorlinks = true}
\hypersetup{linkcolor=blue}
\usepackage{cleveref}

\crefname{subsubsubappendix}{appendix}{appendices}
\Crefname{subsubsubappendix}{Appendix}{Appendices}

\newtheoremstyle{styledef}
  {6pt}
  {6pt}
  {\itshape}
  {0em}
  {\bfseries}
  {}
  {1.5em}
  {}

\theoremstyle{styledef}

\newenvironment{enumerateList}{ \begin{enumerate}[label=(\roman*),wide=0pt, labelindent=\parindent]}{\end{enumerate}}

\newenvironment{hypA}[1]{
  \begin{enumerate}[label=\textbf{\sf(A#1)}]\begin{sf}}
  {\end{sf}\end{enumerate}}

\usepackage[textwidth=70pt]{todonotes}

\newtheorem{thm}{Theorem}
\crefname{thm}{theorem}{theorems}
\Crefname{thm}{Theorem}{Theorems}
\newtheorem{ex}{Example}
\crefname{ex}{example}{examples}
\Crefname{ex}{Example}{Examples}
\newtheorem{lem}{Lemma}
\crefname{lem}{Lemma}{lemmas}
\Crefname{lem}{Lemma}{Lemmas}
\newtheorem{prop}{Proposition}
\crefname{prop}{Proposition}{Propositions}
\Crefname{Prop}{Proposition}{Propositions}

\crefname{coro}{corollary}{corollaries}
\Crefname{Coro}{Corollary}{Corollaries}

\crefname{defi}{definition}{definitions}
\Crefname{Def}{Definition}{Definitions}
\newtheorem{rem}{Remark}
\crefname{rem}{remark}{remarks}
\crefname{appendix}{Appendix}{Appendices}
\Crefname{appendix}{Appendix}{Appendices}\Crefname{rem}{Remark}{Remarks}

\newenvironment{hyp}[1]{
\begin{enumerate}[label=\textbf{\sf(#1\arabic*)},resume=hyp#1]\begin{sf}}
{\end{sf}\end{enumerate}}
\crefname{hyp}{}{ass}
\Crefname{hyp}{}{Ass}

\newenvironment{hypAGrad}[1]{
\begin{enumerate}[align=left,label=\textbf{\sf($\tilde{\mathrm{\sf{A}}}_{h'}$)}]\begin{sf}}
{\end{sf}\end{enumerate}}

\renewenvironment{leftbar}[2][\hsize]
{
    \def\FrameCommand
    {
        {\color{#2}\vrule width 3pt}
        \hspace{0pt}
    }
    \MakeFramed{\hsize#1\advance\hsize-\width\FrameRestore}
}
{\endMakeFramed}

\newcommandx{\DAR}[4][1=\alpha,2=\theta,3=\phi,4=x]{D^{(#1)}(#2,#3;#4)}
\newcommandx{\DG}[1][1=\gamma]{D^{(#1)}_{\mathrm{G}}}
\newcommandx{\DGtwo}[2][1=\gamma_1, 2=\gamma_2]{D^{(#1, #2)}_{\mathrm{G}}}
\newcommandx{\DsAB}[2][1=\alpha, 2=\beta]{D^{(#1, #2)}_{\mathrm{sAB}}}

\newcommandx{\barell}[2][1=\theta, 2=\phi]{\overline{\ell}_{#1, #2}}
\newcommandx{\repell}[3][1=\theta, 2=\phi, 3=\phi]{\Lambda_{#1,#2,#3}}
\newcommandx{\repelltwo}[3][1=\theta, 2=\phi, 3=\phi]{\Lambda_{#1,#2,#3}}
\newcommandx{\repellstandard}[3][1=\theta, 2=\phi,3=\phi]{\widetilde{\Lambda}_{#1,#2, #3}}
\newcommandx{\repbarell}[2][1=\theta, 2=\phi]{\overline{\Lambda}_{#1,#2}}

\newcommand{\eqdef}{:=}
\newcommand{\eqREP}{\overset{\mathrm{(REP)}}{=}}
\newcommandx{\gammaA}[1][1=\alpha]{\gamma^{(#1)}}
\newcommandx{\gammaREP}[2][1=\alpha, 2=\psi]{#2\text{-}\mathrm{G}_2^{({#1}, \mathrm{REP})}(\theta, \phi; x)}
\newcommandx{\gammaDREP}[2][1=\alpha, 2=\psi]{#2\text{-}\mathrm{G}_2^{({#1}, \mathrm{DREP})}(\theta, \phi; x)} 
\newcommandx{\VRREP}[2][1=\alpha, 2=\psi]{#2\text{-}\mathrm{G}_1^{({#1}, \mathrm{REP})}(\theta, \phi; x)}
\newcommandx{\VRDREP}[2][1=\alpha, 2=\psi]{#2\text{-}\mathrm{G}_1^{({#1}, \mathrm{DREP})}(\theta, \phi; x)} 

\newcommandx{\gradNOREP}[3][1={N},2=\psi,3=\alpha]{#2\text{-}g_{#1}^{(#3, \mathrm{NAIVE})}(\phi;x)}
\newcommandx{\gradIDEAL}[3][1={N},2=\psi,3=\alpha]{#2\text{-}g_{#1}^{(\mathrm{INTER})}(\phi;x)}
\newcommandx{\gradVIMCO}[3][1={N},2=\psi,3=\alpha]{#2\text{-}g_{#1}^{(#3, \mathrm{VIMCO})}(\phi;x)}
\newcommandx{\gradVIMCOam}[3][1={N},2=\psi,3=\alpha]{#2\text{-}g_{#1}^{(\mathrm{VIMCO\text{-}AM})}(\phi;x)}
\newcommandx{\gradVIMCOgm}[3][1={N},2=\psi,3=\alpha]{#2\text{-}g_{#1}^{(\mathrm{VIMCO\text{-}GM})}(\phi;x)}
\newcommandx{\gradVIMCOzvam}[3][1={N},2=\psi,3=\alpha]{#2\text{-}g_{#1}^{(\mathrm{#3, VIMCO\text{-}zvAM})}(\phi;x)}
\newcommandx{\gradVIMCOstar}[3][1={N},2=\psi,3=\alpha]{#2\text{-}g_{#1}^{(#3, \mathrm{VIMCO\text{-zv}\star})}(\phi;x)}
\newcommandx{\gradVIMCOeff}[3][1={N},2=\psi,3=\alpha]{#2\text{-}g_{#1}^{(#3, \mathrm{VIMCO\text{-}0})}(\phi;x)}

\newcommandx{\gradREPls}[3][1={N},2=\psi,3=\alpha]{#2\text{-}g_{#1}^{(#3, \mathrm{REP\text{-}LS})}(\phi;x)}
\newcommandx{\gradREPg}[4][1={N},2=\psi,3=\alpha,4=\beta]{#2\text{-}g_{#1}^{(#3,#4, \mathrm{REP})}(\phi;x)}
\newcommandx{\gradDREP}[3][1={N},2=\psi,3=\alpha]{#2\text{-}g_{#1}^{(#3,\mathrm{DREP})}(\phi;x)}
\newcommandx{\vREP}[2][1=\alpha, 2=\psi]{#2\text{-}V^{(#1, \mathrm{REP})}}
\newcommandx{\vVIMCOam}[2][1=\alpha, 2=\psi]{#2\text{-}V^{(#1, \mathrm{VIMCO\text{-}AM})}}
\newcommandx{\vREINF}[2][1=\alpha, 2=\psi]{#2\text{-}V^{(#1, \mathrm{VIMCO\text{-}\star})}}
\newcommandx{\vVIMCOgen}[2][1=\alpha, 2=\psi]{#2\text{-}V_\eta^{(#1, \mathrm{VIMCO})}}
\newcommandx{\vVIMCOgm}[2][1=\alpha, 2=\psi]{#2\text{-}V^{(#1, \mathrm{VIMCO\text{-}GM})}}
\newcommandx{\vVIMCOstar}[2][1=\alpha, 2=\psi]{#2\text{-}\tilde{V}^{(#1, \mathrm{VIMCO})}}
\newcommandx{\vVIMCOgzvAM}[2][1=\alpha, 2=\psi]{#2\text{-}\tilde{V}^{(#1, \mathrm{VIMCO}\text{-}\mathrm{zvAM})}}
\newcommandx{\vVIMCOgzvstar}[2][1=\alpha, 2=\psi]{#2\text{-}\tilde{V}^{(#1, \mathrm{VIMCO}\text{-}\mathrm{zv}\star)}}

\newcommandx{\genvDREP}[2][1=\alpha, 2=\psi]{#2\text{-}V^{(#1, \mathrm{DREP})}}
\newcommandx{\vDREPun}[2][1=\alpha, 2=\psi]{\genvDREP[#1][#2]_1}
\newcommandx{\vDREPdeux}[2][1=\alpha, 2=\psi]{\genvDREP[#1][#2]_2}

\newcommandx{\lbd}[2][1=\theta,2=\phi]{\lambda_{#1, #2}^{(\alpha)}}
\newcommand{\lrav}[1]{\left|#1 \right|}
\newcommand{\lrn}[1]{\left\|#1 \right\|}
\newcommand{\lr}[1]{\left(#1 \right)}
\newcommand{\lrb}[1]{\left[#1 \right]}

\newcommand{\lrcb}[1]{\left\{#1 \right\}}

\newcommandx{\argcond}[2]{\lrb{\left.#1 \,\right|\, #2}}

\newcommandx{\liwae}[1][1=N]{\ell^{(\mathrm{IWAE})}_{#1}}
\newcommandx{\liren}[2][1=\alpha, 2=N]{\ell^{(#1)}_{#2}}
\newcommandx{\lirenb}[3][1=\alpha, 2=N, 3=\beta]{\ell^{(#1, #3)}_{#2}}
\newcommandx{\lirenL}[3][1=\alpha, 2=N, 3=L]{\ell^{(#1)}_{#2, #3}}
\newcommandx{\lirenBB}[2][1=\alpha, 2=N]{\tilde{\ell}^{(#1)}_{#2}}

\newcommand{\PE}{\mathbb E}
\newcommand{\PP}{\mathbb P}

\newcommandx{\renorm}[2][1=\ell, 2=\alpha]{\varphi_{N, #1}^{(#2)}}
\newcommand{\rmd}{\mathrm d}
\newcommand{\rme}{\mathrm e}
\newcommand{\rset}{\mathbb{R}}

\newcommandx{\RalphaN}[1][1=\alpha]{R_{#1,N}}
\newcommandx{\RalphaNb}[2][1=\alpha,2=N]{R_{#1,#2}}
\newcommandx{\Rtalpha}[2][1=\alpha, 2=\phi]{\overline{w}_{\theta, #2}^{(#1)}}

\newcommandx{\sumW}[2][1=j,2=\beta]{\overline{W}_{#1,N}(\beta)}

\newcommandx{\tw}[4][1=\phi, 2=\varepsilon,3=\phi',4=\theta]{\tilde{w}_{#4, #1, #3}(#2; x)}

\newcommandx{\tvDREP}[1][1=\alpha]{\tilde{v}^{(#1, \mathrm{DREP})}_{\theta, \phi}}
\newcommandx{\w}[2][1=\theta, 2=\phi]{w_{#2}}
\newcommandx{\normW}[3][1=\theta, 2=\phi,3=1-\alpha]{W_{#2}^{(#3)}}
\newcommandx{\barw}[2][1=\theta, 2=\phi]{\bar{w}_{#2}}
\newcommandx{\wAB}[2][1=\theta, 2=\phi]{w^{(\alpha, \beta)}_{#2}}
\newcommandx{\wb}[2][1=\theta, 2=\Phi]{w_{\boldsymbol{#2}}}

\newcommandx{\Zalpha}[1][1=\phi']{Z_{\alpha}(#1)}
\newcommandx{\ZalphaN}[1][1=\phi']{\hat{Z}_{N,\alpha}(#1)}

\newcommand{\weaklimit}{\overset{\mathrm{dist.}}{\longrightarrow}}

\begin{document}
\title{Control Variate Design for Score-Function Gradient Estimators in Importance-Weighted Variational Inference}

\author{Kam\'elia Daudel \\
ESSEC Business School \\ daudel@essec.edu  \And
Minh-Ngoc Tran \\ 
The University of Sydney Business School \\ minh-ngoc.tran@sydney.edu.au \And Cheng Zhang \\ Peking University \\ chengzhang@math.pku.edu.cn
}
\maketitle
{\small{}{}{}}{\small\par}

\begin{abstract}
Importance-weighted variational inference approximates intractable densities by optimizing multi-sample Monte Carlo objectives whose tightness increases with the number of importance samples $N$. When reparameterization is unavailable, optimization relies on score-function (REINFORCE) gradient estimators. In this paper, we provide a comprehensive analysis of score-function gradient estimators for importance-weighted objectives by characterizing their asymptotic expectation and variance. We establish that standard variational inference for Monte Carlo objectives (VIMCO) gradient estimators suffer from a vanishing signal-to-noise ratio (SNR) as $N$ increases, which arises from a suboptimal control variate choice and can impair optimization. To address this issue, we derive a general mechanism for constructing VIMCO gradient estimators that avoid SNR collapse. We identify the High-Effective Sample Size optimal variance importance sampling estimator as belonging to this category of estimators and establish its $\sqrt{N}$ SNR scaling. This estimator does not, however, have zero variance at optimality and achieving its $\sqrt{N}$ SNR rate may require prohibitively large $N$ when the importance weights are degenerate. We overcome these limitations by introducing VIMCO gradient estimators that attain zero variance at optimality and support importance weight tempering. Experiments with discrete latent-variable and state-space models validate the theoretical predictions and illustrate the benefits of our estimators. \looseness=-1
\end{abstract}

\keywords{REINFORCE gradient estimators \and Signal-to-Noise Ratio \and Importance-Weighted Auto-Encoder \and Sticking-the-Landing gradient estimators}

\setcounter{tocdepth}{2}


\section{Introduction}

Sampling from a density known up to a normalizing constant is a fundamental problem in statistics. Variational inference (VI) provides a computationally attractive approach to approximating such densities by  recasting sampling as an optimization problem \citep{jordan1999,blei2017variational}. More precisely, VI seeks the closest approximation to an unknown target density within a tractable family of probability densities $\mathcal{Q}$ by maximizing a variational bound. Although the most traditional variational bound is the evidence lower bound (ELBO), its optimization often yields variational approximations that concentrate too strongly around posterior modes and underrepresent posterior uncertainty \citep{li2016renyi,margossian2024variationalinferenceuncertaintyquantification,minka2005divergence,Bui2016BlackboxF,margossian2025generalized}.  \looseness=-1

To mitigate these issues, variational bounds based on importance weighting principles and whose tightness relative to the ELBO increases with the number of Monte Carlo samples $N$ have been proposed. A prominent example is the importance-weighted autoencoder (IWAE) bound \citep{burda2015importance}, which was subsequently extended under the name VR-IWAE bound \citep{daudel2022Alpha} to include a parameter $\alpha \in [0,1)$ tied to the Variational Rényi (VR) bound of \cite{li2016renyi}. A common way to optimize such variational bounds is to perform stochastic gradient ascent (SGA) and two main approaches to gradient estimation have emerged: reparameterized \citep{kingma2014autoencoding,pmlr-v32-rezende14} and score-function gradient estimators \citep{williams1992}. \looseness=-1 

Reparameterized gradient estimators rely on the assumption that gradients of expectations can be re-expressed in a convenient form by writing random variables as differentiable transformations of parameter-independent noise. While reparameterized gradient estimators are widely used in importance-weighted VI \citep{burda2015importance,li2016renyi,daudel2022Alpha,Tucker2019DoublyRG} and extensively studied \citep{daudel2024learningimportanceweightedvariational,rainforth2018tighter,dhaka2021challenges}, a major limitation of these estimators is that their construction restricts both the choice of the data-generating process and the variational family. For example, they are not compatible with discrete latent variables and, more broadly, are often not applicable to many practically challenging problems, including likelihood-free inference \citep{wood2010statistical, Ong:STCO2018} and state-space models \citep{Durbin:2001}. \looseness=-1

In contrast, score-function (REINFORCE) gradient estimators do not require the differentiability assumptions inherent to reparameterized gradient estimators. Despite their potential to tackle complex problems beyond the scope of reparameterization-based methods, the adoption and analysis of score-function gradient estimators are considerably less prevalent in importance-weighted VI compared to their reparameterized counterparts. Within this literature, \cite{mnih2016variational} notably introduced variational inference for Monte Carlo objectives (VIMCO), a family of unbiased score-function gradient estimators of the IWAE bound that use per-sample control variates for variance reduction. The subsequent work of \citet{lievinNEURIPS2020} further investigated control-variate design for importance-weighted objectives and introduced the optimal variance importance sampling (OVIS) family of score-function gradient estimators. Despite these methodological advances, formal theoretical guarantees are lacking for the score-function gradient estimators used in practice by both papers. \looseness=-1

In this work, we establish a comprehensive analysis of score-function gradient estimators in importance-weighted VI, yielding new insights into their practical behavior. Our analysis exposes fundamental limitations of standard VIMCO gradient estimators and guides the construction of theoretically sound and computationally efficient alternatives. The remainder of this paper is then organized as follows. 

After reviewing the basics of importance-weighted VI in \Cref{sec:Baground}, we devote \Cref{sec:mainREINF} to the construction and analysis of score-function gradient estimators in this setting. We begin in \Cref{{subsec:asympREINFgrad}} by defining and analyzing the score-function gradient of the VR-IWAE bound as a function of the number of Monte Carlo samples $N$ and of the parameter $\alpha \in [0,1)$. Our theoretical analysis shows that, as $N \to \infty$, this gradient points in a meaningful optimization direction that improves over standard ELBO-based optimization for all $\alpha \in [0,1)$. Its magnitude, however, vanishes when $\alpha=0$ but generally remains nonzero when $\alpha \in (0,1)$. \looseness=-1

The score-function gradient of the VR-IWAE bound admits unbiased estimators which we study in \Cref{subsec:gradientEstimREINF}. We demonstrate the essential role of variance reduction as $N$ grows, which leads us to introduce a general family of VIMCO-type gradient estimators for the VR-IWAE bound parameterized by $N$, $\alpha$ and a flexible control variate function $f_{-i}^{(\alpha)}$. Specializing to $\alpha=0$, we analyze two choices within this family which correspond to the practical gradient estimators of \cite{mnih2016variational}. While both estimators are proven to be effective for variance reduction, we show that they suffer from a critical limitation: their signal-to-noise ratio (SNR) vanishes as $N$ increases. 

We find that this issue stems from the choice of the control variate function. Specifically, we derive a key theorem that characterizes the leading asymptotic variance of VIMCO-type gradient estimators with deterministic control variate functions and identifies the variance-minimizing choice. When $\alpha=0$, the variance-minimizing control variate is exactly zero, and the resulting gradient estimator coincides with the High-Effective Sample Size OVIS (High-ESS OVIS) gradient estimator from \cite{lievinNEURIPS2020}, for which we establish a $\sqrt{N}$ SNR rate scaling. Yet, High-ESS OVIS does not attain zero variance at optimality when the variational approximation is equal to the target density. Moreover, its asymptotic $\sqrt{N}$ SNR rate regime may appear only at computationally prohibitive values of $N$ in practice when the importance weights are degenerate. We overcome these limitations by using our key theorem to guide the construction of novel VIMCO gradient estimators that (i) retain favourable $\sqrt{N}$ SNR scaling properties, (ii) achieve zero gradient variance at optimality and (iii) are amenable to importance weights tempering via the hyperparameter $\alpha\in[0,1)$. \looseness=-1

Finally, Section~\ref{sec:NumEx} presents empirical evidence validating our theory and demonstrating the effectiveness of the proposed methods, while Section~\ref{sec:ccl} outlines directions for future research. Deferred results, proofs, and additional experiments are provided in the appendix. \looseness=-1

\section{Background}
\label{sec:Baground}

Consider a statistical model with joint distribution $p(x, z)$, where $x$ is an observation and $z$ is a latent variable valued in a latent space $\mathsf{Z}$. For many models of interest, the posterior density $p(z|x)$ is only known up to its normalizing constant, rendering direct sampling infeasible. Variational inference (VI) addresses this difficulty by seeking an optimal probability density $q_\phi(z)$ parameterized by $\phi \in \Phi$ that provides a close approximation to $p(z|x)$ while remaining amenable to efficient sampling. This is achieved by  
optimizing a variational bound -- that is, a lower bound on the marginal log likelihood -- over $\phi \in \Phi$, where typically $\phi = (\phi_1, \ldots, \phi_b) \in \Phi \subseteq \rset^b$. While the most traditional variational bound is the Evidence Lower BOund (ELBO) \looseness=-1
\begin{align} \label{eq:ELBO-def}
  &\mathrm{ELBO} (\phi; x) = \PE_{z \sim q_\phi} \lr{ \log \w(z;x)} \quad \mbox{where} \quad \w(z;x) = \frac{p(x,z)}{q_\phi(z)} ~,~ z \in \mathsf{Z},
\end{align}
several extensions of the ELBO leveraging importance sampling ideas have been introduced. Two notable instances of these so-called importance-weighted VI bounds are the importance-weighted auto-encoder (IWAE) bound \citep{burda2015importance} and its generalization the Variational Rényi-IWAE (VR-IWAE) bound \citep{daudel2022Alpha}. The latter is defined as \looseness=-1
\begin{align} \label{eq:lalpha-def} 
  \liren(\phi; x) = \frac{1}{1-\alpha} \PE_{z_j \overset{\mathrm{i.i.d.}}{\sim} q_\phi}\lr{\log \lr{ \frac{1}{N} \sum_{j = 1}^N \w(z_j;x)^{1-\alpha} }}, \quad \alpha \in [0,1), \quad N \in \mathbb{N}^\star
  \end{align}  
and it recovers the ELBO when $N = 1$ or $\alpha \to 1$ as well as the IWAE bound when $\alpha = 0$.

A core feature of importance-weighted VI bounds is that they get tighter as $N$ increases. In particular, the VR-IWAE bound gets tighter as $N$ increases and/or $\alpha$ decreases (provided the support of the posterior density is contained within the support of $q_\phi$): for all $\alpha, \alpha' \in [0,1)$ such that $\alpha \leq \alpha'$ and all $N \in \mathbb{N}^*$, \looseness=-1
\begin{align}
 & \mathrm{ELBO} (\phi; x) \leq  \liren(\phi; x) \leq  \liren[\alpha][N+1](\phi; x) \leq \ell(x), \label{eq:tighten} \\
 & \mathrm{ELBO} (\phi; x) \leq \liren[\alpha'](\phi; x) \leq \liren(\phi; x) \leq \ell(x), \label{eq:tightenTwo}
\end{align}
where $\ell(x)$ denotes the marginal log likelihood. Another characteristic of importance-weighted VI bounds is that they can be straightforwardly optimized via stochastic gradient ascent (SGA) by deriving unbiased gradient estimators under the reparameterization trick assumption \citep{kingma2014autoencoding,pmlr-v32-rezende14,burda2015importance,daudel2022Alpha}. Yet, a main shortcoming of reparameterized gradient estimators is that the reparameterization trick assumption limits the choice of the data-generating process as well as the choice of the variational family. 
The standard approach in such instances is to use score-function gradients \citep{williams1992}, as they do not assume that the reparameterization trick is available. Specifically, 
\cite{mnih2016variational} propose \textit{variational inference for Monte Carlo objectives} (VIMCO) gradient estimators, which are unbiased score-function gradient estimators of the IWAE bound defined as follows  
\begin{align}
 & \gradVIMCO[N][\psi][\mathrm{IWAE}] = \sum_{j = 1}^N \frac{\w(z_j;x)}{\sum_{\ell = 1}^N \w(z_\ell;x)} \partial_\psi \log \w(z_j;x) \label{eq:VIMCOgenIWAE} \\
 & +  \sum_{i=1}^N {\partial_\psi \log q_{\phi}(z_i) \lrb{ \log \lr{ \frac{1}{N} \sum_{j = 1}^N \w(z_j;x)} - \log \lr{\frac{1}{N} \lr{ \sum_{\overset{j=1}{j\neq i}}^N \w(z_j;x) + f_{-i}(z_{1:N},\phi;x)}}}}. \nonumber
\end{align}
Here, $\psi$ is a one-dimensional component of $(\phi_1, \ldots, \phi_b)$ and $f_{-i}$ is a control variate function of $z_{1:N} = (z_1, \ldots, z_N)$, $\phi$ and $x$ that outputs a random variable which does not depend on $z_i$. In practice, \cite{mnih2016variational} consider two choices of $f_{-i}$: the arithmetic mean and the geometric mean, which correspond to taking $f_{-i}(z_{1:N},\phi;x) = (N-1)^{-1} \sum_{{j=1},{j\neq i}}^N \w(z_j;x)$ and $f_{-i}(z_{1:N},\phi;x) = \exp ((N-1)^{-1} \sum_{{j=1},{j\neq i}}^N \log \w(z_j;x))$ respectively. They observe empirically that the geometric mean performs slightly better compared to the arithmetic mean. \cite{lievinNEURIPS2020} subsequently propose the optimal variance importance sampling (OVIS) framework, with their High-Effective Sample Size (ESS) OVIS gradient estimator of the IWAE bound corresponding to $f_{-i}(z_{1:N},\phi;x) = 0$. \looseness=-1

In this paper, we move beyond the heuristic construction of VIMCO gradient estimators of the IWAE bound from \cite{mnih2016variational,lievinNEURIPS2020} to provide a comprehensive theoretical understanding of score-function gradient estimators in importance-weighted VI. More precisely, we propose and analyze a general family of score-function gradient estimators for the VR-IWAE bound, including novel control-variate constructions, whose effectiveness is demonstrated through theoretical analysis and empirical evaluation. Our work not only fills important theoretical gaps left by \cite{mnih2016variational,lievinNEURIPS2020}, but also leads to substantial improvements over the methodology introduced therein. \looseness=-1

\section{Importance-weighted VI without Reparameterizing}
\label{sec:mainREINF}

We begin by providing the analytical expression for the score-function gradient of the VR-IWAE bound. Letting $\alpha \in [0,1)$, $\psi \in \{\phi_1, \ldots, \phi_b \}$, the score-function gradient of the VR-IWAE bound w.r.t. $\psi$ is given under common differentiability assumptions by \looseness=-1
\begin{multline}
\partial_\psi  \liren(\phi; x) = \PE_{z_j \overset{\mathrm{i.i.d.}}{\sim} q_\phi}\lr{\sum_{j = 1}^N \frac{\w(z_j;x)^{1-\alpha}}{\sum_{\ell = 1}^N \w(z_\ell;x)^{1-\alpha}} \partial_\psi \log \w(z_j;x)} \\ + \frac{1}{1-\alpha} \PE_{z_j \overset{\mathrm{i.i.d.}}{\sim} q_\phi} \lr{ \lr{\sum_{i=1}^N \partial_\psi \log q_{\phi}(z_i)} \log \lr{ \frac{1}{N} \sum_{j = 1}^N \w(z_j;x)^{1-\alpha} }}.\label{eq:reinfVRIWAEFull}
\end{multline}
The expression \eqref{eq:reinfVRIWAEFull} naturally gives rise to the construction of unbiased estimators, enabling a direct route to algorithmic implementation. For instance, letting $z_1, \ldots, z_N$ be i.i.d. samples generated from $q_\phi$, an immediate unbiased estimator of $\partial_\psi  \liren(\phi; x)$ is \looseness=-1
\begin{multline}
\gradNOREP = \sum_{j = 1}^N \normW(z_j;x) \partial_\psi \log \w(z_j;x) \\ + \frac{1}{1-\alpha}  \lr{\sum_{i=1}^N \partial_\psi \log q_{\phi}(z_i)} \log \lr{ \frac{1}{N} \sum_{j = 1}^N \w(z_j;x)^{1-\alpha} } \label{eq:NaiveREINFestimFull} 
\end{multline}
where $\normW(z_j;x) = \frac{\w(z_j;x)^{1-\alpha}}{\sum_{\ell = 1}^N \w(z_\ell;x)^{1-\alpha}}$ for all $j=1\ldots N$ . As we will elaborate on in a later subsection, several other key unbiased estimators of $\partial_\psi  \liren(\phi; x)$ can be built based on \eqref{eq:reinfVRIWAEFull}. 

Our first aim in this section is to investigate the behavior of the score-function gradient of the VR-IWAE bound \eqref{eq:reinfVRIWAEFull} as a function of its hyperparameters $(N,\alpha)$ in order to identify its structural properties. We will then turn to the derivation and study of unbiased estimators of this gradient, examining their efficiency and exploring their practical viability.

\subsection{Asymptotics for the Gradient of the VR-IWAE Bound}
\label{subsec:asympREINFgrad}

Let $x$ be a fixed observation, $\alpha \in [0,1)$, $\Phi\subseteq\rset^b$ be an open subset, $\psi$ be a one-dimensional component of $\phi = (\phi_1, \ldots, \phi_b)$ and assume that for all
$\phi\in\Phi$, the assumption below holds: \looseness=-1
\begin{hyp}{A} 
\item \label{hyp:VRIWAEwell-defined}   $\int p(x,z)\nu(\rmd z) <
  \infty$, $\int q_\phi(z)\nu(\rmd z)=1$ and
  $q_\phi(z),p(z|x)>0$ for $\nu$-a.e. $z$.
\end{hyp}
Here $\nu$ is a $\sigma$-finite measure on the space of latent variables $\mathsf{Z}$. Typically, $\nu$ could be the Lebesgue measure on $\mathbb{R}^d$, the Lebesgue measure restricted to a subset of $\mathbb{R}^d$ or the counting measure when $z$ has discrete support. To avoid specifying the dominating measure, we use the notation $\PE$ throughout the paper.  If the probability distribution of the involved random variables has not been specified beforehand, we use a subscript to the symbol $\PE$ to indicate their densities with respect to $\nu$ as we did in \Cref{sec:Baground}. We also assume that the mapping $\phi \mapsto q_\phi(z)$ is differentiable on $\Phi$ for $\nu$-almost every $z$ and that every expectation differentiated with respect to $\phi$ satisfies the corresponding Leibniz integral rule. Finally we let $z \sim q_\phi$ and $z_1,z_2, \ldots$ be i.i.d. copies of $z$. \looseness=-1

Let us thus investigate the asymptotic behavior of $\partial_\psi \liren  (\phi;x)$ as $N \to \infty$ when we do not make the reparameterization trick assumption, with the aim of elucidating the role of its hyperparameters $(N, \alpha)$. To guide our intuition, we appeal to \cite{daudel2024learningimportanceweightedvariational} who show that: as $N \to \infty$, \looseness=-1
  \begin{align} \label{eq:partialVRiWAErep}
   \partial_\psi \liren  (& \phi;x) \eqREP \partial_\psi \mathrm{VR}^{(\alpha)}(\phi; x) - \frac{1}{2 N} \partial_\psi [\gammaA(\phi;x)^2]+o\left(\frac{1}{N}\right) \\
\mbox{with} \quad & \mathrm{VR}^{(\alpha)}(\phi; x) = \frac{1}{1-\alpha} \log \lr{\PE \lr{\w(z; x)^{1-\alpha}}} \label{eq:defVRbound} \\
  & \gammaA(\phi;x)^2 = \frac{1}{1-\alpha}\mathbb{V}\lr{\frac{\w(z;x)^{1-\alpha}}{\PE(\w(z;x)^{1-\alpha})}}, \label{eq:defGammaAlphaFirst}
\end{align} 
where $\eqREP$ indicates that the equality is valid under the reparameterization trick assumption. We argue that \eqref{eq:partialVRiWAErep} should in fact be independent of the reparameterization trick assumption and to establish such a result, we introduce the following assumptions. \looseness=-1 
\begin{hyp}{A} 
  \item \label{hyp:inverseW} There exist $\mu>1$ and $N \in \mathbb{N}^\star$ such that $\PE(|N^{-1} \sum_{i = 1}^N \w(z_i;x)^{1-\alpha}|^{-\mu}) < \infty$.
\end{hyp}
\begin{hypA}{$_{h}$} 
  \item \label{hyp:momentW} We have that $\PE(\w(z;x)^{(1-\alpha)h}) < \infty$.
\end{hypA}
\begin{hypAGrad}{$_{h'}$} 
  \item \label{hyp:momentGradTwo} We have that $\PE\lr{\lrav{\partial_{\psi} \lr{\log q_\phi(z)}}^{h'}} < \infty$.
\end{hypAGrad}
\ref{hyp:inverseW} is a common assumption that already appears in existing works and that is fully discussed in \cite{daudel2024learningimportanceweightedvariational}. \ref{hyp:momentW} and \ref{hyp:momentGradTwo} are moment conditions in which $h$ and $h'$ are positive exponents to be specified. We now present our first result.

\begin{thm} \label{thm:REINFORCEGradientExpectation}  Assume~\ref{hyp:inverseW}, \ref{hyp:momentW} and \ref{hyp:momentGradTwo} with $h = 4$, $h'> 2$. Then, as $N \to \infty$: \looseness=-1
\begin{align}
\partial_\psi  \liren(\phi; x) = \partial_\psi \mathrm{VR}^{(\alpha)} (\phi;x) +o(1). \label{eq:REINFLeadingOrder}
\end{align}
Further assuming that $h = 8$ and $h' = 4$, we have: as $N\to\infty$,
\begin{align}
  \partial_\psi  \liren(\phi; x) = \partial_\psi \mathrm{VR}^{(\alpha)} (\phi;x) - \frac{1}{2N} \partial_\psi [\gammaA(\phi;x)^2] + o \lr{\frac{1}{N}}. \label{eq:REINFfirstOrder}
\end{align}
\end{thm}
The proof of \Cref{thm:REINFORCEGradientExpectation} is deferred to \Cref{app:subsec:thm:REINFORCEGradientExpectation}. \Cref{thm:REINFORCEGradientExpectation} characterizes the behavior of the score-function gradient of the VR-IWAE bound $\liren(\phi;x)$ as $N \to \infty$. It is the first asymptotic result in importance-weighted VI which bypasses the reparameterization trick assumption and formally studies score-function gradients taken with respect to $\psi \in \{ \phi_1,\ldots, \phi_b \}$ as $N \to \infty$.  As anticipated, the asymptotic behavior described in \eqref{eq:partialVRiWAErep} is independent of the reparameterization trick assumption, although the proof technique differs substantially depending on whether this assumption is made or not (see \Cref{eq:FurtherCommentsThmOne} in the appendix for further comments on the proof details).
For all $\psi \in \{ \phi_1,\ldots, \phi_b \}$, we have thus obtained that the score-function gradient of the VR-IWAE bound converges to $\partial_\psi \mathrm{VR}^{(\alpha)} (\phi;x)$ at a fast $1/N$ rate, where the quantity $\mathrm{VR}^{(\alpha)}(\phi; x)$ is exactly the marginal log likelihood $\ell(x)$ when $\alpha = 0$ and corresponds otherwise to the Variational-Rényi (VR) bound \citep{li2016renyi}. We now discuss the methodological implications of \Cref{thm:REINFORCEGradientExpectation}. \looseness=-1

\Cref{thm:REINFORCEGradientExpectation} predicts that unbiased score-function gradient estimators of the VR-IWAE bound based on \eqref{eq:reinfVRIWAEFull} can be used to maximize the VR bound w.r.t. $\phi$ when $\alpha \in (0,1)$. This in turn is expected to lead to better variational approximations of the posterior density than those based on the ELBO \citep{li2016renyi,margossian2024variationalinferenceuncertaintyquantification,minka2005divergence,Bui2016BlackboxF,margossian2025generalized}.  Furthermore, $\partial_\psi \liren[\alpha](\phi; x)$ points in the direction which minimizes $\gammaA[0](\phi;x)^2 = \mathbb{V} \lr{ {p(z|x)}/{q_\phi(z)}}$ when $\alpha = 0$ (as $\partial_\psi \mathrm{VR}^{(0)} (\phi;x) = \partial_\psi \ell(x) = 0$ for all $\psi \in \{ \phi_1, \ldots, \phi_b \}$): optimizing $\phi$ to minimize $\mathbb{V} \lr{ {p(z|x)}/{q_\phi(z)}}$ inherently penalizes under-dispersion, driving $q_\phi$ to capture regions of posterior probability mass that the standard ELBO typically ignores \citep{tan2020conditionally, Owenbook2013, rainforth2018tighter}. \looseness=-1

Hence, \Cref{thm:REINFORCEGradientExpectation} motivates running SGA schemes based on \eqref{eq:reinfVRIWAEFull} to maximize the VR bound w.r.t. $\phi$ when $\alpha \in (0,1)$ and to minimize $\mathbb{V} \lr{ {p(z|x)}/{q_\phi(z)}}$ w.r.t $\phi$ when $\alpha = 0$. This approach, which is expected to improve on ELBO maximization, is further motivated by the fact that we cannot readily propose unbiased score-function gradient estimators of the VR bound (see, e.g., \eqref{eq:GradVR} in \Cref{app:subsec:thm:REINFORCEGradientExpectation}) nor of $\gammaA[0](\phi;x)^2$ to learn $\phi$. Consequently, our next step is to establish a principled framework for designing unbiased score-function gradient estimators of the VR-IWAE bound.

\subsection{Estimating the Gradient of the VR-IWAE Bound}
\label{subsec:gradientEstimREINF}
 
The first estimator we consider is $\gradNOREP$ (NAIVE($\alpha$)) as defined in \eqref{eq:NaiveREINFestimFull}. We begin with a lemma focusing on the variance of the first term in the right-hand side of \eqref{eq:NaiveREINFestimFull}. \looseness=-1
\begin{lem} \label{prop:VarREINFfirstV}  Assume~\ref{hyp:inverseW}, \ref{hyp:momentW} and \ref{hyp:momentGradTwo} with $h,h'>4$. Then, as $N \to \infty$, \looseness=-1
\begin{multline*}
  \mathbb{V}\lr{\sum_{j = 1}^N \normW(z_j;x) \partial_\psi \log \w(z_j;x)} \\
   = \frac{1}{N} \mathbb{V}\lr{\frac{\w(z;x)^{1-\alpha}}{\PE(\w(z;x)^{1-\alpha})} \lrb{ \partial_\psi \log q_\phi(z)- \PE\lr{\frac{\w(z;x)^{1-\alpha}}{\PE(\w(z;x)^{1-\alpha})}\partial_\psi \log q_\phi(z)}}} + o \lr{\frac{1}{N}}.
\end{multline*}
\end{lem}
The proof of \Cref{prop:VarREINFfirstV} is deferred to \Cref{app:subsec:prop:VarREINFfirstV}. We next study the variance of the second term in the right-hand side of \eqref{eq:NaiveREINFestimFull} and use it to deduce the asymptotic behavior of $\mathbb{V}(\gradNOREP)$ as $N \to \infty$ when $\psi \in \{\phi_1, \ldots, \phi_b \}$. 
\begin{prop} \label{prop:VarREINFsecondV}  Assume~\ref{hyp:inverseW}, \ref{hyp:momentW} and \ref{hyp:momentGradTwo} with  $h> 8$, $h'> 4$. Then, as $N \to \infty$, \looseness=-1
\begin{align*}
  & \mathbb{V}\lr{\lr{\sum_{i=1}^N \partial_\psi \log q_{\phi}(z_i)} \log \lr{ \frac{1}{N} \sum_{j = 1}^N {\w(z_j;x)^{1-\alpha}}} } = O(N) \\ 
& \mbox{and} \quad \mathbb{V}(\gradNOREP) = N \mathbb{V}\lr{\partial_\psi \log q_\phi(z)}\lr{\mathrm{VR}^{(\alpha)} (\phi;x)}^2 +o(N). 
\end{align*}
\end{prop}
The proof of \Cref{prop:VarREINFsecondV} is deferred to \Cref{app:subsec:prop:VarREINFsecondV}. 
\Cref{prop:VarREINFsecondV} shows that the variance of NAIVE($\alpha$) scales linearly with $N$ unless $\mathbb{V}\lr{\partial_\psi \log q_\phi(z)} \mathrm{VR}^{(\alpha)} (\phi;x) = 0$, which is a highly undesirable property for an importance-weighted VI gradient estimator. \looseness=-1

Nevertheless, we know from \Cref{prop:VarREINFfirstV} and \Cref{prop:VarREINFsecondV} that the second term in the right-hand side of \eqref{eq:NaiveREINFestimFull} is the one responsible for the large variance of the NAIVE($\alpha$) gradient estimator. This result suggests appealing to variance reduction schemes for the second term of this estimator in order to improve its asymptotic variance. Motivated by the work of \cite{mnih2016variational}, we consider a variance-reduction strategy that adds per-sample baselines to the NAIVE($\alpha$) gradient estimator \eqref{eq:NaiveREINFestimFull}. 
 Namely, for all $\alpha \in [0,1)$ and all $\psi  \in \{\phi_1, \ldots, \phi_b \}$, we define the VIMCO($\alpha$) gradient estimator of the VR-IWAE bound w.r.t. $\psi$ as \looseness=-1
\begin{multline}
  \gradVIMCO = \sum_{j = 1}^N \normW(z_j;x) \partial_\psi \log  \w(z_j;x) \\
  -\frac{1}{1-\alpha}  \sum_{i=1}^N {\partial_\psi \log q_{\phi}(z_i)  \log \lr{1 - \normW(z_i;x) + \frac{f_{-i}^{(\alpha)}(z_{1:N},\phi;x)}{\sum_{j = 1}^N \w(z_j;x)^{1-\alpha}}}}, \label{eq:VIMCOgenVRIWAE}
\end{multline}
where $f_{-i}^{(\alpha)}$ denotes a nonnegative function of $z_{1:N}$, $\phi$ and $x$ that outputs a random variable which does not depend on $z_i$ for all $i = 1 \ldots N$. 
The estimator \eqref{eq:VIMCOgenVRIWAE} recovers \eqref{eq:VIMCOgenIWAE} when $\alpha = 0$ and it is an unbiased score-function gradient estimator of the VR-IWAE bound w.r.t. $\psi \in \{\phi_1, \ldots, \phi_b \}$. We now study its asymptotic variance for key choices of $f_{-i}^{(\alpha)}(z_{1:N},\phi;x)$, starting with the ones proposed in \cite{mnih2016variational}. \looseness=-1

\subsubsection{Arithmetic mean (AM) and geometric mean (GM) gradient estimators}

Set $\alpha = 0$ and let $\gradVIMCOam$ and $\gradVIMCOgm$ denote the gradient estimators which correspond to taking $f_{-i}^{(0)}(z_{1:N},\phi;x) = (N-1)^{-1} \sum_{{j=1},{j\neq i}}^N \w(z_j;x)$ and $f_{-i}^{(0)}(z_{1:N},\phi;x) = \exp ((N-1)^{-1} \sum_{{j=1},{j\neq i}}^N \log \w(z_j;x))$ respectively for all $i = 1 \ldots N$ in \eqref{eq:VIMCOgenVRIWAE}. These estimators are the ones implemented in \cite{mnih2016variational} and the theorem below captures their asymptotic variance as $N \to \infty$ when $\psi \in \{ \phi_1, \ldots, \phi_b \}$. \looseness=-1

\begin{thm} \label{thm:VarianceReduction}  Assume~\ref{hyp:inverseW}, \ref{hyp:momentW} and \ref{hyp:momentGradTwo} with $h > 8$, $h' > 4$. Then, as $N \to \infty$, \looseness=-1
\begin{align} \label{eq:VarAM}
 & \mathbb{V}\lr{\gradVIMCOam} = \frac{1}{N} \mathbb{V}\lr{\partial_\psi \log q_\phi(z)} + o \lr{\frac{1}{N}}. 
\end{align}  
Further assume 
that $h' > 8$, $\PE(|\log \w(z;x)|^{h''}) < \infty$ and $\PE(|\log \w(z;x) \partial_\psi \log q_\phi(z)|^{h'''}) < \infty$ with $h''>4$ and $h'''>8$. Then, as $N \to \infty$,
\begin{align} \label{eq:VarGM}
   \mathbb{V} \lr{\gradVIMCOgm} = \frac{1}{N} \left(\frac{\exp\lrb{\PE\lr{\log \w(z;x)}}}{\PE(\w(z;x))}\right)^2 \mathbb{V}\lr{\partial_\psi \log q_\phi(z)} + o \lr{\frac{1}{N}}. 
\end{align}
\end{thm}
The proof of \Cref{thm:VarianceReduction} is deferred to \Cref{app:subsec:thm:VarianceReduction}. \Cref{thm:VarianceReduction} shows that introducing per-sample baselines in the NAIVE($0$) gradient estimator \eqref{eq:NaiveREINFestimFull} as done in $\gradVIMCOam$ and $\gradVIMCOgm$ yields a variance reduction phenomenon. Furthermore, Jensen's inequality implies that the asymptotic variance of VIMCO-GM is lower than that of VIMCO-AM, which provides a theoretical explanation for the empirical advantage of VIMCO-GM observed by \citet{mnih2016variational}. \looseness=-1

Crucially, casting the results we obtained in terms of signal-to-noise ratio (SNR) yields a much deeper understanding of score-function gradient estimators with per-sample baselines that goes beyond the considerations made in \cite{mnih2016variational}. For a random variable $X$ the SNR is given by $\mathrm{SNR}[X] = |\PE(X)|/\sqrt{\mathbb{V}(X)}$. SNRs are used to evaluate whether the mean dominates the standard deviation: an unbiased gradient estimator is expected to be accurate if the SNR is high (so that the target dominates the additive stochastic error) and noisy otherwise. Then, \Cref{thm:REINFORCEGradientExpectation,thm:VarianceReduction} show that the use of VIMCO-AM and VIMCO-GM can be detrimental to the process of learning $\phi$ by degrading the SNR of these gradient estimators as $N$ increases: for all $\psi \in \{\phi_1, \ldots, \phi_b \}$ and as $N \to \infty$, \looseness=-1
\begin{align}
&   \mathrm{SNR}[\gradVIMCOam] =  \Theta \lr{\frac{1}{\sqrt{N}}}
\label{eq:SNRrateAMandGM} \\
& \mathrm{SNR}[\gradVIMCOgm] =  \Theta \lr{\frac{1}{\sqrt{N}}}, \nonumber
\end{align}
when $\partial_\psi [\gammaA[0](\phi;x)^2] \neq 0$ and $\mathbb{V}\lr{\partial_\psi \log q_\phi(z)} > 0$. Since VIMCO-AM and VIMCO-GM are but two options within the family of score-function gradient estimators of the VR-IWAE bound that appeal to per-sample baselines \eqref{eq:VIMCOgenVRIWAE}, the negative SNR results above warrant investigating alternatives to VIMCO-AM and VIMCO-GM in that family. \looseness=-1

\subsubsection{Beyond the VIMCO-AM and VIMCO-GM gradient estimators}
\label{subsec:beyondVIMCOamgm}

Both the AM and GM per-sample baselines converge almost surely to a deterministic limit as $N\to\infty$. This motivates studying the subclass of VIMCO($\alpha$) estimators with deterministic per-sample baselines and determining which constant minimizes the leading asymptotic variance for a given component $\psi$. The following result characterizes the asymptotic variance throughout this subclass. \looseness=-1

\begin{thm} \label{prop:hintCV}  Assume~\ref{hyp:inverseW}, \ref{hyp:momentW} and \ref{hyp:momentGradTwo} with $h>8$, $h' > 4$. Let $f_{-i}^{(\alpha)}(z_{1:N},\phi;x)$ be deterministic and equal to $\eta \geq 0$ for all $i = 1 \ldots N$ in \eqref{eq:VIMCOgenVRIWAE}. Then, as $N \to \infty$,
\begin{align}
\mathbb{V}\lr{\gradVIMCO} & =  \frac{  \vVIMCOgen(\phi;x)}{N} + o\lr{\frac{1}{N}}, \nonumber \\
\mbox{where} \quad \vVIMCOgen(\phi;x) & = \frac{1}{(1-\alpha)^2} \mathbb{V} \left(\alpha \frac{\w(z;x)^{1-\alpha}}{\PE(\w(z;x)^{1-\alpha})} \biggl[ \partial_\psi \log q_\phi(z)  \right. \nonumber \\ & \left. - \PE\lr{\frac{\w(z;x)^{1-\alpha}}{\PE(\w(z;x)^{1-\alpha})} \partial_\psi \log q_\phi(z)} \biggr] - \frac{\eta \partial_\psi \log q_\phi(z)}{\PE(\w(z;x)^{1-\alpha})} \right). \label{eq:vVIMCOgen}
\end{align}
Letting $\tilde{z}_\alpha \sim \tilde{q}_{\phi,\alpha}$ with $\tilde{q}_{\phi,\alpha} \propto q_\phi \w(\cdot;x)^{1-\alpha}$ and further assuming that $\mathbb{V}(\partial_\psi \log q_\phi(z)) >0$, $\vVIMCOgen(\phi;x)$ is minimal for
\begin{align*}
  \eta = \eta_\psi^{(\alpha, \star)} \eqdef \alpha \PE(\w(z;x)^{1-\alpha}) \frac{\mathbb{V}(\partial_\psi \log q_\phi(\tilde{z}_\alpha))}{\mathbb{V}(\partial_\psi \log q_\phi(z))}.
\end{align*}
\end{thm}
The proof of \Cref{prop:hintCV} is deferred to \Cref{app:prop:hintCV}. Focusing first on the case $\alpha = 0$, we see that the optimal value of $\eta$ in \Cref{prop:hintCV} is given by $\eta_\psi^{(0, \star)} = 0$ and is in general not equal to either $\eta = \PE(\w(z;x))$ or $\eta = \exp [\PE(\log \w(z;x))]$ as in VIMCO-AM and VIMCO-GM. Thus, a straightforward choice of per-sample baseline when $\alpha = 0$ is $f_{-i}^{(0)}(z_{1:N},\phi;x) = 0$. Beyond computational benefits, the corresponding VIMCO($0$) gradient estimator is of primary theoretical interest: as established in the next theorem, its behavior departs from the earlier results we established for VIMCO-AM and VIMCO-GM. \looseness=-1
\begin{thm} \label{thm:optimalVarReduction} Set $\alpha = 0$. Assume~\ref{hyp:inverseW}, \ref{hyp:momentW} and \ref{hyp:momentGradTwo} with $h > 12$, $h' > 4$.  Let $f_{-i}^{(0)}(z_{1:N},\phi;x) = 0$ for all $i = 1 \ldots N$ in \eqref{eq:VIMCOgenVRIWAE}. Then, as $N \to \infty$, \looseness=-1
\begin{align*}
 & \mathbb{V}\lr{\gradVIMCO[N][\psi][0]} = \frac{\vVIMCOstar[0](\phi;x)}{N^3} + o \lr{\frac{1}{N^3}} \\
& \mbox{where} \quad \vVIMCOstar[0] =  \mathbb{V} \left(\frac{ \w(z;x)^{2} \partial_\psi \log q_\phi(z)}{2 \PE(\w(z;x))^2}\right. \left.  - \w(z;x) \frac{\PE(\w(z;x)^{2} \partial_\psi \log q_\phi(z))}{\PE(\w(z;x))^3}     \right).
\end{align*}
\end{thm}
The proof of \Cref{thm:optimalVarReduction} is deferred to \Cref{app:subsec:thm:optimalVarReduction}. \Cref{thm:optimalVarReduction} shows that the variance of VIMCO($0$) when $f_{-i}^{(0)}(z_{1:N},\phi;x) = 0$ decays at a $1/N^3$ convergence rate, so that: as $N \to \infty$, \looseness=-1
\begin{align}
& \mathrm{SNR}[\gradVIMCO[N][\psi][0]] = \sqrt{N} \dfrac{|\frac{1}{2} \partial_\psi [\gammaA[0](\phi;x)^2] + o \lr{1}|}{ \sqrt{\vVIMCOstar[0](\phi;x)} + o(1)} \label{eq:VIMCOstarSNRzero}
\end{align}
with $\partial_\psi [\gammaA[0](\phi;x)^2] \neq 0$. Hence, VIMCO($0$) with $f_{-i}^{(0)}(z_{1:N},\phi;x) = 0$ benefits from increasing $N$ in terms of SNR: its SNR rate scales as $\sqrt{N}$, which is in stark contrast with the $1/\sqrt{N}$ SNR rates for VIMCO-AM and VIMCO-GM. More precisely, it points in expectation in the direction that minimizes $\mathbb{V} \lr{ {p(z|x)}/{q_\phi(z)}}$ w.r.t $\phi$ and resolves the counterintuitive behavior \eqref{eq:SNRrateAMandGM} of the estimators proposed in \cite{mnih2016variational}. 

VIMCO($0$) with $f_{-i}^{(0)}(z_{1:N},\phi;x) = 0$ is exactly the High-Effective Sample Size version of the OVIS gradient estimator (High-ESS OVIS) for the IWAE bound from \cite{lievinNEURIPS2020}. Yet, the construction leading to this gradient estimator in our work differs from the one in \cite{lievinNEURIPS2020}. Indeed, \cite{lievinNEURIPS2020} provided a heuristic asymptotic derivation of a $1/N^3$ variance scaling for an idealized control variate based on a second-order Taylor expansion, before constructing the tractable High-ESS OVIS gradient estimator through additional approximations. \cite{lievinNEURIPS2020} did not formally prove that high-ESS OVIS enjoyed an asymptotic $\sqrt{N}$ SNR scaling, nor did they characterize its asymptotic expectation and variance. \Cref{thm:REINFORCEGradientExpectation,thm:optimalVarReduction} fill this theoretical gap, while \Cref{prop:hintCV} positions high-ESS OVIS as an optimal choice in asymptotic SNR terms within VIMCO gradient estimators with deterministic per-sample baselines when $\alpha = 0$.

So far, we have used \Cref{prop:hintCV} to identify the  $f_{-i}^{(\alpha)}(z_{1:N},\phi;x) = 0$ per-sample baseline as optimal in terms of asymptotic SNR when $\alpha=0$. The scope of the result is, however, considerably broader: it provides a general mechanism for constructing VIMCO gradient estimators across $\alpha\in[0,1)$. This perspective naturally raises the question of whether additional desirable properties can be achieved beyond those offered by high-ESS OVIS. \looseness=-1

\begin{enumerateList}
  \item \textbf{Desirable property 1: Zero variance at posterior optimality ($q_\phi=p(\cdot|x)$).} High-ESS OVIS does not generally have zero variance at posterior optimality: for all $N \geq 2$, \looseness=-1 
\begin{align*}
  & \mathbb{V}\lr{\gradVIMCOam} = \mathbb{V}\lr{\gradVIMCOgm} = \frac{1}{N} \mathbb{V}(\partial_\psi \log q_\phi(z)) 
  \\
  & \mathbb{V}\lr{\psi\text{-}g_N^{(\mathrm{High}\text{-}\mathrm{ESS~OVIS})}} = \frac{1}{N}\lrb{1+ N \log \lr{ 1 - \frac{1}{N}}}^2  \mathbb{V}(\partial_\psi \log q_\phi(z)), 
\end{align*}
when $\mathbb{V}(\partial_\psi \log q_\phi(z)) > 0$. Hence, although the non-asymptotic variance at posterior optimality of High-ESS OVIS improves on VIMCO-AM and VIMCO-GM (provided that $\mathbb{V}(\partial_\psi \log q_\phi(z)) > 0$), it is nonetheless non-zero. From an optimization perspective, the property that the gradient variance vanishes near the optimum is referred to as the interpolation condition. When this condition holds, the convergence rate of SGA is known to
improve significantly \citep{pmlr-v89-vaswani19a}. In the VI literature, the analogous property is known as sticking the landing  \citep{roeder2017}: prior work has sought to achieve this property for ELBO gradient estimators both with the reparameterization trick \citep{roeder2017} and without it \citep{NEURIPS20209c22c0b5}, as well as for VR-IWAE gradient estimators under the reparameterization trick assumption \citep{Tucker2019DoublyRG,daudel2024learningimportanceweightedvariational}.

\item \textbf{Desirable property 2: Predicted asymptotic regimes attained for computationally feasible values of $N$.} The favourable High-ESS OVIS asymptotic regimes originating from \Cref{thm:REINFORCEGradientExpectation,thm:optimalVarReduction} do not imply that said asymptotic regimes are attained for values of $N$ that are computationally feasible in practice. This distinction is especially relevant when the weights suffer from degeneracy, as may occur when the variational approximation is initialized far from the posterior density. Tempering the weight $\w(z;x)$ through $\alpha\in(0,1)$ provides a standard means of mitigating the weight degeneracy by increasing the ESS \citep{finke2019importance,lievinNEURIPS2020}. Our study refines this narrative by characterizing explicitly how $\alpha$ enters the asymptotics of VIMCO gradient estimators and further advocates for the construction of estimators valid for all $\alpha\in[0,1)$. \looseness=-1 
\end{enumerateList}

Guided by these two desirable properties, we seek novel VIMCO($\alpha$) gradient estimators valid for all $\alpha \in [0,1)$ with zero variance at posterior optimality. Starting with the zero variance at posterior optimality part, \eqref{eq:VIMCOgenVRIWAE} implies that a sufficient condition is for the per-sample baseline to satisfy at posterior optimality: \looseness=-1
\begin{align*}
  & \frac{f_{-i}^{(\alpha)}(z_{1:N},\phi;x)}{\PE(\w(z;x)^{1-\alpha})} = c_{N}^{(\mathrm{zv}, \alpha)} \quad \mathrm{a.s.} \\
  & \mbox{where} \quad c_{N}^{(\mathrm{zv}, \alpha)} = N \exp\lr{-\frac{1-\alpha}{N}} - N +1. 
\end{align*}
This brings us to introduce the new per-sample baselines
\begin{align}
 & f_{-i}^{(\alpha)}(z_{1:N},\phi;x) = c_{N}^{(\mathrm{zv}, \alpha)}~\frac{  a_{-i,1,2}^{(\alpha)} (z_{1:N}, \phi;x) - \frac{a_{-i,1,1}^{(\alpha)} (z_{1:N}, \phi;x)^2}{a_{-i,1,0}^{(\alpha)} (z_{1:N}, \phi;x)}}{a_{-i,0,2}^{(\alpha)} (z_{1:N}, \phi;x) - (a_{-i,0,1}^{(\alpha)} (z_{1:N}, \phi;x))^2}, & N \geq 3  \label{eq:fOptimOne} \\
& f_{-i}^{(\alpha)}(z_{1:N},\phi;x)  = c_{N}^{(\mathrm{zv}, \alpha)}  ~\frac{1}{N-1} \sum_{j=1, j \neq i}^N \w(z_j;x)^{1-\alpha}, & N \geq 2 \label{eq:fCheap} 
\end{align}
with $a_{-i,k,\ell}^{(\alpha)} (z_{1:N}, \phi;x) = \frac{1}{N-1} \sum_{j=1, j \neq i}^N (\w(z_j;x)^{1-\alpha})^k (\partial_\psi \log q_\phi(z_j))^\ell$ for $k,\ell \in \{0,1,2\}$. We refer to the resulting gradient estimators as VIMCO-zv$\star$($\alpha$) and VIMCO-zvAM($\alpha$), respectively. By construction, both estimators have zero variance at posterior optimality for all $\alpha\in[0,1)$. Moreover, since $c_{N}^{(\mathrm{zv}, \alpha)}= \alpha + o(1)$ as $N$ increases, VIMCO-zv$\star$($\alpha$) provides a plug-in counterpart of the optimal deterministic baseline $\eta = \eta_\psi^{(\alpha, \star)}$ identified in \Cref{prop:hintCV}. As for VIMCO-zvAM($\alpha$) , it replaces the coordinate-specific ratio of variances in $\eta_\psi^{(\alpha, \star)}$ by one and provides a plug-in counterpart of the deterministic baseline $\eta =\alpha \PE(\w(z;x)^{1-\alpha})$. As such, VIMCO-zvAM($\alpha$) avoids the additional estimation cost required by VIMCO-zv$\star$($\alpha$), while remaining identical to VIMCO-zv$\star$($\alpha$) at posterior optimality.

The choices of $f_{-i}^{(\alpha)}(z_{1:N},\phi;x)$ made in \eqref{eq:fOptimOne} and \eqref{eq:fCheap} can in principle be analyzed using proof techniques similar to those developed in our previous results. The resulting characterization is expected to be consistent with earlier findings, in the sense that we would get: for all $\alpha \in [0,1)$, all $\psi \in \{ \phi_1, \ldots, \phi_b \}$ and as $N \to \infty$, 
\begin{align*}
&\mathbb{V}(\gradVIMCOstar) = \frac{\vVIMCOgen(\phi;x)}{N} + o\lr{\frac{1}{N}}, & \quad \eta = \eta_\psi^{(\alpha, \star)} 
 \\
&\mathbb{V}\lr{\gradVIMCOzvam} = \frac{1}{N} \vVIMCOgen(\phi;x) + o \lr{\frac{1}{N}}, & \quad \eta =\alpha \PE(\w(z;x)^{1-\alpha})
\end{align*}
with $\mathbb{V}(\gradVIMCOstar)$ and $\mathbb{V}(\gradVIMCOzvam)$ scaling as $1/N^3$ for $\alpha = 0$. We refer to \Cref{app:thmzv} for a precise statement regarding VIMCO-zvAM($\alpha$). The derivations for VIMCO-zv$\star$($\alpha$) unfold similarly and are not included in the paper for brevity. Instead, we conclude this section by investigating the extension of VIMCO-zv$\star$($\alpha$) and VIMCO-zvAM($\alpha$) to  $\alpha \to 1$. These extensions are well defined, and denoting them  VIMCO-zv$\star$($1$) and VIMCO-zvAM($1$) respectively, we have that VIMCO-zvAM($1$) corresponds to the REINFORCE leave-one-out gradient estimator for the ELBO from \cite{salimans2014using,kool2019buy,NEURIPS20209c22c0b5}, while VIMCO-zv$\star$($1$) relates to well-known gradient estimators of the ELBO that appeal to optimal control variates \citep{ranganath2014black}. We refer to \Cref{app:secAlphaTo1} for more details. \looseness=-1

\subsubsection{Discussion}
Our theoretical analysis reveals that the deterioration of standard score-function gradient estimators in importance-weighted VI as the number of importance samples $N$ increases results from the interaction between the importance-weighted variational bound and the choice of control variate. Although the estimators introduced in \cite{mnih2016variational} (equivalent to VIMCO-AM and VIMCO-GM) reduce the variance of the naive score-function gradient estimator of the IWAE bound, they nonetheless suffer from a vanishing SNR as $N$ increases. \Cref{prop:hintCV} provides a general mechanism to construct alternative VIMCO-style gradient estimators resolving this degeneracy. We use this generator theorem to propose new VIMCO gradient estimators with desirable properties such as zero variance at posterior optimality and amenability to weight tempering through the hyperparameter $\alpha \in [0,1)$.

Moreover, our results provide a clearer view of the role of $(N,\alpha)$ in importance-weighted VI. A $\sqrt{N}$ SNR scaling alone does not fully capture the usefulness of a gradient estimator for optimization, as it measures the magnitude of the gradient signal only relative to its noise. Notably, the gradient signal itself vanishes as $N\to\infty$ when $\alpha=0$ for High-ESS OVIS, although its variance decreases rapidly enough for its SNR to grow as $\sqrt{N}$. Thus, variance reduction can make an increasingly weak gradient signal statistically more reliable without preventing its absolute magnitude from vanishing. The situation is different when $\alpha\in(0,1)$: the gradient of the VR-IWAE bound converges to the (generally) nonzero gradient of the corresponding VR bound, while the gradient noise of the proposed estimators decreases with $N$. Positive values of $\alpha$ therefore affect more than the stability of the importance weights and their ESS: they can fundamentally alter the asymptotic scale of the gradient signal. From this perspective, $N$ and $\alpha$ play distinct roles, with $N$ governing the accuracy of the gradient estimate and $\alpha$ additionally impacting the strength of its limiting signal.

How these asymptotic properties translate into optimization efficiency ultimately also depends on the stochastic optimization scheme employed, including its learning-rate schedule and possible adaptation to the scale of the gradients. Our results therefore do not imply that increasing either $N$ or $\alpha$ is universally preferable: they highlight a trade-off between bound tightness, gradient signal, and gradient noise that should be taken into account when choosing $(N,\alpha)$. We refer to \Cref{app:sec:practical:considerations} for further practical considerations on this choice, which also cover the case where the model $p_\theta(x,z)$ is parametrized by $\theta \in \Theta$ (as it is for instance the case in maximum likelihood estimation).

\section{Numerical Experiments}
\label{sec:NumEx}

In this section, we provide empirical evidence supporting our theoretical claims. 
\subsection{Gaussian Example}

\label{subsec:ex}
Let $\theta, \phi \in \rset$. Set
$p(z|x) = \mathcal{N}(z;\theta, 1)$ and
$q_\phi(z) = \mathcal{N}(z; \phi, 1)$. Here, $\theta = 0$ is fixed and the goal is to learn $\phi$. Detailed derivations are provided in \Cref{app:subsub:ex:Gaussian} for this example, which satisfies the assumptions of \Cref{thm:REINFORCEGradientExpectation} as well as \Cref{thm:VarianceReduction,thm:optimalVarReduction,prop:hintCV,thm:sAM} and for which all the quantities appearing in these results admit closed-form expressions. 
To showcase the validity of our theoretical results, we first set $\phi = 1$ and we vary $\alpha \in \{0.0, 0.2, 0.5, 0.7, 0.9\}$ as well as $N \in \{ 2^j : j = 2,\ldots,12\}$. SNR estimates are computed using 1000 Monte Carlo samples, the reported results are averaged over 10 independent replications and displayed in Figure~\ref{fig:Gaussian_distant_case}. \looseness=-1 

Observe that all gradient estimators follow their predicted asymptotic behavior as $N$ increases. When $\alpha = 0$, High-ESS OVIS, VIMCO-zvAM($0$) and VIMCO-zv$\star$($0$) enjoy a $\sqrt{N}$ SNR scaling, whereas the SNR of VIMCO-AM and VIMCO-GM decay at a $1/\sqrt{N}$ rate. When $\alpha \in (0,1)$, VIMCO-zvAM($\alpha$) and VIMCO-zv$\star$($\alpha$) exhibit a $\sqrt{N}$ scaling of the SNR. In this specific example, the per-sample baselines of VIMCO-zvAM($\alpha$) and VIMCO-zv$\star$($\alpha$) are equal asymptotically. Furthermore, the leading order terms for the asymptotic variance of High-ESS OVIS and VIMCO-zvAM($0$) are almost equal when $\phi = 1$ (see \Cref{app:subsub:ex:Gaussian} for details). As a result, the corresponding curves are nearly indistinguishable in Figure~\ref{fig:Gaussian_distant_case}. Notice also that larger values of $\alpha$ lead to faster convergence to the asymptotic regime, in the sense that smaller values of $N$ are sufficient to achieve the predicted SNR scaling. \looseness=-1

\begin{figure}[t]
\begin{center}
    \includegraphics[width=\textwidth]{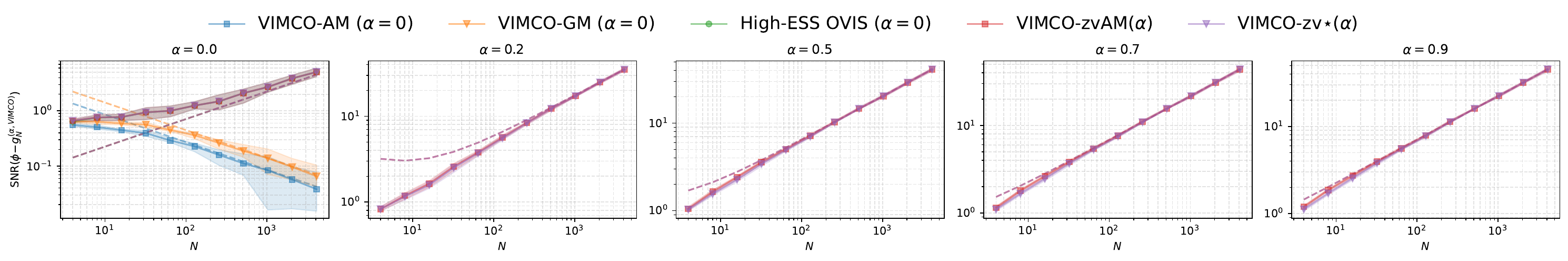}
    \caption{SNR as a function of $N$ for different values of $\alpha$ when $\phi=1$ in the Gaussian example. The solid lines correspond to the SNR estimates for different VIMCO gradient estimators, while the dashed lines correspond to the theoretical predictions. The shaded area represents the 5\% to 95\% percentile range across 10 independent runs.}
    \label{fig:Gaussian_distant_case}
\end{center}
\end{figure}

\begin{figure}[t]
\begin{center}
    \includegraphics[width=0.7\textwidth]{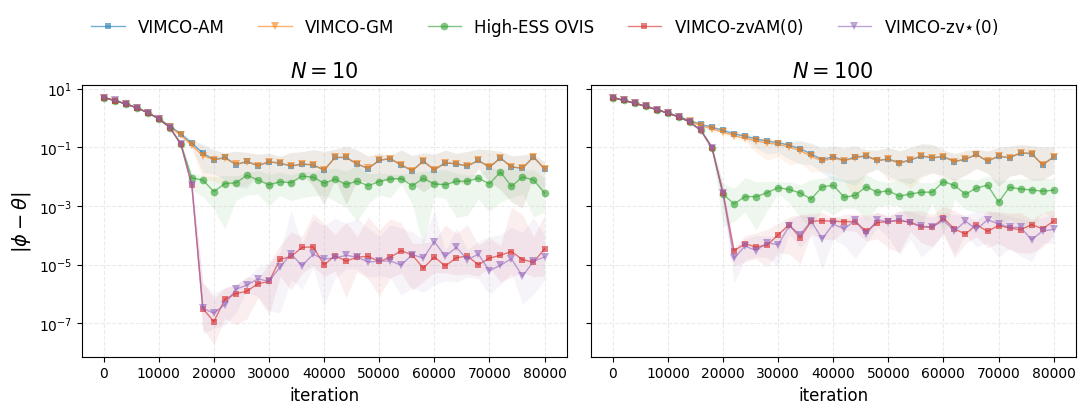}
    \caption{Performance of VIMCO gradient estimators when $\alpha=0$ in the Gaussian example. The shaded area represents the 5\% to 95\% percentile range across 10 independent runs.}
    \label{fig:Gaussian_alpha_0}
\end{center}

\end{figure}

Figure~\ref{fig:Gaussian_alpha_0} further examines the case $\alpha = 0$ and reports the absolute difference $|\phi - \theta|$ along the optimization trajectories for the VIMCO-AM, VIMCO-GM, High-ESS OVIS, VIMCO-zvAM($0$) and VIMCO-zv$\star$($0$) methods when $\phi$ starts as $5$, $N \in \{ 10, 100 \}$ and we use the ADAM optimizer \citep{kingma2015adam} with a learning rate of 1e-3. Although all the methods behave similarly in the initial phase, the trajectories driven by VIMCO-AM and VIMCO-GM quickly plateau upon entering the near-optimal region, followed by High-ESS OVIS. In contrast, optimization based on VIMCO-zvAM($0$) and VIMCO-zv$\star$($0$) continues to make progress and eventually provides a more accurate approximation, illustrating the practical benefit of zero variance at posterior optimality. \Cref{fig:Gaussian_closeMeanVar} in \Cref{app:subsub:ex:Gaussian} provides complementary evidence that, relative to VIMCO-AM and VIMCO-GM, the High-ESS OVIS, VIMCO-zvAM($0$) and VIMCO-zv$\star$($0$) gradient estimators are substantially less noisy as the quality of the variational approximation improves. Lastly, notice that increasing $N$ does not necessarily improve optimization for the Gaussian example. The $N=10$ zero-variance estimators actually attain a smaller final error than their $N=100$ counterparts. This is compatible with our theory since a better SNR alone does not automatically imply faster optimization under an ADAM optimization scheme with a fixed learning rate.

\subsection{Sigmoid Belief Network (SBN)}

We consider the problem of training a SBN \citep{neal1992connectionist} as an archetype of a generative model with discrete latent variables for which reparameterized gradient estimators are not directly applicable \citep{mnih2016variational,lievinNEURIPS2020,masrani2019thermodynamic}. We keep the exact architecture of the model used in \cite{lievinNEURIPS2020}, with a base learning rate fixed to 3e-4 for the ADAM optimizer \citep{kingma2015adam} and mini-batches of size $24$. Furthermore, we use the statically binarized MNIST dataset \citep{salakhutdinov2008quantitative} with the original training/validation/test splits of size 50k/10k/10k. As VIMCO-zv$\star$($\alpha$) with $\alpha >0$ is too computationally expensive for this model, we learn the parameters of the SBN using VIMCO-zvAM($\alpha$) with $\alpha \in \{0., 0.2, 0.5, 0.7, 0.9 \}$ and we add VIMCO-AM, VIMCO-GM and High-ESS OVIS as baselines. All methods are trained with $N \in \{10, 50\}$ for $2 \times 10^6$ optimization steps, and the reported training trajectories are averaged over three independent runs. Performance is assessed based on the marginal log-likelihood estimate given by the IWAE bound with $N = 5000$ \citep{burda2015importance}, evaluated on $10k$ training datapoints, and we also report the ELBO estimate. \looseness=-1

\Cref{fig:SBNperf} plots the marginal log-likelihood and ELBO estimates throughout training for all the methods considered when $N = 50$. While increasing $\alpha$ leads to markedly faster progress during the early stages of optimization for the marginal log-likelihood, the runs with $\alpha=0.5$ and $0.7$ eventually overtake the run with $\alpha=0.9$. At the same time, larger values of $\alpha$ yield higher ELBO values throughout training, with $\alpha=0.9$ attaining the largest ELBO among the configurations considered (in line with the fact that the VR-IWAE bound converges to the ELBO as $\alpha \to 1$). \Cref{fig:SBNperf} therefore illustrates the trade-off induced by $\alpha$: larger values can initially facilitate optimization and lead to higher ELBO values, but at the cost of optimizing a looser importance-weighted objective, so that large values of $\alpha$ may not yield the best marginal log-likelihood. A similar trend is observed when $N = 10$, see \Cref{fig:SBNperf10} in \Cref{app:SBN}. \looseness=-1

\begin{figure}[t]
\begin{center}
\includegraphics[width=\textwidth]{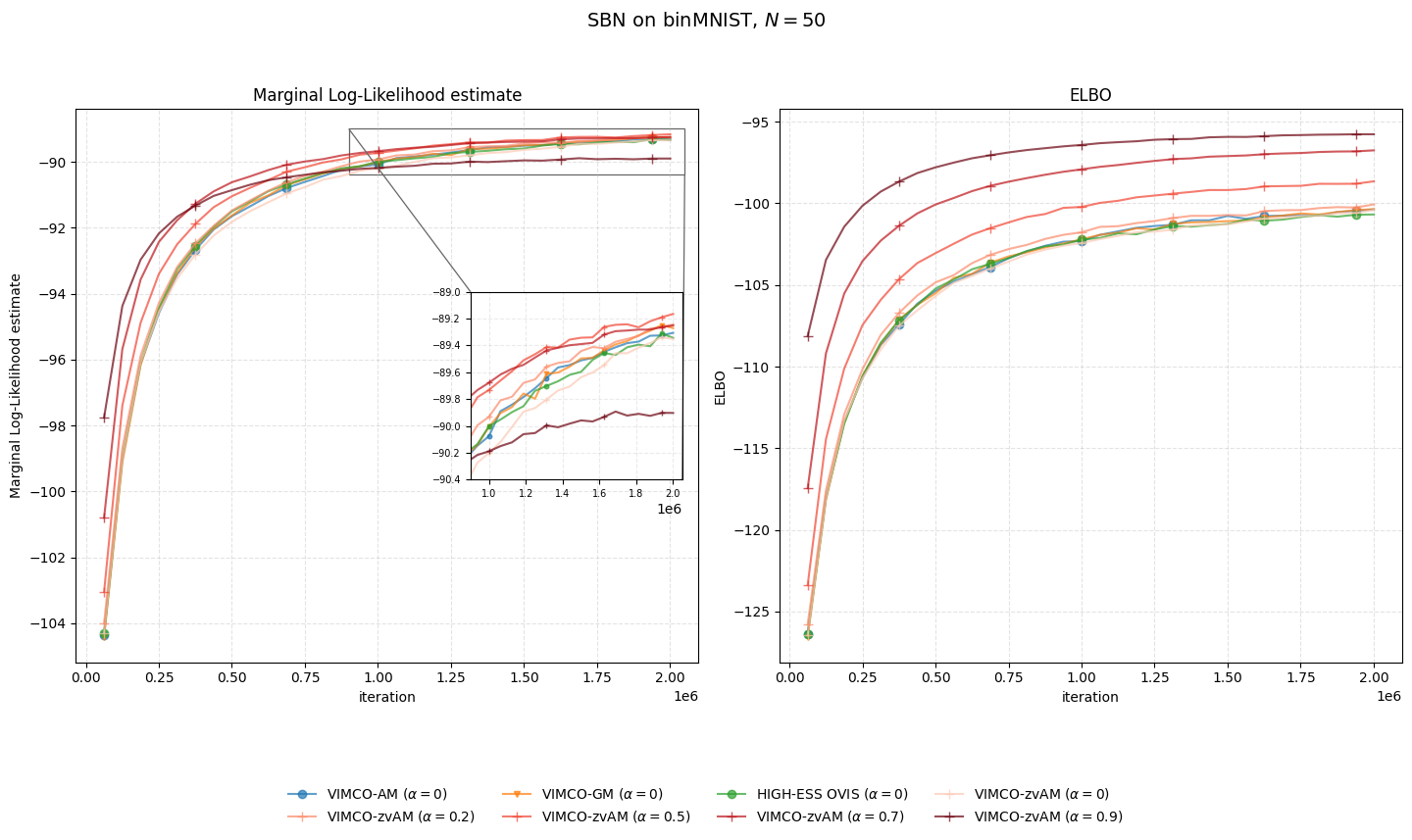}
\end{center}
\caption{Performance of VIMCO gradient estimators in the SBN example averaged over 3 independent runs and expressed in terms of marginal log-likelihood (left) and ELBO (right).}
\label{fig:SBNperf}
\end{figure}

To better understand the optimization behavior observed in \Cref{fig:SBNperf}, \Cref{figSBNrate} examines the quality of the inference-network gradient estimates as a function of $N$. More precisely, we fix a trained SBN checkpoint and a common mini-batch of 24 observations across all estimators and values of $N \in \{2^j : j = 1, \ldots, 13 \}$. For each configuration, we generate $200$ independent VIMCO gradient realizations and for each inference-network parameter coordinate $\psi$, we estimate the gradient mean and variance across these realizations. We report the gradient magnitude, gradient variance and SNR obtained by averaging across the inference-network parameter coordinates.

Focusing on VIMCO-zvAM($\alpha$) when $\alpha>0$ first, \Cref{figSBNrate} then shows that as $N$ increases, the gradient magnitude stabilizes, with larger values of $\alpha$ producing a stronger gradient signal. At the same time, the gradient variance decreases rapidly, resulting in a steadily increasing SNR. These observations are consistent with our theoretical results: for fixed $\alpha>0$, the VR-IWAE gradient converges to the generally nonzero gradient of the corresponding VR bound, while the variance of VIMCO-zvAM($\alpha$) decreases at rate $1/N$, yielding an asymptotic $\sqrt{N}$ SNR scaling. As in the Gaussian experiment, $\sqrt{N}$ SNR rates require a higher sample size $N$ to appear as $\alpha$ decreases. In particular, the asymptotic variance separation between  VIMCO-zvAM($\alpha$), High-ESS OVIS, VIMCO-AM and VIMCO-GM has not yet emerged in \Cref{figSBNrate}. \looseness=-1

Taken together, \Cref{fig:SBNperf,figSBNrate} demonstrate that the trade-off between gradient signal, gradient noise and bound tightness predicted by our theoretical analysis persists in a high-dimensional model with discrete latent variables. Moreover, $\alpha$ plays a central role in determining the computational requirements of importance-weighted VI. Smaller values of $\alpha$ target tighter objectives, but may require a substantially larger $N$ before sufficiently accurate gradient estimates are obtained. Increasing $\alpha$ can move the estimator into a favorable SNR regime for a considerably smaller $N$, at the price of optimizing a looser objective. This trade-off becomes particularly consequential in settings where increasing $N$ is itself computationally expensive.

\begin{figure}[t]
  \begin{center}
\includegraphics[width=\textwidth]{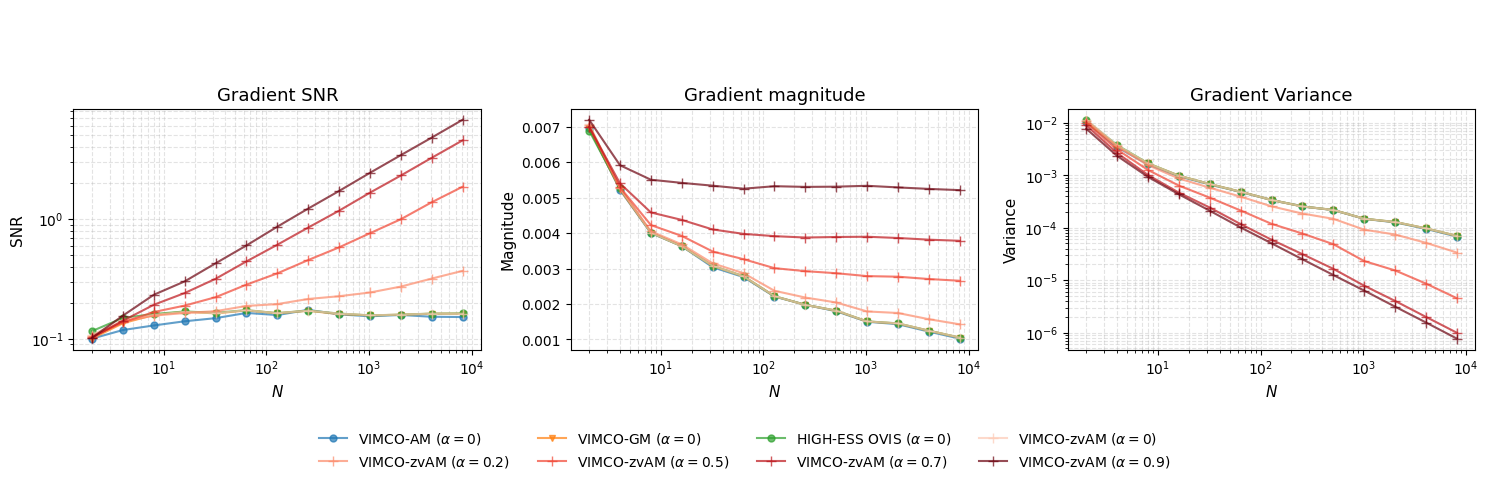}
  \end{center}

\caption{SNR (left), gradient magnitude (middle) and gradient variance (right) as a function of $N$ in the SBN example at a fixed trained SBN checkpoint.}
\label{figSBNrate}
\end{figure}

\subsection{Variational Bayesian Inference for State-Space models}
In this section, we apply the proposed method to achieve computationally efficient and highly accurate variational Bayesian inference in state-space models (SSMs). SSMs are a widely used class of models in statistics, econometrics, signal processing, and many other fields \citep{Durbin:2001}, but Bayesian inference for SSMs is often computationally intensive, which has motivated substantial interest in VI as a scalable alternative \citep{tran2017variational,quiroz2023gaussian,martin2024approximating,loaiza2025efficient}. All existing works apply VI to SSMs by directly optimizing the ELBO. While this approach offers computational advantages, its inferential accuracy can be limited.

Letting $\{x_t\}_{t\geq 1}$ be a time series of observations, SSMs use a sequence of latent variables $\{y_t\}_{t\geq 1}$ to explain the dynamics of the observations by assuming that: for all $t \geq 1$, $x_t|y_t\sim g_t(x_t|y_t,z)$, where $y_t$ follow a Markov chain $y_t|y_{t-1}\sim f_t(y_t|y_{t-1},z)$, $t\geq 2$ with initial state $y_1\sim\mu_z$ and $z$ denotes the model parameters (rather than the latent variables as in the previous sections). Given a training data set $x=x_{1:T}$ of size $T$ and denoting $y=y_{1:T}$, the likelihood function   
\begin{align*}
& p(x|z)=\int p(x|y,z)p(y|z) \rmd y \\
& \mbox{where} \quad p(x|y,z) = \prod_{t=1}^Tg_t(x_t|y_t,z),\;\;p(y|z)=\mu_z(y_1)\prod_{t=2}^Tf_t(y_t|y_{t-1},z),
\end{align*}
is intractable except for a few trivial cases, making it challenging to perform Bayesian inference about $z$. Nevertheless, this issue can be circumvented by constructing an unbiased particle filter estimator $\widehat{p}_\mathrm{PF}(x|z)$ of the intractable likelihood $p(x|z)$ \citep{del2004feynman}. A VI approach can then be employed, in which the target density is an approximation of the posterior density $p(z|x)$ where the likelihood has been replaced by its unbiased estimator $\widehat{p}_\mathrm{PF}(x|z)$ \citep{tran2017variational}. As the particle filter estimator $\widehat{p}_\mathrm{PF}(x|z)$ is often not differentiable with respect to $z$ \citep{malik2011particle}, this prevents the use of reparameterized gradient estimators in VI methods. We consider the stochastic volatility model, an important example within the class of SSMs, and our goal is to approximate the parameter posterior $p(z|x)$ using importance-weighted VI \looseness=-1
\begin{align*}
    x_t|y_t&\sim \mathcal{N}(0,\exp(y_t))\\
    y_t|y_{t-1}&\sim\mathcal{N}\lr{\beta_0+\beta_1(y_{t-1}-\beta_0),\sigma^2},\;\;t\geq 2,\;\;y_1\sim\mathcal{N}\lr{\beta_0,\frac{\sigma^2}{1-\beta_1^2}},
\end{align*}
with $\beta_0\in\mathbb{R}$, $\beta_1\in(0,1)$ and $\sigma^2>0$. Following \cite{tran2017variational}, we use the daily exchange rates for the Australian Dollar/U.S. Dollar from 5/1/2010 to 31/12/2013 (further details, including the prior and the particle filter algorithm, can be found in their paper). The target density is the posterior density $p(z|x)$ in which the likelihood has been replaced by its unbiased estimator $\widehat{p}_\mathrm{PF}(x|z)$ and the transformed parameters are $z=(\beta_0,\log((1-\beta_1)/\beta_1),\log\sigma^2)\in\mathbb{R}^3$. The variational approximation $q_\phi(z)$ is a Gaussian $\mathcal{N}(z;\mu,\Sigma)$ where $\phi=(\mu,\Sigma)$. By viewing $\phi$ as belonging to the product manifold $\mathbb{R}^d\otimes\mathbb{S}^d_{++}$ with $\mathbb{S}^d_{++}$ the manifold of symmetric and positive definite matrices, we apply the manifold VI algorithm of \cite{tran2021variational} adapted to the VR-IWAE bound framework to optimize $\phi$. This manifold VI algorithm first projects VIMCO($\alpha$) gradient estimates w.r.t. $\Sigma$ on the tangent spaces to get Riemannian gradients, before updating $\phi$. We denote by $\widehat{\phi}$ the posterior mean estimate returned by this algorithm when using the VIMCO-zv$\star$($\alpha$) gradient estimator with $N=100$ and the annealing strategy for $\alpha$ described in \Cref{subsec:Annealing_schedule_for_alpha}  (see also \Cref{app:VBISSM} for further implementation details). 
\begin{figure}[t!]
    \centering
    \includegraphics[width = 1\textwidth]{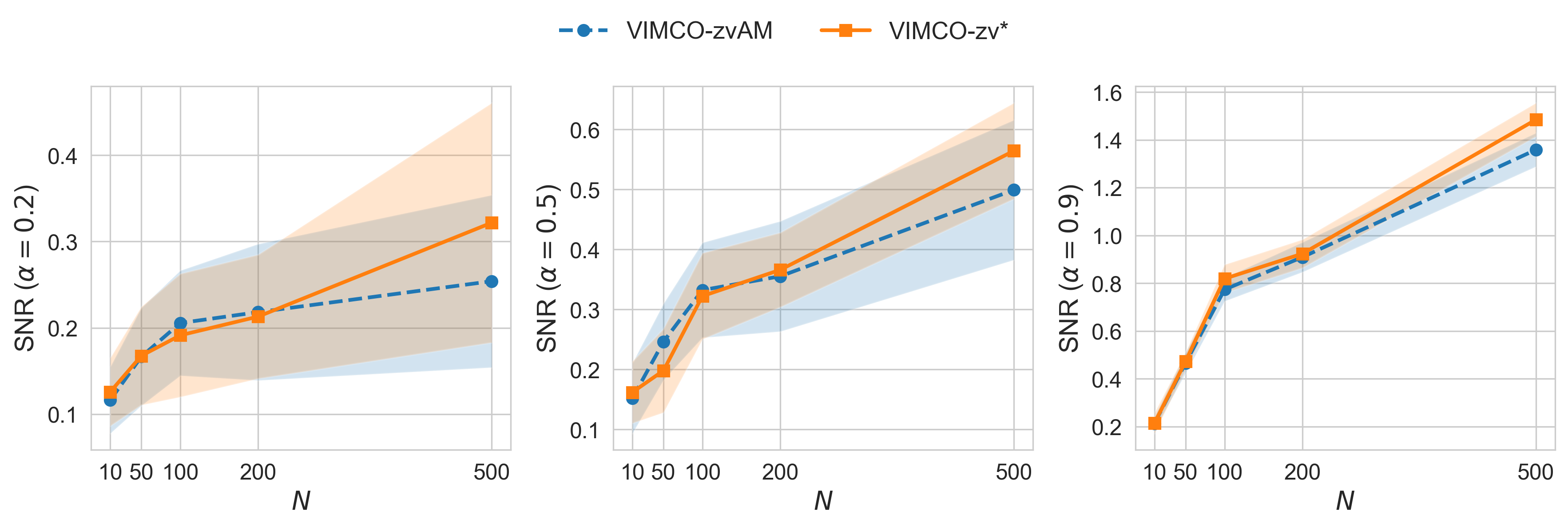}
    \caption{SNR for the VIMCO-zvAM($\alpha$) and VIMCO-zv$\star$($\alpha$) gradient estimators as a function of $N$ for various values of $\alpha$ in the SSMs example. The shaded area shows the standard deviation across replications.}
    \label{fig:SNR_sv_example}
\end{figure}

\Cref{fig:SNR_sv_example} plots the SNR for the VIMCO-zvAM($\alpha$) and VIMCO-zv$\star$($\alpha$) gradient estimators as a function of $N$ and $\alpha$ at $\phi = \widehat{\phi}$. For each curve, we averaged over the components of the gradient vector and 50 replicates. The experimental results in \Cref{fig:SNR_sv_example} agree with the theoretical results from \Cref{subsec:gradientEstimREINF}, which predict that the SNR of the VIMCO($\alpha$) gradient estimators scale as $\sqrt{N}$ for all $\alpha \in (0,1)$, with VIMCO-zv$\star$($\alpha$) exhibiting slightly higher SNR. \looseness=-1

We are interested in examining the performance of the VIMCO gradient estimators relative to the control-variate ELBO gradient estimator (ELBO-cv). The score-function ELBO-cv estimator is well established in the literature (see, e.g., \cite{ranganath2014black,salimans2014using,tran2017variational}) and is equivalent to our VIMCO-zv$\star$($\alpha$) estimator in the limit $\alpha\to1$; see \Cref{app:secAlphaTo1}.
Table \ref{tab:method_comparison} reports the Monte Carlo standard deviations of the variational parameter estimates, together with the number of iterations required, across 20 independent runs. To ensure a fair comparison, the ELBO-cv and VIMCO methods are allocated the same per-iteration computational budget, as measured by $S_{\text{total}}$. For each method, the fixed learning rate is selected from a grid to yield the best corresponding objective value.
The results indicate that VIMCO-zv$\star$($\alpha$) with annealed $\alpha$ provides the best combination of computational efficiency and stability, followed by VIMCO-zvAM($\alpha$) with annealed $\alpha$. \looseness=-1

\begin{table}[t]
{\small
    \centering
    \renewcommand{\arraystretch}{0.9}
    \setlength{\tabcolsep}{15pt}    
    \begin{tabular}{l|c|c|c|c}
    \hline\hline
        \multirow{2}{*}{Method}
        & \multicolumn{2}{c|}{$S_\text{total}=100$}
        & \multicolumn{2}{c}{$S_\text{total}=500$} \\
        \cline{2-5}
        & MC Std & \# iter & MC Std & \# iter \\
        \hline
        ELBO-cv   & 0.0445 & 1499 (258) & 0.0194 & 1647 (188) \\
        \hline
        VIMCO-zv$\star$ (annealed) & 0.0428 & 782 (27) & 0.0108 & 997 (146) \\
        \hline
        VIMCO-zvAM (annealed) & 0.0536 & 776 (76) & 0.0109 & 1510 (148) \\
        \hline\hline
    \end{tabular}
    \caption{
    The table reports the Monte Carlo standard deviations of the variational parameter estimates, together with the number of iterations required, across 20 independent runs, and for two cases of per-iteration budget $S_{\text{total}}$.
    }
    \label{tab:method_comparison}
    }
\end{table}


We now turn to approximating the posterior distribution. We use the variational distribution obtained by optimizing the VR-IWAE objective as a proposal distribution in an importance-sampling-resampling procedure \citep{domke2018}. We refer to the resulting approach as VR-IWAE-IS. This strategy is motivated by the fact that the VR-IWAE bound with $\alpha=0$ is known to lead to variational approximations exhibiting a mass-covering property \citep{domke2018,maddison2017,jiang2026bures}, making it particularly suitable as an importance proposal. We note that the VIMCO methods with annealed $\alpha$ ultimately reach $\alpha=0$ in this example, so that their final variational approximations can be used directly in the VR-IWAE-IS procedure. To contrast the performance of VR-IWAE-IS, we also consider the ELBO-IS procedure.

\Cref{fig:sv_example_ELBO_vs_VRIWAE} shows the marginal posterior estimates obtained using the ELBO and VR-IWAE, with the latter optimized using the VIMCO-zv$\star$($\alpha$) gradient estimator with annealed $\alpha$. For comparison, we also report the corresponding marginals obtained using VR-IWAE-IS and the pseudo-marginal MCMC method of \cite{andrieu2010particle}. Since pseudo-marginal MCMC targets the exact posterior when the likelihood is replaced by a non-negative unbiased estimator, we use its output as a benchmark for the posterior distribution.
The figure shows that the ELBO approximation substantially underestimates the posterior variance. In contrast, the VR-IWAE approach considerably alleviates this underestimation, while its VR-IWAE-IS counterpart yields posterior estimates that align even more closely with the MCMC benchmark. The relative effective sample sizes of the importance samples from VR-IWAE-IS and ELBO-IS are 49.6\% and 5\%, respectively, with corresponding Pareto shape estimates of $\widehat{k}=0.34$ and $0.57$. As discussed by \cite{vehtari2024pareto}, smaller values of $\widehat{k}$ indicate more reliable importance sampling, while $\widehat{k}>0.5$ suggests that the importance weights may have an infinite second moment. Together, these diagnostics suggest that VR-IWAE-IS provides a more reliable approximation to the posterior distribution than ELBO-IS. 
This application demonstrates that the proposed VIMCO($\alpha$) gradient estimators provide an efficient and highly accurate tool for Bayesian inference in SSMs.  

\looseness=-1
\begin{figure}[t!]
    \centering
    \includegraphics[width = .8\textwidth]{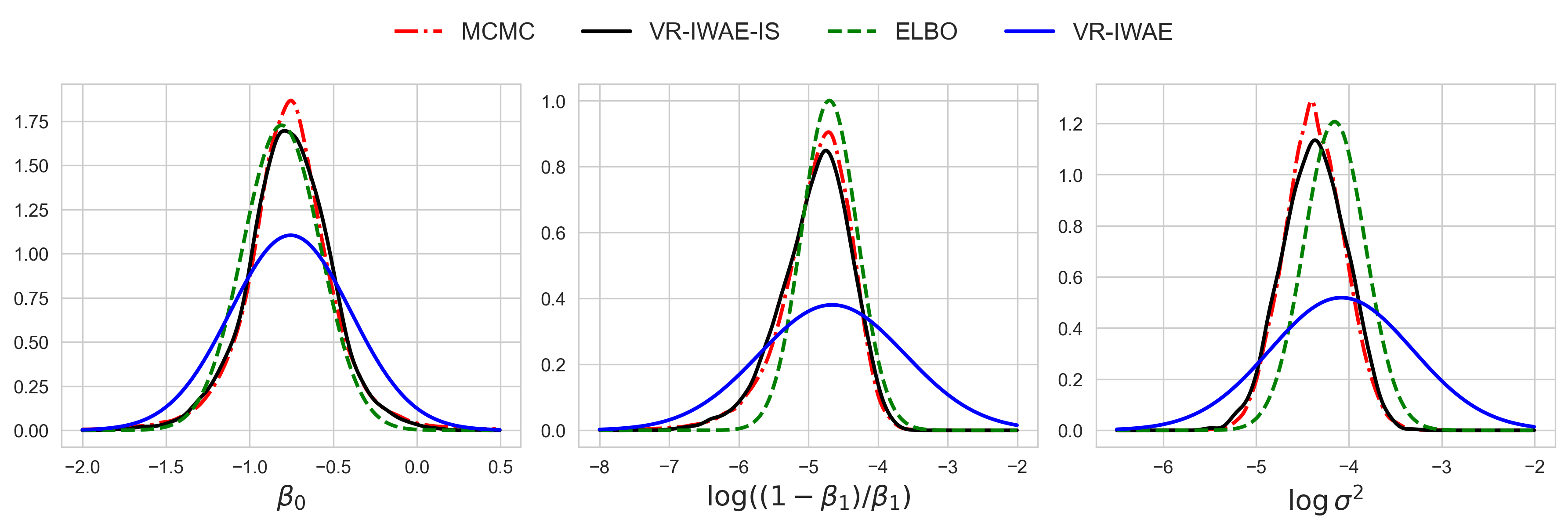}
    \caption{Final marginal posterior estimates in the Variational Bayesian SSMs example.}
    \label{fig:sv_example_ELBO_vs_VRIWAE}
\end{figure}

 \looseness=-1
\section{Conclusion}
\label{sec:ccl}

Our work provides a theoretical and empirical analysis of score-function gradient estimators for importance-weighted VI. We explain the vanishing SNR of VIMCO-AM and VIMCO-GM through the interaction between the variational objective and the choice of control variate. We use this analysis to construct alternative  gradient estimators that (i) eliminate the SNR degeneracy, (ii) have zero-variance at posterior optimality, (iii) are amenable to importance weights tempering and (iv) can improve gradient signal at the cost of optimizing a looser variational bound. Our experiments illustrate that these improvements in gradient estimation can translate into more effective VI optimization. \looseness=-1

The versatility of VIMCO gradient estimators opens several compelling research directions. One promising direction involves developing theoretically grounded annealing schedules for $\alpha$ \cite[analogous to those established in the Sequential Monte Carlo literature, see, e.g.][]{del2012adaptive} and strategies to minimize computational cost of estimating $\eta_\psi^{(\alpha, \star)}$. Furthermore, the assumptions we made in the paper are likely conservative and could be significantly weakened in future work. Lastly, our work provides an asymptotic characterization of VIMCO gradient estimators that complements established results under the reparameterization trick assumption. Our findings challenge the prevailing heuristic that reparameterized gradient estimators are inherently superior to score-function alternatives, as the asymptotic $\sqrt{N}$ SNR rate of VIMCO-zvAM($0$) matches that of the more sophisticated doubly-reparameterized gradient estimator \citep{Tucker2019DoublyRG,daudel2024learningimportanceweightedvariational}. Future research focusing on the behavior of the leading constants within these asymptotic rates is thus expected to further refine the understanding of these two fundamental estimation frameworks. \looseness=-1

\appendix

\appendix

\section{Deferred proofs for \Cref{subsec:asympREINFgrad}}

In the following, $\psi \in \{\phi_1,\ldots, \phi_b \}$, $z \sim q_\phi$ and $z_1,z_2, \ldots$ are i.i.d. copies of $z$.  

\subsection{Proof of \Cref{thm:REINFORCEGradientExpectation}}
\label{app:subsec:thm:REINFORCEGradientExpectation}

We start by capturing the behavior as $N \to \infty$ of the first term in the right-hand side of \eqref{eq:reinfVRIWAEFull}. We then capture the behavior as $N \to \infty$ of the second term in this right-hand side, before pairing both results together to get \Cref{thm:REINFORCEGradientExpectation}.

\subsubsection{Behavior as $N \to \infty$ of the first term in the right-hand side of \eqref{eq:reinfVRIWAEFull}}

We have the following proposition.

\begin{prop} \label{prop:FirstTermExp} Let $\psi \in \{ \phi_1,\ldots, \phi_b \}$ and assume~\ref{hyp:inverseW}, \ref{hyp:momentW} and \ref{hyp:momentGradTwo} with $h, h'>2$. Then: as $N \to \infty$,
\begin{align}
  \PE \lr{\sum_{j = 1}^N \normW(z_j;x) \partial_\psi \log \w(z_j;x)} = - \frac{\PE(\w(z;x)^{1-\alpha} \partial_\psi \log q_\phi(z))   }{\PE(\w(z;x)^{1-\alpha})}  +o(1). \label{eq:FirstTermExpectOone}
\end{align}  
Further assuming that $h,h'\geq 4$, we have: as $N\to\infty$,
\begin{multline} \label{FirtTermExp1OverN}
  \PE \lr{\sum_{j = 1}^N \normW(z_j;x) \partial_\psi \log \w(z_j;x)} \\ = - \frac{\PE(\w(z;x)^{1-\alpha} \partial_\psi \log q_\phi(z))   }{\PE(\w(z;x)^{1-\alpha})}  - \frac{ R_1^{(\alpha)}(\phi;x)}{N}   + o \lr{\frac{1}{N}}
\end{multline}
where 
\begin{multline*}
  R_1^{(\alpha)}(\phi;x) := \frac{\PE(\w(z;x)^{1-\alpha} \partial_\psi \log q_\phi(z)) \PE(\w(z;x)^{2(1-\alpha)})}{\PE(\w(z;x)^{1-\alpha})^3} \\ -  \frac{\PE(\w(z;x)^{2(1-\alpha)} \partial_\psi \log q_\phi(z))}{\PE(\w(z;x)^{1-\alpha})^2}.
\end{multline*}
\end{prop}

\begin{proof}
First note that since $\psi \in \{ \phi_1, \ldots, \phi_b \}$, it holds that
\begin{align*}
 \PE \lr{\sum_{j = 1}^N \normW(z_j;x) \partial_\psi \log \w(z_j;x)} = - \PE \lr{\sum_{j = 1}^N \normW(z_j;x) \partial_\psi \log q_\phi(z_j)}. 
\end{align*}
Set $X = \w(z;x)^{1-\alpha}  \partial_\psi \log q_\phi(z)$, $Y = \w(z;x)^{1-\alpha}$ as well as 
$$
\overline{X}_N = \frac{1}{N} \sum_{j = 1}^N \w(z_j;x)^{1-\alpha}  \partial_\psi \log q_\phi(z_j) \quad \mbox{and} \quad \overline{Y}_N = \frac{1}{N} \sum_{j = 1}^N \w(z_j;x)^{1-\alpha}.$$ 
Under ~\ref{hyp:inverseW}, \ref{hyp:momentW} and \ref{hyp:momentGradTwo} with $h,h'>2$, \eqref{eq:prop:handleRNsnr1} from \Cref{prop:handleRN} holds and: as $N \to \infty$,
\begin{align*}
\PE & \lr{\sum_{j = 1}^N \normW(z_j;x) \partial_\psi \log q_\phi(z_j)} = \frac{\PE(\w(z;x)^{1-\alpha} \partial_\psi \log q_\phi(z))   }{\PE(\w(z;x)^{1-\alpha})} + o(1)
\end{align*} 
(as $\PE(Y^h) < \infty$ with $h > 2$ and $\PE(|X|^{\tilde{h}}) < \infty$ with $\tilde{h} > 1$ by the Cauchy-Schwarz inequality).  Further assuming that $h, h' \geq 4$, \eqref{eq:prop:handleRNsnr1plus1N} from \Cref{prop:handleRN} holds and: as $N\to\infty$, \looseness=-1
\begin{align*}
  \PE & \lr{\sum_{j = 1}^N \normW(z_j;x) \partial_\psi \log q_\phi(z_j)} = \frac{\PE(\w(z;x)^{1-\alpha} \partial_\psi \log q_\phi(z))   }{\PE(\w(z;x)^{1-\alpha})} \\ & \qquad \qquad \qquad \qquad + \frac{1}{N}  \frac{\PE(\w(z;x)^{1-\alpha} \partial_\psi \log q_\phi(z)) \mathbb{V}(\w(z;x)^{1-\alpha})}{\PE(\w(z;x)^{1-\alpha})^3} \\ & \qquad \qquad \qquad \qquad - \frac{1}{N} \frac{\mathbb{C}\mathrm{ov}\lr{\w(z;x)^{1-\alpha} \partial_\psi \log q_\phi(z),\w(z;x)^{1-\alpha}}}{\PE(\w(z;x)^{1-\alpha})^2}   + o \lr{\frac{1}{N}}.
\end{align*}
(as $\PE(Y^h) < \infty$ with $h\geq 4$ and $\PE(|X|^{\tilde{h}}) < \infty$ with $\tilde{h} = 2$ by the Cauchy-Schwarz inequality). Observing that 
\begin{multline*}
\frac{\PE(\w(z;x)^{1-\alpha} \partial_\psi \log q_\phi(z)) \mathbb{V}(\w(z;x)^{1-\alpha})}{\PE(\w(z;x)^{1-\alpha})^3} \\ = \frac{\PE(\w(z;x)^{1-\alpha} \partial_\psi \log q_\phi(z)) \PE(\w(z;x)^{2(1-\alpha)})}{\PE(\w(z;x)^{1-\alpha})^3} - \frac{\PE(\w(z;x)^{1-\alpha} \partial_\psi \log q_\phi(z))}{\PE(\w(z;x)^{1-\alpha})} 
\end{multline*}
and that
\begin{multline*}
\frac{\mathbb{C}\mathrm{ov}\lr{\w(z;x)^{1-\alpha} \partial_\psi \log q_\phi(z),\w(z;x)^{1-\alpha}}}{\PE(\w(z;x)^{1-\alpha})^2} \\ =  \frac{\PE(\w(z;x)^{2(1-\alpha)} \partial_\psi \log q_\phi(z))}{\PE(\w(z;x)^{1-\alpha})^2}  - \frac{\PE(\w(z;x)^{1-\alpha} \partial_\psi \log q_\phi(z))}{\PE(\w(z;x)^{1-\alpha})}
\end{multline*}
yields the desired result \eqref{FirtTermExp1OverN}.
\end{proof}

\subsubsection{Behavior as $N \to \infty$ of the second term in the right-hand side of \eqref{eq:reinfVRIWAEFull}}

We first provide a useful lemma.

\begin{lem} \label{lem:LogTaylorBound} For all $r > 0$, and all $k \in \mathbb{N}^\star$, it holds that
\begin{align}
&    \left| \log r - \lrb{\sum_{i=1}^k (-1)^{i-1} \frac{(r-1)^i}{i}} \right| \leq \frac{|r-1|^{k+1}}{k+1}\lr{1+\frac{1}{r}} \label{eq:LogNthOrderTwo}.
\end{align} 
\end{lem}

\begin{proof} 
Set $f(t) = \log (1+t(r-1))$ for all $t \in [0,1]$ and write that 
\begin{align} \label{eq:interNOrderTaylorLog}
  f(1) - f(0) - \sum_{i=1}^k \frac{f^{(i)}(0)}{i!} & = \int_0^1 \lr{f'(t) - \sum_{i=1}^k \frac{t^{i-1}}{(i-1)!} f^{(i)}(0)} \rmd t.
\end{align}
Next we use that: for all $i = 1 \ldots k$ and all $t \in [0,1]$, 
\begin{align*}
    f^{(i)}(t) = (-1)^{i-1}(i-1)! \lr{\frac{r-1}{1+t(r-1)}}^{i} 
\end{align*}
which follows easily by induction and we obtain that 
\begin{align*}
    &  f'(t) = \frac{r-1}{1+t(r-1)} \\
    & \sum_{i=1}^k \frac{t^{i-1}}{(i-1)!} f^{(i)}(0) = (r-1) \sum_{i=1}^k (-t(r-1))^{i-1} = (r-1) \lrb{ \frac{1-(-t(r-1))^{k}}{1+t(r-1)}}.\end{align*}
Consequently,
we deduce from \eqref{eq:interNOrderTaylorLog} that
\begin{align*}
\log r - \lrb{\sum_{i=1}^k (-1)^{i-1} \frac{(r-1)^i}{i}} & = (-1)^k (r-1)^{k+1} \int_0^1  \frac{t^k}{1+t(r-1)} \rmd t.
\end{align*}
The proof is concluded by using that
\begin{align*}
\int_0^1  \frac{t^k}{1+t(r-1)} \rmd t \leq  \begin{cases}
    \frac{1}{(k+1) r} & \mbox{if $r \leq 1$} \\
    \frac{1}{k+1} & \mbox{otherwise.}
  \end{cases}
\end{align*}
\end{proof}

\begin{rem} 
Our proof borrows elements from \cite[Lemma 5]{cherief2025asymptotics}. The main differences are that their result (i) controls $\left| \log r - \lrb{\sum_{i=1}^k (-1)^{i-1} \frac{(r-1)^i}{i}} \right|$ for $k \in \{1, 2 \}$ only and (ii) they handle integrals of the form $\int_0^1  \frac{t^k}{1+t(r-1)} \rmd t$ by using the following upper bound \looseness=-1
\begin{align*}
  \left|\int_0^1  \frac{t^k}{1+t(r-1)} \rmd t\right| \leq |(r-1)^{-1} \log r|, \quad k \in \{1,2\}.
\end{align*}
\end{rem}
We next present the following proposition.
%
%
%

\begin{prop} \label{prop:SecondTermExp} Let $\psi \in \{ \phi_1,\ldots, \phi_b \}$ and assume~\ref{hyp:inverseW}, \ref{hyp:momentW} with $h\geq 4$ and \ref{hyp:momentGradTwo} with $h'\geq 2$. Then, as $N \to \infty$:
\begin{multline} \label{eq:SecondTermExpOofOne}
  \PE\lr{ \lr{\sum_{i=1}^N \partial_\psi \log q_\phi(z_i)} \log \lr{ \frac{1}{N} \sum_{j = 1}^N \w(z_j;x)^{1-\alpha} }} \\ = \frac{\PE\lr{\w(z;x)^{1-\alpha} \partial_\psi \log q_\phi(z)}}{\PE(\w(z;x)^{1-\alpha})} + o(1).
\end{multline} 
Further assuming that $h \geq 8$, we have: as $N \to \infty$,
\begin{multline} \label{eq:SecondTermExpOofN}
\PE\lr{ \lr{\sum_{i=1}^N \partial_\psi \log q_\phi(z_i)} \log \lr{ \frac{1}{N} \sum_{j = 1}^N \w(z_j;x)^{1-\alpha} }} \\ 
= \frac{\PE\lr{\w(z;x)^{1-\alpha} \partial_\psi \log q_\phi(z)}}{\PE(\w(z;x)^{1-\alpha})} 
 + \frac{R_2^{(\alpha)}(\phi;x)}{N}  + o \lr{\frac{1}{N}}
\end{multline}
where 
\begin{multline*}
R_2^{(\alpha)}(\phi;x) = \frac{\PE\lr{\w(z;x)^{1-\alpha} \partial_\psi \log q_\phi(z)} \PE(\w(z;x)^{2(1-\alpha)})}{\PE(\w(z;x)^{1-\alpha})^3} \\ -\frac{1}{2}\frac{\PE\lr{\w(z;x)^{2(1-\alpha)} \partial_\psi \log q_\phi(z)}}{\PE(\w(z;x)^{1-\alpha})^2}.
\end{multline*}
\end{prop}

\begin{proof}[Proof of \Cref{prop:SecondTermExp}]
Let $X = \partial_\psi \log q_\phi(z)$, $Y = {\w(z;x)^{1-\alpha}}/{\PE(\w(z;x)^{1-\alpha})}$ and \looseness=-1
\begin{align*}
  & \overline{X}_N = \frac{1}{N} \sum_{i=1}^N \partial_\psi \log q_\phi(z_i) \quad \mbox{and} \quad \overline{Y}_N = \frac{1}{N} \sum_{j = 1}^N \frac{\w(z_j;x)^{1-\alpha}}{\PE(\w(z;x)^{1-\alpha})}.
\end{align*}
The quantity of interest we want to study then reads $\PE(N \overline{X}_N \log \overline{Y}_N)$, with $\PE(X) = \PE(Y-1) = 0$ and thus $\PE(\overline{X}_N) = \PE(\overline{Y}_N-1) = 0$. To analyze this quantity, we will use a Taylor expansion of $\log \overline{Y}_N$ and to control the remainder term. First observe that \looseness=-1
\begin{align*}
   \PE\lr{\overline{X}_N (\overline{Y}_N-1)} & = \frac{\PE(X(Y-1))}{N} = \frac{\PE(XY)}{N} \\
\PE(\overline{X}_N (\overline{Y}_N-1)^2) & = \frac{1}{N^2} \PE\lr{X(Y-1)^2} = \frac{1}{N^2} \lrb{ \PE(XY^2) - 2 \PE(XY)} \\
  \PE(\overline{X}_N (\overline{Y}_N-1)^3) 
  & = \frac{3}{N^2} \PE \lr{XY} \mathbb{V}(Y) + \frac{1}{N^3} \lrb{ \PE(X (Y-1)^3) - 3 \PE \lr{XY} \mathbb{V}(Y)}.
\end{align*}
It follows that 
\begin{align}
\PE \lr{N \overline{X}_N (\overline{Y}_N -1)} = \PE(XY)  \label{eq:eSecondTermExp-one}
\end{align}
and we also have
\begin{multline}
\PE \lr{N \overline{X}_N \lrb{(\overline{Y}_N -1)- \frac{(\overline{Y}_N -1)^2}{2} + \frac{(\overline{Y}_N -1)^3}{3}}} = \PE(XY) \\ + \frac{1}{N} \lrb{\PE \lr{XY} \PE(Y^2) -\frac{1}{2}\PE(XY^2)} + \frac{1}{3 N^2} \lrb{ \PE(X (Y-1)^3) - 3 \PE \lr{XY} \mathbb{V}(Y)} \label{eq:eSecondTermExp-two}
\end{multline}
since $\mathbb{V}(Y) = \PE(Y^2) - \PE(Y)^2 = \PE(Y^2) -1$. We now prove \eqref{eq:SecondTermExpOofOne} and \eqref{eq:SecondTermExpOofN} separately. 
\begin{enumerateList}
    \item Proof of \eqref{eq:SecondTermExpOofOne}.
    We will show that 
    \begin{align}
    R_N := \PE \lr{N \overline{X}_N \lrcb{\log \overline{Y}_N - {(\overline{Y}_N -1)}}} = o (1) \label{eq:OofOneExp}.
    \end{align} 
    Paired up with \eqref{eq:eSecondTermExp-one}, this will imply 
    \begin{align} \label{eq:firstResultSecondTermExp}
        \PE \lr{N \overline{X}_N \log \overline{Y}_N } = \PE(XY) + o(1),
    \end{align}
    which is exactly the first result from \Cref{prop:SecondTermExp}. Let us thus show \eqref{eq:OofOneExp}. Using \eqref{eq:LogNthOrderTwo} from \Cref{lem:LogTaylorBound} with $r = \overline{Y}_N$ and $k = 1$, we get
\begin{align*}
    |R_N| \leq \frac{N}{2} \PE\lr{|\overline{X}_N| |\overline{Y}_N-1|^{2}\lr{1+\frac{1}{\overline{Y}_N}} }.
\end{align*}
Consequently, 
\begin{align*}
    |R_N| \leq \frac{N}{2} \lr{ \PE\lr{|\overline{X}_N| |\overline{Y}_N-1|^{2}} + \PE\lr{|\overline{X}_N| \frac{(\overline{Y}_N-1)^{2}}{\overline{Y}_N}}}
\end{align*}
and using the Cauchy-Schwarz inequality, we get 
\begin{align*}
        |R_N| \leq \frac{N}{2} \PE\lr{|\overline{X}_N|^2}^{\frac{1}{2}}   \lrb{  \PE\lr{|\overline{Y}_N-1|^{4}}^{\frac{1}{2}} + \PE\lr{\frac{|\overline{Y}_N-1|^{4}}{\overline{Y}_N^2}}^{\frac{1}{2}}}.
\end{align*}
Next using \Cref{lem:additional-tech-lemma-bis} with $a = \overline{Y}_N$, $b = 1$, $\mu_0 = 4$ and $\eta_0 \in (2,4)$ paired up with the fact that $|a + b|^{1/2} \leq |a|^{1/2} + |b|^{1/2}$ for all $a,b \in \rset$, we have 
\begin{align*}
   \PE \lr{ \frac{|\overline{Y}_N-1|^{4}}{\overline{Y}_N^2}}^{\frac{1}{2}} \leq \PE \lr{\lrav{\overline{Y}_N-1}^{\eta_0}\;\overline{Y}_N^{2-\eta_0}}^{\frac{1}{2}} + \PE \lr{\frac{\lrav{\overline{Y}_N-1}^{\eta_0}}{\overline{Y}_N^2}}^{\frac{1}{2}}.
\end{align*}
Using the Holder's inequality, we then get that: for all $p,q>1$ with $1/p + 1/q = 1$
\begin{align*}
   \PE \lr{ \frac{|\overline{Y}_N-1|^{4}}{\overline{Y}_N^2}}^{\frac{1}{2}} \leq \PE \lr{\lrav{\overline{Y}_N-1}^{p\eta_0}}^{\frac{1}{2p}} \lr{\PE\lr{\overline{Y}_N^{q(2-\eta_0)}}^{\frac{1}{2q}} + \PE \lr{{\overline{Y}_N^{-2q}}}^{\frac{1}{2q}}}.
\end{align*}
At this stage, we have thus obtained that 
\begin{align*}
    |R_N| \leq \frac{N}{2} \PE\lr{|\overline{X}_N|^2}^{\frac{1}{2}}   \lrb{  \PE\lr{|\overline{Y}_N-1|^{4}}^{\frac{1}{2}} + \PE \lr{\lrav{\overline{Y}_N-1}^{p\eta_0}}^{\frac{1}{2p}} \lr{\PE\lr{\overline{Y}_N^{q(2-\eta_0)}}^{\frac{1}{2q}} + \PE \lr{{\overline{Y}_N^{-2q}}}^{\frac{1}{2q}}}}
\end{align*}
Under our assumptions, we can then apply \eqref{eq:petrov-centered} from \Cref{lem:petrov} with $(Z, \overline{Z}_N) = (X,\overline{X}_N)$ and $(Z,\overline{Z}_N) = (Y-1,\overline{Y}_N - 1)$ respectively, which gives
\begin{align*}
  & \sup_{N \in \mathbb{N}^\star} N^{\frac{1}{2}} \PE(|\overline{X}_N|^2)^{\frac{1}{2}} < \infty \\
  & \sup_{N \in \mathbb{N}^\star} N \PE\lr{ |\overline{Y}_N-1|^{4}}^{\frac{1}{2}} < \infty \\
  & \sup_{N \in \mathbb{N}^\star} N^{\frac{\eta_0}{4}} \PE\lr{ |\overline{Y}_N-1|^{p\eta_0}}^{\frac{1}{2p}} < \infty.
\end{align*}
where in the last display we have used that $\eta_0 \in (2,4)$. To conclude the proof, we use that under \ref{hyp:inverseW} we have by \Cref{lem:equivlimsupconditionGen}: for all $\mu >0$ : as $N \to \infty$, $\PE((\overline{Y}_N)^{-\mu}) = O(1)$.

\item Proof of \eqref{eq:SecondTermExpOofN}. We will show that
\begin{align}
R_N:= \PE \lr{N \overline{X}_N \lrcb{\log \overline{Y}_N - \lrb{(\overline{Y}_N -1)- \frac{(\overline{Y}_N -1)^2}{2} + \frac{(\overline{Y}_N -1)^3}{3}}}} = o \lr{\frac{1}{N}}. \label{eq:OofOneOverNExp} 
\end{align}
Paired up with \eqref{eq:eSecondTermExp-two} (in which the last term behaves as $o(1/N)$ since $|\PE(X(Y-1)^3)| \leq \PE(|X|^2)^{1/2} \PE(|Y-1|^6)^{1/2}$ using the Cauchy-Schwarz inequality which implies $|\PE(X(Y-1)^3)| < \infty$ under our assumptions), this will give 
\begin{align*}
    \PE(N \overline{X}_N \log \overline{Y}_N) = \PE(XY) + \frac{1}{N} \lrb{\PE \lr{XY} \PE(Y^2) -\frac{1}{2}\PE(XY^2)} + o \lr{\frac{1}{N}}
\end{align*}
and this is exactly the second result from \Cref{prop:SecondTermExp}. 
Using \eqref{eq:LogNthOrderTwo} from \Cref{lem:LogTaylorBound} with $r = \overline{Y}_N$ and $k = 3$ yields 
\begin{align*}
|R_N| \leq \frac{N}{4} \lrb{ \PE \lr{|\overline{X}_N| |\overline{Y}_N-1|^4} + \PE\lr{|\overline{X}_N| \frac{|\overline{Y}_N-1|^4}{\overline{Y}_N}}}.
\end{align*}
The proof technique is very similar to (i) from there and we can deduce that \eqref{eq:OofOneOverNExp} holds under our assumptions by adapting (i).

\end{enumerateList}

\end{proof}

\subsubsection{Pairing \Cref{prop:FirstTermExp,prop:SecondTermExp} together to prove \Cref{thm:REINFORCEGradientExpectation}}

\begin{proof}[Proof of \Cref{thm:REINFORCEGradientExpectation}] The proof of \eqref{eq:REINFLeadingOrder} follows from \Cref{prop:FirstTermExp,prop:SecondTermExp} paired up with 
\begin{align}
\partial_\psi \mathrm{VR}^{(\alpha)}(\phi;x) = \frac{\alpha}{1-\alpha} \PE\lr{\frac{\w(z;x)^{1-\alpha}}{\PE(\w(z;x)^{1-\alpha})}\partial_\psi \log q_\phi(z) }.
\label{eq:GradVR}
\end{align}
To get \eqref{eq:REINFfirstOrder}, we once again use \Cref{prop:FirstTermExp,prop:SecondTermExp} paired up with \eqref{eq:GradVR}. The proof is concluded if we can show that
\begin{align}
\frac{1}{2} \partial_\psi [\gammaA(\phi;x)^2] & = - \frac{\PE(\w(z;x)^{2(1-\alpha)} \partial_\psi \log q_\phi(z))}{\PE(\w(z;x)^{1-\alpha})^2} \nonumber \\
& \quad + \frac{1}{2(1-\alpha)}\frac{\PE(\w(z;x)^{2(1-\alpha)} \partial_\psi \log q_\phi(z) )}{\PE(\w(z;x)^{1-\alpha})^2} \nonumber \\ 
& \quad + \frac{\PE(\w(z;x)^{2(1-\alpha)}) \PE(\w(z;x)^{1-\alpha} \partial_\psi \log q_\phi(z))}{\PE(\w(z;x)^{1-\alpha})^3} \nonumber \\
& \quad - \frac{1}{1-\alpha} \frac{\PE(\w(z;x)^{2(1-\alpha)}) \PE(\w(z;x)^{1-\alpha} \partial_\psi \log q_\phi(z) )}{\PE(\w(z;x)^{1-\alpha})^3} \nonumber
\end{align}
that is 
\begin{multline}
\frac{1}{2}(1-\alpha) \partial_\psi [\gammaA(\phi;x)^2]  = \frac{1}{2}(1-2(1-\alpha)) \frac{\PE(\w(z;x)^{2(1-\alpha)} \partial_\psi \log q_\phi(z))}{\PE(\w(z;x)^{1-\alpha})^2} \\ - \alpha \frac{\PE(\w(z;x)^{2(1-\alpha)}) \PE(\w(z;x)^{1-\alpha} \partial_\psi \log q_\phi(z) )}{\PE(\w(z;x)^{1-\alpha})^3} \label{eq:GradGamma}
\end{multline}
Let us thus prove \eqref{eq:GradGamma}. 
Since 
\begin{align*}
  \frac{1}{2} (1-\alpha) \partial_\psi [\gammaA(\phi;x)^2] & = \frac{1}{2} \partial_\psi \lr{\frac{\PE(\w(z;x)^{2(1-\alpha)})}{\PE(\w(z;x)^{1-\alpha})^2}} \\
   & = \frac{1}{2} {\frac{\partial_\psi \PE(\w(z;x)^{2(1-\alpha)})}{\PE(\w(z;x)^{1-\alpha})^2}} - {\frac{\PE(\w(z;x)^{2(1-\alpha)}) \partial_\psi \PE(\w(z;x)^{1-\alpha})}{\PE(\w(z;x)^{1-\alpha})^3}}
  \end{align*}
we get \eqref{eq:GradGamma} by plugging in the above the following identity: for all $u>0$, 
  \begin{align*}
    \partial_\psi \PE(\w(z;x)^{u}) & = -u \PE(\w(z;x)^{u} \partial_\psi \log q_\phi(z)) + \PE(\w(z;x)^{u} \partial_\psi \log q_\phi(z) ) \\
    & = (1-u) \PE(\w(z;x)^{u} \partial_\psi \log q_\phi(z)),
  \end{align*}
which concludes the proof.
\end{proof}

\begin{rem}\label{eq:FurtherCommentsThmOne} The proof of \Cref{thm:REINFORCEGradientExpectation} consists in pairing up \Cref{prop:FirstTermExp,prop:SecondTermExp} together, where these two results capture the behavior as $N \to \infty$ of the first and second term in the right-hand side of \eqref{eq:reinfVRIWAEFull} when $\psi \in \{ \phi_1, \ldots, \phi_b \}$ respectively.
  
The proof of \Cref{prop:FirstTermExp} builds on the fact that the first term in the right-hand side of \eqref{eq:reinfVRIWAEFull} can be written as $\PE(\overline{X}_N / \overline{Y}_N)$ for well-chosen sample averages $\overline{X}_N$ and $\overline{Y}_N$, with the asymptotic behavior of estimators of the form $\overline{X}_N / \overline{Y}_N$ as $N \to \infty$ being studied in \cite[Proposition 6]{daudel2024learningimportanceweightedvariational} under general conditions. 
Obtaining \Cref{prop:FirstTermExp} thus follows from \cite[Proposition 6]{daudel2024learningimportanceweightedvariational} and the challenge in analyzing the asymptotic behavior of $\partial_\psi \liren(\phi;x)$ comes from the second term in the right-hand side of \eqref{eq:reinfVRIWAEFull}. One key ingredient in the proof of \Cref{prop:SecondTermExp} is \Cref{lem:LogTaylorBound}, which is a general lemma that establishes an upper bound on the difference between $\log(r)$ and its $k$-th order Taylor expansion for all $r> 0$. 
\end{rem}

\section{Deferred proofs for \Cref{subsec:gradientEstimREINF}}

  \subsection{Proof of \Cref{prop:VarREINFfirstV}}

\label{app:subsec:prop:VarREINFfirstV}

First note that since $\psi \in \{ \phi_1, \ldots, \phi_b \}$, it holds that
\begin{align*}
 \mathbb{V} \lr{\sum_{j = 1}^N \normW(z_j;x) \partial_\psi \log \w(z_j;x)} = \mathbb{V} \lr{\sum_{j = 1}^N \normW(z_j;x) \partial_\psi \log q_\phi(z_j)}. 
\end{align*}
Set $X = \w(z;x)^{1-\alpha}  \partial_\psi \log q_\phi(z)$, $Y = \w(z;x)^{1-\alpha}$ as well as 
$$
\overline{X}_N = \frac{1}{N} \sum_{j = 1}^N \w(z_j;x)^{1-\alpha}  \partial_\psi \log q_\phi(z_j) \quad \mbox{and} \quad \overline{Y}_N = \frac{1}{N} \sum_{j = 1}^N \w(z_j;x)^{1-\alpha}.$$ 
Applying \eqref{eq:prop:handleRNsnrVarAnalysis} from \Cref{prop:handleRN} yields the desired result.

\subsection{Proof of \Cref{prop:VarREINFsecondV}} 

\label{app:subsec:prop:VarREINFsecondV}

For notational convenience, we denote $X = \partial_\psi \log q_\phi(z)$, $Y = {\w(z;x)^{1-\alpha}}/{\PE(\w(z;x)^{1-\alpha})}$ as well as
\begin{align*}
  & \overline{X}_N = \frac{1}{N} \sum_{i=1}^N \partial_\psi \log q_\phi(z_i) \quad \mbox{and} \quad \overline{Y}_N = \frac{1}{N} \sum_{j = 1}^N \frac{\w(z_j;x)^{1-\alpha}}{\PE(\w(z;x)^{1-\alpha})},
\end{align*}
with $\PE(X) = \PE(Y-1) = 0$ and thus $\PE(\overline{X}_N) = \PE(\overline{Y}_N-1) = 0$. We will first study the quantity $\mathbb{V}(N \overline{X}_N \log \overline{Y}_N)$. To that end, note that the central limit theorem yields
\begin{align*}
  &\sqrt{N} \lrb{\begin{matrix} \overline{X}_N\\ \overline{Y}_N -1
      \end{matrix}}\weaklimit \mathcal{N}(0, \boldsymbol{\Sigma}) \quad \mbox{where} \quad \boldsymbol{\Sigma} = \lr{\begin{matrix}
      \mathbb{V}(X) & \mathbb{C}\mathrm{ov}(X, Y) \\
      \mathbb{C}\mathrm{ov}(X,Y) & \mathbb{V}\lr{Y}
    \end{matrix}}.
\end{align*}
Furthermore, denoting $f(u,v) = u \log v$, and observing that 
\begin{align*}
&\left.\lrb{\begin{matrix} \partial_u f(u,v) \\ \partial_v f(u,v)
      \end{matrix}}\right|_{(u,v) = (0,1)} = 0 \\
      & \left.\lr{\begin{matrix} \partial_{uu} f(u,v) & \partial_{uv} f(u,v)  \\ \partial_{vu} f(u,v) & \partial_{vv} f(u,v) 
      \end{matrix}}\right|_{(u,v) = (0,1)} = \lr{\begin{matrix} 0 & 1 \\ 1 & 0 \end{matrix} }
\end{align*}
we can apply the multivariate second order Delta method \citep[see, e.g. Theorem 3.7 from][]{beutner2024Delta} and we get that
\begin{align*}
  N \overline{X}_N \log \overline{Y}_N \weaklimit Z_1 Z_2
\end{align*}
where $\boldsymbol{Z} \sim \mathcal{N}(0, \boldsymbol{\Sigma})$ and $\boldsymbol{Z} = (Z_1, Z_2)$. Next we show that the sequence $((N \overline{X}_N \log \overline{Y}_N)^2)_{N \geq 1}$ is uniformly integrable,  that is we show that there exists $\delta > 0$ such that 
\begin{align}
  \limsup_{N \in \mathbb{N}^*} \PE \lr{(N \overline{X}_N \log \overline{Y}_N)^{2(1+\delta)}} < \infty. \label{eq:UnifIntegrabilityVar}
\end{align}
To that end, we write that by Holder's inequality: for all $p, q \geq 1$ with $1/p + 1/q = 1$, 
\begin{align*}
  \PE \lr{|N \overline{X}_N \log \overline{Y}_N|^{2(1+\delta)}} \leq \lrb{ N^{1+\delta} \PE(|\overline{X}_N|^{2(1+\delta) p})^\frac{1}{p} }\times \lrb{  N^{1+\delta} \PE (|\log \overline{Y}_N|^{2(1+\delta)q})^{1/q}}.
\end{align*}
Further observe that
\begin{align*}
  |\log \overline{Y}_N|^{2(1+\delta)q} \leq 2^{2(1+\delta)q-1} \lr{ |\log \overline{Y}_N - (\overline{Y}_N -1)|^{2(1+\delta)q}+|\overline{Y}_N-1|^{2(1+\delta)q}}
\end{align*}
so that 
\begin{multline*}
  \PE \lr{  |\log \overline{Y}_N|^{2(1+\delta)q}}^{1/q} \\ 
  \leq (2^{2(1+\delta)q-1})^{1/q} \lr{ \PE(|\log \overline{Y}_N - (\overline{Y}_N -1)|^{2(1+\delta)q})+\PE(|\overline{Y}_N-1|^{2(1+\delta)q})}^{1/q}.
\end{multline*}
Hence, we have that 
\begin{multline*}
  \PE \lr{  |\log \overline{Y}_N|^{2(1+\delta)q}}^{1/q} \\ 
  \leq (2^{2(1+\delta)q-1})^{1/q} \lr{ \PE(|\log \overline{Y}_N - (\overline{Y}_N -1)|^{2(1+\delta)q})^{1/q}+\PE(|\overline{Y}_N-1|^{2(1+\delta)q})^{1/q}}.
\end{multline*}
We have thus obtained that 
\begin{multline*}
  \PE \lr{|N \overline{X}_N \log \overline{Y}_N|^{2(1+\delta)}} \leq \lrb{ N^{1+\delta} \PE(|\overline{X}_N|^{2(1+\delta) p})^\frac{1}{p} }\\ \times \lrb{  N^{1+\delta} (2^{2(1+\delta)q-1})^{1/q} \lr{ \PE(|\log \overline{Y}_N - (\overline{Y}_N -1)|^{2(1+\delta)q})^{1/q}+\PE(|\overline{Y}_N-1|^{2(1+\delta)q})^{1/q}}}.
\end{multline*}
Under our assumptions, and setting $p=q=2$, $h=4(1+\delta)q$ and $h'' = 2(1+\delta)p$, we can then apply \eqref{eq:petrov-centered} from \Cref{lem:petrov} with $(Z, \overline{Z}_N) = (X,\overline{X}_N)$ and $(Z,\overline{Z}_N) = (Y-1,\overline{Y}_N - 1)$ respectively, which gives
\begin{align*}
 & \sup_{N \in \mathbb{N}^*} N^{1+\delta} \PE(|\overline{X}_N|^{2(1+\delta) p})^\frac{1}{p} < \infty. \\
 & \sup_{N \in \mathbb{N}^*} N^{1+\delta} \PE(|\overline{Y}_N-1|^{2(1+\delta)q})^{1/q} < \infty.
\end{align*}
The final step to obtain \eqref{eq:UnifIntegrabilityVar} is to show that 
\begin{align} \label{eq:secondPartToBound}
      \sup_{N \in \mathbb{N}^*} N^{1+\delta}  \PE(|\log \overline{Y}_N - (\overline{Y}_N -1)|^{2(1+\delta)q})^{1/q} < \infty.
\end{align}
Using \eqref{eq:LogNthOrderTwo} from \Cref{lem:LogTaylorBound} with $r = \overline{Y}_N$ and $k = 1$, we get
$$
\left| \log \overline{Y}_N - (\overline{Y}_N - 1) \right| \leq \frac{|\overline{Y}_N-1|^{2}}{2}\lr{1+\frac{1}{\overline{Y}_N}}
$$ 
so that
\begin{align*}
  \PE\lr{\left| \log \overline{Y}_N - (\overline{Y}_N - 1) \right|^{2(1+\delta)q} } \leq 2^{-2(1+\delta)q} \PE \lr{|\overline{Y}_N-1|^{4(1+\delta)q}\lr{1+\frac{1}{\overline{Y}_N}}^{2(1+\delta)q}}.
\end{align*}
As a result, we have that 
\begin{align*}
  \PE\lr{\left| \log \overline{Y}_N - (\overline{Y}_N - 1) \right|^{2(1+\delta)q} } \leq \PE\lr{|\overline{Y}_N-1|^{4(1+\delta)q} \lr{1 + \frac{1}{\overline{Y}_N^{2(1+\delta)q}}}}
\end{align*}
from which we deduce 
\begin{align*}
  \PE\lr{\left| \log \overline{Y}_N - (\overline{Y}_N - 1) \right|^{2(1+\delta)q} }^{\frac{1}{q}} & \leq { \PE\lr{|\overline{Y}_N-1|^{4(1+\delta)q}}^{\frac{1}{q}} + \PE\lr{ \frac{|\overline{Y}_N-1|^{4(1+\delta)q}}{\overline{Y}_N^{2(1+\delta)q}}}^{\frac{1}{q}}},
\end{align*}
Similarly to the proof techniques used in \Cref{prop:SecondTermExp}, we get that \eqref{eq:secondPartToBound} holds under our assumptions. 
Consequently,  
\begin{align*}
  & \lim_{N \to \infty} \PE\lr{(N \overline{X}_N \log \overline{Y}_N)^2} = \PE((Z_1 Z_2)^2) \\
  & \lim_{N \to \infty} \PE\lr{N \overline{X}_N \log \overline{Y}_N} = \PE(Z_1 Z_2)
\end{align*}
and thus 
\begin{align*}
  \lim_{N \to \infty} \mathbb{V}\lr{N \overline{X}_N \log \overline{Y}_N} = \mathbb{V}(Z_1 Z_2). 
\end{align*}
By Isserlis's theorem, we have that $\PE((Z_1 Z_2)^2) = \PE(Z_1^2) \PE(Z_2^2) + 2 (\PE(Z_1 Z_2))^2$ so that
\begin{align*}
  \mathbb{V}(Z_1 Z_2) & = \PE(Z_1^2) \PE(Z_2^2) + (\PE(Z_1 Z_2))^2 = \mathbb{V}(X) \mathbb{V}\lr{Y} + \lr{\mathbb{C}\mathrm{ov}\lr{X,Y}}^2.
\end{align*}
We thus get that: as $N \to \infty$, 
\begin{align*}
  \mathbb{V}\lr{N \overline{X}_N \log \overline{Y}_N} = \mathbb{V}(X) \mathbb{V}\lr{Y} + {\mathbb{C}\mathrm{ov}\lr{X,Y}}^2 + o(1)
\end{align*}
that is
\begin{multline}
  \mathbb{V}\lr{\lr{\sum_{i=1}^N \partial_\psi \log q_\phi(z_i)} \log \lr{ \frac{1}{N} \sum_{j = 1}^N \frac{\w(z_j;x)^{1-\alpha}}{\PE(\w(z;x)^{1-\alpha})} } } \\ = \mathbb{V}(\partial_\psi \log q_\phi(z)) \mathbb{V}\lr{\frac{\w(z;x)^{1-\alpha}}{\PE(\w(z;x)^{1-\alpha})}} + {\mathbb{C}\mathrm{ov}\lr{\partial_\psi \log q_\phi(z), \frac{\w(z;x)^{1-\alpha}}{\PE(\w(z;x)^{1-\alpha})}}}^2 + o(1). \label{eq:UsefulForAnalysis}
\end{multline}
Moving on to
\begin{align*}
    \mathbb{V}\lr{\lr{\sum_{i=1}^N \partial_\psi \log q_\phi(z_i)} \log \lr{ \frac{1}{N} \sum_{j = 1}^N {\w(z_j;x)^{1-\alpha}}} },
\end{align*}
that is, $\mathbb{V}(N \overline{X}_N [\log \overline{Y}_N + \log \PE(\w(z;x)^{1-\alpha})])$ and writing
\begin{multline*}
  \mathbb{V}\lr{N \overline{X}_N \lrb{\log \overline{Y}_N + \log \PE(\w(z;x)^{1-\alpha})} } =   \mathbb{V}\lr{N \overline{X}_N \log \overline{Y}_N} \\ + N \mathbb{V}\lr{X}\lr{\log \PE(\w(z;x)^{1-\alpha})}^2 + 2\mathbb{C}\mathrm{ov}(N\overline{X}_N \log \overline{Y}_N, N\overline{X}_N \log \PE(\w(z;x)^{1-\alpha})),
\end{multline*} 
we obtain using Cauchy-Schwarz inequality
\begin{multline*}
\left|\mathbb{C}\mathrm{ov}(N\overline{X}_N \log \overline{Y}_N, N\overline{X}_N \log \PE(\w(z;x)^{1-\alpha}))\right| \\ \leq \sqrt{N} \sqrt{\mathbb{V}\lr{N \overline{X}_N \log \overline{Y}_N}} \sqrt{\mathbb{V}\lr{X}\lr{\log \PE(\w(z;x)^{1-\alpha})}^2}.
\end{multline*}
Using \eqref{eq:UsefulForAnalysis}, that is using the fact that $\mathbb{V}\lr{N \overline{X}_N \log \overline{Y}_N} = O(1)$ gives 
\begin{align}
    \mathbb{V}\lr{\lr{\sum_{i=1}^N \partial_\psi \log q_\phi(z_i)} \log \lr{ \frac{1}{N} \sum_{j = 1}^N {\w(z_j;x)^{1-\alpha}}} } = O(N). \label{eq:NaiveOneV}
\end{align}
The proof is concluded by expanding $\mathbb{V}(\gradNOREP)$ and observing that the covariance term which appears in the expansion is controlled by the Cauchy-Schwarz inequality thanks to \eqref{eq:NaiveOneV} and \Cref{prop:VarREINFfirstV}.

\subsection{Intermediary results}

\begin{prop} \label{lem:ControlCNAlpha} Let $\ell \in \{ 1, 2 \}$, $\psi \in \{ \phi_1, \ldots, \phi_b \}$ and $\eta \geq 0$. Assume \ref{hyp:inverseW} and that 
  \begin{align*}
  & \PE(({\w(z;x)^{(1+\ell)(1-\alpha)}} |\partial_\psi \log q_\phi(z)|)^2) < \infty \\
  & \eta \PE(|\partial_\psi \log q_\phi(z)|^2) < \infty.  
  \end{align*} 
  Then, as $N \to \infty$,
 \begin{align*}
\PE\lr{ \lr{\sum_{i=1}^N \lr{\frac{\w(z_i;x)^{1-\alpha} - \eta}{\sum_{j = 1}^N \w(z_j;x)^{1-\alpha}}}^{1+\ell} |\partial_\psi \log q_\phi(z_i)| \frac{1}{1 - \frac{\w(z_i;x)^{1-\alpha}}{\sum_{j = 1}^N \w(z_j;x)^{1-\alpha}}} }^2  } = o \lr{\frac{1}{N^{2\ell-1}}}.
 \end{align*}
\end{prop}

\begin{proof}
Denoting \begin{align*}
 C_{N}^{(\alpha)} =  \PE\lr{ \lr{\sum_{i=1}^N \lr{\frac{\w(z_i;x)^{1-\alpha}}{\sum_{j = 1}^N \w(z_j;x)^{1-\alpha}}}^{1+\ell} |\partial_\psi \log q_\phi(z_i)| \frac{1}{1 - \frac{\w(z_i;x)^{1-\alpha}}{\sum_{j = 1}^N \w(z_j;x)^{1-\alpha}}} }^2  },
\end{align*}
we have that 
\begin{align*}
   C_{N}^{(\alpha)} =  \PE\lr{ \lr{\sum_{i=1}^N \frac{\w(z_i;x)^{(1+\ell)(1-\alpha)}}{\lr{\sum_{j = 1}^N \w(z_j;x)^{1-\alpha}}^\ell} |\partial_\psi \log q_\phi(z_i)| \frac{1}{\sum_{\overset{j = 1}{j\neq i}}^N \w(z_j;x)^{1-\alpha}} }^2  }.  
\end{align*}
Using the Cauchy-Schwarz inequality $(1/N \sum_i u_i)^2 \leq 1/N \sum_i u_i^2$ that is $(\sum_i u_i)^2 \leq N \sum_i u_i^2$,
\begin{align*}
  C_{N}^{(\alpha)} & \leq N \PE\lr{\sum_{i=1}^N \lr{\frac{\w(z_i;x)^{(1+\ell)(1-\alpha)}}{\lr{\sum_{j = 1}^N \w(z_j;x)^{1-\alpha}}^\ell} |\partial_\psi \log q_\phi(z_i)| \frac{1}{\sum_{\overset{j = 1}{j\neq i}}^N \w(z_j;x)^{1-\alpha}} }^2  }
\end{align*}
so that using first the fact that $z_1, \ldots, z_N$ are i.i.d. and then that $\sum_{{j = 1}}^{N} \w(z_j;x)^{1-\alpha} \geq \sum_{{j = 1}}^{N-1} \w(z_j;x)^{1-\alpha}$, we get 
\begin{align*}
   C_{N}^{(\alpha)} & \leq N^2 \PE\lr{\lr{\frac{\w(z_N;x)^{(1+\ell)(1-\alpha)}}{\lr{\sum_{j = 1}^N \w(z_j;x)^{1-\alpha}}^{\ell}} |\partial_\psi \log q_{\phi}(z_N|x)| \frac{1}{\sum_{{j = 1}}^{N-1} \w(z_j;x)^{1-\alpha}} }^2  } \\
   & \leq N^2 \PE\lr{\lr{{\w(z_N;x)^{(1+\ell)(1-\alpha)}} |\partial_\psi \log q_{\phi}(z_N|x)| \frac{1}{(\sum_{{j = 1}}^{N-1} \w(z_j;x)^{1-\alpha})^{1+\ell}} }^{2}}\\
   &= N^2 \PE\lr{\lr{{\w(z_N;x)^{(1+\ell)(1-\alpha)}} |\partial_\psi \log q_{\phi}(z_N|x)|}^2} \PE\lr{\frac{1}{(\sum_{{j = 1}}^{N-1} \w(z_j;x)^{1-\alpha})^{2(1+\ell)}}}.
\end{align*}
Thus, we have that
\begin{align*}
  C_{N}^{(\alpha)} \leq \frac{N^2}{(N-1)^{2(1+\ell)}} \PE\lr{\lr{{\w(z;x)^{(1+\ell)(1-\alpha)}} |\partial_\psi \log q_\phi(z)|}^2} \PE((\overline{Y}_{N-1})^{-2(1+\ell)})
\end{align*}
with $\overline{Y}_{N} = 1/N \sum_{i=1}^N \w(z_j;x)^{1-\alpha}$. We conclude that $C_{N}^{(\alpha)} = o(1/N^{2\ell-1})$ using that under our assumptions (i) $\PE(({\w(z;x)^{(1+\ell)(1-\alpha)}} |\partial_\psi \log q_\phi(z)|)^2) < \infty$ (ii) for all $\mu >0$ : as $N \to \infty$, $\PE((\overline{Y}_N)^{-\mu}) = O(1)$ by \Cref{lem:equivlimsupconditionGen}. Similarly, we obtain under our assumptions that 
\begin{align*}
 C_{N}^{(\alpha)'} = \eta^{2(1+\ell)} \PE\lr{ \lr{\sum_{i=1}^N \lr{\frac{1}{\sum_{j = 1}^N \w(z_j;x)^{1-\alpha}}}^{1+\ell} |\partial_\psi \log q_\phi(z_i)| \frac{1}{1 - \frac{\w(z_i;x)^{1-\alpha}}{\sum_{j = 1}^N \w(z_j;x)^{1-\alpha}}} }^2  }
\end{align*}
satisfies $ C_{N}^{(\alpha)'} = o(1/N^{2\ell -1})$, which concludes the proof since $(\w(z;x)^{1-\alpha}-\eta)^{1+\ell} \leq 2^{\ell} (\w(z;x)^{(1+\ell)(1-\alpha)}+ \eta^{1+\ell})$.
\end{proof}

\begin{prop} \label{lem:ToProveTwoVIMCOam} Assume \ref{hyp:inverseW},  \ref{hyp:momentW} and \ref{hyp:momentGradTwo} with $h > 8$, $h' > 4$. Then, as $N \to \infty$,
\begin{align}
\mathbb{V}\lr{\sum_{i=1}^N \partial_\psi \log q_\phi(z_i) \lrb{\log \lr{1 - \normW(z_i;x)} + \normW(z_i;x)}} = o \lr{\frac{1}{N}}. \label{eq:ToProveTwoVIMCOam}
\end{align}
\end{prop}

\begin{proof}
  Denoting
\begin{align*}
B_{N}^{(\alpha)} = \mathbb{V}\lr{\sum_{i=1}^N \partial_\psi \log q_\phi(z_i) \lrb{\log \lr{1 - \normW(z_i;x)} + \normW(z_i;x)}},
\end{align*}
we will get \eqref{eq:ToProveTwoVIMCOam} if we can show that $B_{N}^{(\alpha)} =  o \lr{1/N}$. Since
\begin{align*}
  B_{N}^{(\alpha)} \leq \PE \lr{ \lr{\sum_{i=1}^N \partial_\psi \log q_\phi(z_i) \lrb{\log \lr{1 - \normW(z_i;x)} + \normW(z_i;x)}}^2  },
\end{align*}
setting $r = 1 - \normW(z_i;x)$, we have that $r> 0$ and \eqref{eq:LogNthOrderTwo} from \Cref{lem:LogTaylorBound} with $k = 1$ implies 
\begin{align*}
  |B_{N}^{(\alpha)}| \leq \frac{1}{4} \PE \lr{ \lr{\sum_{i=1}^N \lr{\normW(z_i;x)}^2 |\partial_\psi \log q_\phi(z_i)| \lrb{1+ \frac{1}{1 - \normW(z_i;x)}} }^2  }. 
\end{align*}
Consequently, we have that 
\begin{multline}
    |B_{N}^{(\alpha)}| \leq \PE \lr{ \lr{\sum_{i=1}^N \lr{\normW(z_i;x)}^2 |\partial_\psi \log q_\phi(z_i)|}^2} \\ + \PE\lr{ \lr{\sum_{i=1}^N \lr{\normW(z_i;x)}^2 |\partial_\psi \log q_\phi(z_i)| \frac{1}{1 - \normW(z_i;x)} }^2  }. \label{eq:decompositionVarREINF1overN}
\end{multline}
Denoting $X = \w(z;x)^{2(1-\alpha)} |\partial_\psi \log q_\phi(z)|$, $\overline{X}_N = 1/N \sum_{i=1}^N \w(z_i;x)^{2(1-\alpha)} |\partial_\psi \log q_\phi(z_i)|$, $Y = \w(z;x)^{1-\alpha}$ and $\overline{Y}_N = 1/N \sum_{i=1}^N \w(z_i;x)^{1-\alpha}$, the first term in the right-hand side of \eqref{eq:decompositionVarREINF1overN} becomes
\begin{align*}
\frac{1}{N^2} \PE \lr{\lr{\frac{\overline{X}_N}{\overline{Y}_N^2}}^2} = \frac{1}{N^2} \lrb{ \mathbb{V}\lr{\frac{\overline{X}_N}{\overline{Y}_N^2}} + \mathbb{E}\lr{\frac{\overline{X}_N}{\overline{Y}_N^2}}^2}.
\end{align*}
Under our assumptions, we can apply \eqref{eq:prop:handleRNsnrVarAnalysis} and \eqref{eq:prop:handleRNsnr1} from \Cref{prop:handleRN} which yields: as $N \to \infty$, 
\begin{align*}
& \mathbb{V} \lr{\frac{\overline{X}_N}{\overline{Y}_N^2}} = \frac{1}{N}  \frac{\mathbb{V}\lr{X-2Y \frac{\PE(X)}{\PE(Y)}}}{\PE(Y)^4}   +o\lr{\frac{1}{N}} \\
& \PE \lr{\frac{\overline{X}_N}{\overline{Y}_N^2}} = \frac{\PE(X)}{\PE(Y)^2} + o(1)
\end{align*}
so that the first term in the right-hand side of \eqref{eq:decompositionVarREINF1overN} is $o(1/N)$. As for the second term in the right-hand side of \eqref{eq:decompositionVarREINF1overN}, it is $o(1/N)$ by \Cref{lem:ControlCNAlpha} with $(\ell,\eta) = (1,0)$. Consequently, we have obtained that $|B_{N}^{(\alpha)}| = o(1/N)$.
\end{proof}

\begin{prop} \label{prop:ToProveTwoVIMCOsam}  Assume \ref{hyp:inverseW},  \ref{hyp:momentW} and \ref{hyp:momentGradTwo} with $h > 8$, $h' > 4$. Then, as $N \to \infty$,
\begin{multline}
  \mathbb{V}\lr{\sum_{j = 1}^N \normW(z_j;x) \partial_\psi \log q_\phi(z_j)} \\ = \frac{1}{N} \mathbb{V} \lr{\frac{\w(z;x)^{1-\alpha}}{\PE(\w(z;x)^{1-\alpha})} \lrb{ \partial_\psi \log q_\phi(z)  - \frac{\PE(\w(z;x)^{1-\alpha} \partial_\psi \log q_\phi(z))}{\PE(\w(z;x)^{1-\alpha})} } } + o \lr{\frac{1}{N}} \label{eq:ToproveThreeVIMCOam}
\end{multline}
and
\begin{multline}
\log\lr{\frac{N -1 + \alpha}{N-1}} \mathbb{C}\mathrm{ov} \lr{\sum_{j = 1}^N \normW(z_j;x) \partial_\psi \log q_\phi(z_j)   ,    \sum_{i=1}^N \partial_\psi \log q_\phi(z_i) } \\
 = \frac{\alpha}{N} \mathbb{C}\mathrm{ov}  \left(\frac{\w(z;x)^{1-\alpha}}{\PE(\w(z;x)^{1-\alpha})} \biggl[ \partial_\psi \log q_\phi(z)  \right. \\ \left. - \frac{\PE(\w(z;x)^{1-\alpha} \partial_\psi \log q_\phi(z))}{\PE(\w(z;x)^{1-\alpha})}    \biggr], \partial_\psi \log q_\phi(z) \right) + o\lr{\frac{1}{N}}. \label{eq:ToProveOneVIMCOam}
\end{multline}
\end{prop}

\begin{proof}
Applying \eqref{eq:prop:handleRNsnrVarAnalysis} from \Cref{prop:handleRN} yields \eqref{eq:ToproveThreeVIMCOam}. Denoting $X = \w(z;x)^{1-\alpha} \partial_\psi \log q_\phi(z)$, $Y = \w(z;x)^{1-\alpha}$, $Z = \partial_\psi \log q_\phi(z)$, $\overline{X}_N = N^{-1} \sum_{j = 1}^N \w(z_j;x)^{1-\alpha} \partial_\psi \log q_\phi(z_j)$, $\overline{Y}_N = N^{-1} \sum_{j = 1}^N \w(z_j;x)^{1-\alpha}$ as well as $\overline{Z}_N = N^{-1} \sum_{i=1}^N \partial_\psi \log q_\phi(z_i)$, we can rewrite \eqref{eq:ToProveOneVIMCOam} as 
  \begin{align}
 \log \lr{\frac{N-1+\alpha}{N-1}} N \PE \lr{\frac{\overline{X}_N}{\overline{Y}_N} \overline{Z}_N} = \frac{\alpha}{N} \mathbb{C}\mathrm{ov}\lr{\frac{X}{\PE(Y)} - Y \frac{\PE(X)}{\PE(Y)^2},Z} + o \lr{\frac{1}{N}} \label{eq:ToProveOneVIMCOamRewritten}.
  \end{align}
  Let us thus show \eqref{eq:ToProveOneVIMCOamRewritten}. Using that $\PE(\overline{Z}_N) = 0$, we begin by writing that 
\begin{align*}
  \PE \lr{\frac{\overline{X}_N}{\overline{Y}_N} \overline{Z}_N} = \PE \lr{\frac{(\overline{X}_N - \PE(X))\overline{Z}_N}{\overline{Y}_N}} - \frac{\PE(X)}{\PE(Y)} \PE\lr{\frac{(\overline{Y}_N- \PE(Y))\overline{Z}_N}{\overline{Y}_N}}.
\end{align*}
Next using that $1/\overline{Y}_N = 1/\PE(Y) + (\PE(Y)- \overline{Y}_N)/(\PE(Y) \overline{Y}_N)$, we get
\begin{multline*}
    \PE \lr{\frac{\overline{X}_N}{\overline{Y}_N} \overline{Z}_N} = \frac{\PE((\overline{X}_N - \PE(X))\overline{Z}_N)}{\PE(Y)} - \frac{1}{\PE(Y)} \PE \lr{\frac{(\overline{X}_N - \PE(X))\overline{Z}_N(\overline{Y}_N - \PE(Y))}{\overline{Y}_N}} \\ - \frac{\PE(X)}{\PE(Y)^2} \PE\lr{{(\overline{Y}_N- \PE(Y))\overline{Z}_N}} + \frac{\PE(X)}{\PE(Y)^2} \PE\lr{\frac{(\overline{Y}_N- \PE(Y))^2\overline{Z}_N}{\overline{Y}_N}}.
\end{multline*}
In other words, 
\begin{multline*}
    \PE \lr{\frac{\overline{X}_N}{\overline{Y}_N} \overline{Z}_N} = \frac{1}{N} \lr{\frac{\mathbb{C}\mathrm{ov}(X,Z)}{\PE(Y)} - \frac{\PE(X)\mathbb{C}\mathrm{ov}(Y,Z)}{\PE(Y)^2} } \\ + \frac{1}{\PE(Y)} \lrb{ \frac{\PE(X)}{\PE (Y)}\PE\lr{\frac{(\overline{Y}_N- \PE(Y))^2\overline{Z}_N}{\overline{Y}_N}} -  \PE \lr{\frac{(\overline{X}_N - \PE(X))\overline{Z}_N(\overline{Y}_N - \PE(Y))}{\overline{Y}_N}} }.
\end{multline*}
Under our assumptions, 
 we have using \ref{item:eq:prop:handleRN1} from \Cref{prop:handleRN} that 
\begin{align*}
  \PE\lr{\frac{(\overline{Y}_N- \PE(Y))^2\overline{Z}_N}{\overline{Y}_N}} = o \lr{\frac{1}{N}}.
\end{align*}
Furthermore, using the Cauchy-Schwarz inequality twice yields
\begin{align*}
  \PE \lr{\frac{(\overline{X}_N - \PE(X))\overline{Z}_N(\overline{Y}_N - \PE(Y))}{\overline{Y}_N}} \leq \lr{\frac{\mathbb{V}(X)}{N}}^{\frac{1}{2}} \PE(\overline{Z}_N^4)^\frac{1}{4} \PE \lr{ \frac{(\overline{Y}_N- \PE(Y))^4}{\overline{Y}_N^4}}^{\frac{1}{4}}.
\end{align*}
Under our assumptions, we can apply (i) equation \eqref{eq:petrov-centered} from \Cref{lem:petrov} and (ii) \Cref{lem:equivlimsupconditionGen} which gives for all $\mu >0$ : as $N \to \infty$, $\PE((\overline{Y}_N)^{-\mu}) = O(1)$ so that
\begin{align*}
  & \sup_{N \in \mathbb{N}^*} N^{\frac{1}{2}} \PE(\overline{Z}_N^4)^\frac{1}{4} < \infty \\
  & \sup_{N \in \mathbb{N}^*} N^{\frac{1}{2}}  \PE \lr{ \frac{(\overline{Y}_N- \PE(Y))^4}{\overline{Y}_N^4}}^{\frac{1}{4}} < \infty.
\end{align*}
We deduce  
\begin{align*}
  \PE \lr{\frac{(\overline{X}_N - \PE(X))\overline{Z}_N(\overline{Y}_N - \PE(Y))}{\overline{Y}_N}}  = o \lr{\frac{1}{N}}.
\end{align*}
and the proof of \eqref{eq:ToProveOneVIMCOamRewritten} is completed.
\end{proof}

We will also make use of the following inequality: 
for all $a,b \in \rset$,
\begin{align} \label{eq:exponBound}
  |\exp(a) -\exp(b)| = \left| \int_{b}^{a} \exp(t) \rmd t \right| \leq \lr{\exp(a) + \exp(b)}|a - b|,
\end{align}
which follows since $u \mapsto \exp(u)$ is an increasing function on $\rset$, and of the next proposition.

\begin{prop} \label{lem:VIMCOgm} Let $(Y,Z)$ be an $\rset_+^* \times \rset$ random vector. Let $(Y_1,Z_1), (Y_2,Z_2), \ldots$ be a sequence of i.i.d. random vectors with the same distribution as $(Y,Z)$. Assume that $\PE(Y^h)< \infty$, $\PE(|Z|^{h'}) < \infty$, $\PE((Y|Z|)^{h''}) <\infty$ and $\PE(|\log(Y)|^{h'''}) < \infty$ with $h>2$, $h', h'', h'''>4$. For all $N \in \mathbb{N}^*$, denote $\overline{X}_N = (N-1)^{-1} \sum_{i=1}^N \log Y_i$, $\overline{Y}_N = N^{-1} \sum_{i=1}^N Y_i$ and $\overline{Z}_N = N^{-1} \sum_{i=1}^N Z_i$. Further assume that there exist $N \geq 1$ and $\mu > 0$ such that $\PE(\overline{Y}_N^{-\mu}) < \infty$. Then: as $N \to \infty$,
  \begin{align}
    & \PE \lr{\frac{\exp (\overline{X}_N)}{\overline{Y}_N} \overline{Z}_{N}} =  \frac{\exp(\PE(X))}{\PE(Y)} \PE(Z) + o(1) \label{eq:VIMCOgmExpec} \\
    & \mathbb{V} \lr{\frac{\exp (\overline{X}_N)}{\overline{Y}_N} \overline{Z}_{N}} = \frac{\boldsymbol{v}^T \boldsymbol{\Sigma} \boldsymbol{v}}{N} + o \lr{\frac{1}{N}}. \label{eq:VIMCOgmVar}
  \end{align}
where 
\begin{align*}
\boldsymbol{\Sigma} = \lr{\begin{matrix}
      \mathbb{V}(X) & \mathbb{C}\mathrm{ov}(X, Y) &  \mathbb{C}\mathrm{ov}(X,Z)\\
      \mathbb{C}\mathrm{ov}(X,Y) & \mathbb{V}\lr{Y} &  \mathbb{C}\mathrm{ov}(Y,Z)\\
      \mathbb{C}\mathrm{ov}(X,Z) & \mathbb{C}\mathrm{ov}(Y,Z) & \mathbb{V}(Z) \\
    \end{matrix}} \quad \mbox{and} \quad
  \boldsymbol{v}
      = \frac{\exp(\PE(X))}{\PE(Y)}   \lrb{\begin{matrix} {\PE(Z)} \\ - \frac{\PE(Z)}{\PE(Y)} \\
    1
      \end{matrix}}.
  \end{align*}
\end{prop}

\begin{proof} 
The central limit theorem paired up with the multivariate Delta method yields
\begin{align*}
\sqrt{N} \lr{\frac{\exp (\overline{X}_N)}{\overline{Y}_N} \overline{Z}_{N} - \frac{\exp(\PE(X))}{\PE(Y)} \PE(Z)} \weaklimit \mathcal{N}(0, \boldsymbol{v}^T \boldsymbol{\Sigma} \boldsymbol{v}),
\end{align*}
where $X = \log Y$. As a result, the desired expressions \eqref{eq:VIMCOgmExpec} and \eqref{eq:VIMCOgmVar} will follow by showing that the sequence $\lr{N \lr{\frac{\exp (\overline{X}_N)}{\overline{Y}_N} \overline{Z}_{N} - \frac{\exp(\PE(X))}{\PE(Y)} \PE(Z)}^2}_{N \geq 1}$ is uniformly integrable. Let us thus show that for a certain $\delta>0$, we have that 
\begin{align}
  \limsup_{N \to \infty} \PE \lr{  N^{1+\delta} \left|\frac{\exp (\overline{X}_N)}{\overline{Y}_N} \overline{Z}_{N} - \frac{\exp(\PE(X))}{\PE(Y)} \PE(Z)\right|^{2(1+ \delta)}} < \infty. \label{eq:lem:VIMCOgmUnif}
\end{align}
We begin by writing that 
\begin{multline*}
2^{1-2(1+\delta)} \PE \lr{\left|\frac{\exp (\overline{X}_N)}{\overline{Y}_N} \overline{Z}_{N} - \frac{\exp(\PE(X))}{\PE(Y)} \PE(Z)\right|^{2(1+ \delta)}} \\
 \leq  \PE \lr{\left|\frac{\exp (\overline{X}_N) - \exp(\PE(X))}{\overline{Y}_N} \overline{Z}_{N}\right|^{2(1+ \delta)}}  + \exp(2(1+\delta)\PE(X)) \PE \lr{ \left|\frac{\overline{Z}_{N}}{\overline{Y}_N}  - \frac{\PE(Z)}{\PE(Y)} \right|^{2(1+ \delta)}}.
\end{multline*}
Hence, \eqref{eq:lem:VIMCOgmUnif} will follow if we can show that 
\begin{align}
 & \limsup_{N \to \infty} \PE \lr{N^{1+\delta} \left|\frac{\overline{Z}_{N}}{\overline{Y}_N}  - \frac{\PE(Z)}{\PE(Y)} \right|^{2(1+ \delta)}} < \infty. \label{eq:toProveLemVIMCOgmTwo} \\
 & \limsup_{N \to \infty} \PE \lr{N^{1+\delta}\left|\frac{\exp (\overline{X}_N) - \exp(\PE(X))}{\overline{Y}_N} \overline{Z}_{N}\right|^{2(1+ \delta)}} < \infty \label{eq:toProveLemVIMCOgmOne}.
\end{align}
We prove \eqref{eq:toProveLemVIMCOgmTwo} and \eqref{eq:toProveLemVIMCOgmOne} separately.
\begin{enumerateList}

 \item Proof of \eqref{eq:toProveLemVIMCOgmTwo}. Since
 \begin{align*}
  \frac{\overline{Z}_{N}}{\overline{Y}_N}  - \frac{\PE(Z)}{\PE(Y)} = \frac{\overline{Z}_{N} - \PE(Z)}{\overline{Y}_N} - \frac{\PE(Z)}{\PE(Y)}\frac{\overline{Y}_{N}-\PE(Y)}{\overline{Y}_N}, 
\end{align*}
\eqref{eq:toProveLemVIMCOgmTwo} holds if we can show that
\begin{align*}
  \PE \lr{N^{1+\delta} \left|\frac{\overline{Z}_{N}-\PE(Z)}{\overline{Y}_N} \right|^{2(1+ \delta)}} = O(1) \quad \mbox{and} \quad   \PE \lr{N^{1+\delta} \left|\frac{\overline{Y}_{N}-\PE(Y)}{\overline{Y}_N}\right|^{2(1+ \delta)}} = O(1).
\end{align*}
Letting $p,q >1$ and setting $2(1+\delta)p = h$ and $2(1+\delta)q = h'$, the two previous asymptotic bounds hold using Hölder's inequality paired up with \eqref{eq:petrov-centered} and the fact that for all $\mu >0$ : $\PE((\overline{Y}_N)^{-\mu}) = O(1)$ as $N \to \infty$ by \Cref{lem:equivlimsupconditionGen}.

 \item Proof of \eqref{eq:toProveLemVIMCOgmOne}. Using \eqref{eq:exponBound} with $(a,b) = (\overline{X}_N, \PE(X))$ yields
\begin{multline}
2^{1-2(1+\delta)}  \PE \lr{  \left|\frac{\exp (\overline{X}_N) - \exp(\PE(X))}{\overline{Y}_N} \overline{Z}_{N}\right|^{2(1+ \delta)}} \leq \PE \lr{ \lr{\frac{\exp (\overline{X}_N)|\overline{X}_N - \PE(X)|}{\overline{Y}_N} |\overline{Z}_{N}|}^{2(1+ \delta)}} \\ + \exp(2(1+\delta)\PE(X))  \PE \lr{  \lr{\frac{|\overline{X}_N - \PE(X)|}{\overline{Y}_N} |\overline{Z}_{N}|}^{2(1+ \delta)}}  \label{eq:lemVIMCOgm}.
\end{multline}
Focusing first on the first term of the right-hand side above, we get that
\begin{multline*}
  \PE \lr{  \lr{\frac{\exp (\overline{X}_N)|\overline{X}_N - \PE(X)|}{\overline{Y}_N} |\overline{Z}_{N}|}^{2(1+ \delta)}} \\ \leq \PE \lr{\lr{|\overline{X}_N - \PE(X)| \frac{1}{N} \sum_{i=1}^N \frac{\exp\lr{\frac{1}{N-1} \sum_{{j=1},{j\neq i}}^N \log Y_j}}{\overline{Y}_N} \exp\lr{\frac{1}{N-1} \log Y_i} |Z_i|}^{2(1+ \delta)}}.
\end{multline*}
Furthermore, Jensen's inequality yields  
\begin{align*}
{\exp\lr{\frac{1}{N-1} \sum_{{j=1},{j\neq i}}^N \log Y_j}} \leq \frac{1}{N-1} \sum_{{j=1},{j\neq i}}^N Y_j \leq \frac{N}{N-1} {\overline{Y}_N}
\end{align*}
and for all $i = 1 \ldots N$ we also have that $\exp\lr{\frac{1}{N-1} \log Y_i} \leq 1+ Y_i$. As a result, 
we deduce 
\begin{align*}
  \PE \lr{  \lr{\frac{\exp (\overline{X}_N)|\overline{X}_N - \PE(X)|}{|\overline{Y}_N|} |\overline{Z}_{N}|}^{2(1+ \delta)}} \leq \PE \lr{\lr{|\overline{X}_N - \PE(X)| \frac{1}{N-1}\sum_{i=1}^N (1+Y_i)|Z_i|}^{2(1+ \delta)}}.
\end{align*}
Consequently, we will obtain that 
\begin{align*}
 \limsup_{N \to \infty} \PE \lr{ N^{1+\delta} \lr{\frac{\exp (\overline{X}_N)|\overline{X}_N - \PE(X)|}{|\overline{Y}_N|} |\overline{Z}_{N}|}^{2(1+ \delta)}} < \infty
\end{align*}
if we can show that 
\begin{align*}
&  \PE(N^{1+\delta} |\overline{X}_N - \PE(X)|^{2(1+\delta)}) = O(1) \\
 & \PE \lr{ N^{1+\delta} \lr{|\overline{X}_N - \PE(X)| \times \left|\frac{1}{N}\sum_{i=1}^N (1+Y_i)|Z_i| - \PE((1+Y)|Z|) \right|}^{2(1+ \delta)}} = O(1).
\end{align*}
The first asymptotic bound follows from \eqref{eq:petrov-centered}. As for the second asymptotic bound, it follows from the Cauchy-Schwarz inequality paired up with \eqref{eq:petrov-centered}. Similarly, the second term in the right-hand side of \eqref{eq:lemVIMCOgm} satisfies 
$$
\limsup_{N\to \infty} \PE \lr{ N^{1+\delta} \lr{\frac{|\overline{X}_N - \PE(X)|}{\overline{Y}_N} |\overline{Z}_{N}|}^{2(1+ \delta)}} < \infty,
$$
and the proof of \eqref{eq:toProveLemVIMCOgmOne} is concluded.

\end{enumerateList}
\end{proof}

\subsection{Proof of \Cref{thm:VarianceReduction}}

\label{app:subsec:thm:VarianceReduction}
%

We show \eqref{eq:VarAM} and \eqref{eq:VarGM} separately, starting with \eqref{eq:VarAM}.

\subsubsection{Proof of \eqref{eq:VarAM} in \Cref{thm:VarianceReduction}}

Since
\begin{multline*}
\gradVIMCOam =  - \log\lr{\frac{N}{N-1}} \sum_{i=1}^N \partial_\psi \log q_\phi(z_i) \\ - \sum_{i=1}^N \partial_\psi \log q_\phi(z_i) \lrb{\log \lr{1 - \normW[\theta][\phi][1](z_i;x)} + \normW[\theta][\phi][1](z_i;x)}.
\end{multline*}
and  
\begin{align*}
\mathbb{V}\lr{\log\lr{\frac{N}{N-1}} \sum_{i=1}^N \partial_\psi \log q_\phi(z_i)} & =  \lr{\log\lr{\frac{N}{N-1}}}^2 N \mathbb{V}\lr{\partial_\psi \log q_\phi(z)} \nonumber \\
& = \frac{1}{N} \mathbb{V}\lr{\partial_\psi \log q_\phi(z)} + o \lr{\frac{1}{N}},
\end{align*}
it remains to show that 
\begin{align}
\mathbb{V}\lr{\sum_{i=1}^N \partial_\psi \log q_\phi(z_i) \lrb{\log \lr{1 - \normW[\theta][\phi][1](z_i;x)} + \normW[\theta][\phi][1](z_i;x)}} = o \lr{\frac{1}{N}},
\end{align}
which follows from \Cref{lem:ToProveTwoVIMCOam} with $\alpha = 0$.

\begin{rem}[From $\gradNOREP$ to $\gradVIMCOam$] \label{eq:remGlobalBaseline}
The VIMCO-AM gradient estimator can be constructed as the logical estimator with per-sample baselines to build from an intermediate (INTER) estimator with a global baseline and a higher asymptotic variance. To see this, we highlight an intermediary result used in the proof of \Cref{prop:VarREINFsecondV}: under the assumptions of \Cref{prop:VarREINFsecondV} and as $N \to \infty$, \eqref{eq:UsefulForAnalysis} from \Cref{app:subsec:prop:VarREINFsecondV} states that \looseness=-1
\begin{multline*}
  \mathbb{V}\lr{\lr{\sum_{i=1}^N \partial_\psi \log q_\phi(z_i)} \log \lr{ \frac{1}{N} \sum_{j = 1}^N \frac{\w(z_j;x)^{1-\alpha}}{\PE(\w(z;x)^{1-\alpha})} } } \\ = \mathbb{V}(\partial_\psi \log q_\phi(z)) \mathbb{V}\lr{\frac{\w(z;x)^{1-\alpha}}{\PE(\w(z;x)^{1-\alpha})}} + {\mathbb{C}\mathrm{ov}\lr{\partial_\psi \log q_\phi(z), \frac{\w(z;x)^{1-\alpha}}{\PE(\w(z;x)^{1-\alpha})}}}^2 + o(1).
\end{multline*}
Now denoting 
\begin{multline}
\gradIDEAL = \sum_{j = 1}^N \normW[\theta][\phi][1](z_j;x) \partial_\psi \log \w(z_j;x) \\ +  \sum_{i=1}^N \partial_\psi \log q_\phi(z_i) \lrb{\log \lr{ \frac{1}{N} \sum_{j = 1}^N \w(z_j;x)} - \log \PE(\w(z;x))}, \label{eq:BASEREINFestim} 
\end{multline}
it holds that $\PE(\gradIDEAL) = \partial_\psi \liren[0](\phi;x)$, with \Cref{prop:VarREINFfirstV} and \eqref{eq:UsefulForAnalysis} also implying that $\mathbb{V}(\gradIDEAL) = O(1)$ as $N \to \infty$ under the assumptions of \Cref{prop:VarREINFsecondV}.

Hence, the asymptotic variance of INTER is lower than that of NAIVE($0$) \eqref{eq:NaiveREINFestimFull} when $\psi \in \{\phi_1, \ldots, \phi_b \}$ thanks to having included a global baseline in its second term which corresponds to the almost sure limit as $N \to \infty$ of $\log (N^{-1} \sum_{j = 1}^N \w(z_j;x))$. Noting that INTER \eqref{eq:BASEREINFestim} is unusable in practice due to $\PE(\w(z;x))$ being intractable, VIMCO-AM is the more asymptotically efficient practical version of INTER. Indeed, writing that
\begin{align*}
  \frac{1}{N} \lr{\sum_{\overset{j = 1}{{j\neq i}}}^N \w(z_j;x)  + f_{-i}^{(0)}(z_{1:N},\phi;x)} = \frac{1}{N-1} \sum_{\overset{j = 1}{{j\neq i}}}^N \w(z_j;x),
\end{align*}
VIMCO-AM replaces the baseline $\log \PE(\w(z;x))$ by per-sample estimators of it and enjoys an improved asymptotic variance compared to INTER. \looseness=-1
\end{rem}

\subsubsection{Proof of \eqref{eq:VarGM} in \Cref{thm:VarianceReduction}}

We will prove a more general result for the VIMCO($\alpha$) gradient estimator \eqref{eq:VIMCOgenVRIWAE} with \looseness=-1
$$
f_{-i}^{(\alpha)}(z_{1:N},\phi;x) = \lambda \exp \lr{\frac{1}{N-1} \sum_{\overset{j = 1}{{j\neq i}}}^{N} \log (\w(z_j;x)^{1-\alpha})}, \quad \lambda>0,
$$
which we denote by $\psi\text{-}g_{N,\eta}^{(\alpha, \mathrm{VIMCO})}$.

\begin{thm} \label{thm:genGM}
Assume~\ref{hyp:inverseW}, \ref{hyp:momentW} and \ref{hyp:momentGradTwo} with $h > 8$, $h' > 8$. Further assume that $\PE(|\log \w(z;x)|^{h''}) < \infty$ and $\PE(|\log \w(z;x) \partial_\psi \log q_\phi(z)|^{h'''}) < \infty$ with $h''>4$ and $h'''>8$. Then, as $N \to \infty$,
\begin{align*} 
   \mathbb{V} \lr{\psi\text{-}g_{N,\eta}^{(\alpha, \mathrm{VIMCO})}} = \frac{1}{N} \vVIMCOgen(\phi;x) + o \lr{\frac{1}{N}} 
\end{align*}
with \begin{multline*}
\vVIMCOgen(\phi;x) =  \frac{1}{(1-\alpha)^2} \mathbb{V} \left(\alpha \frac{\w(z;x)^{1-\alpha}}{\PE(\w(z;x)^{1-\alpha})} \biggl[ \partial_\psi \log q_\phi(z)  \right. \\ \left. - \PE\lr{\frac{\w(z;x)^{1-\alpha}}{\PE(\w(z;x)^{1-\alpha})} \partial_\psi \log q_\phi(z)} \biggr] - \frac{\eta \partial_\psi \log q_\phi(z)}{\PE(\w(z;x)^{1-\alpha})} \right).
\end{multline*}
and $\eta = \lambda \exp[\PE(\log[\w(z;x)^{1-\alpha}])]$.
\end{thm}

\begin{proof} Write that
\begin{multline*}
\psi\text{-}g_{N,\eta}^{(\alpha, \mathrm{VIMCO})} = \frac{1}{1-\alpha} \sum_{j = 1}^N \frac{\alpha~\w(z_j;x)^{1-\alpha}-\lambda \exp \lr{\frac{1-\alpha}{N-1} \sum_{i=1}^N \log \w(z_i;x)}}{\sum_{\ell = 1}^N \w(z_\ell;x)^{1-\alpha}} \partial_\psi \log q_\phi(z_j)
\\ + R_{1,N}^{(\alpha)}(\phi;x) + R_{1, N}^{'(\alpha)}(\phi;x),
\end{multline*}
where
\begin{align*}
& R_{1,N}^{(\alpha)}(\phi;x) = - \frac{1}{1-\alpha}  \sum_{i=1}^N \partial_\psi \log q_\phi(z_i) \lrb{ \log \lr{1 - g_i^{(\alpha)}(z_{1:N},\phi;x) } + g_i^{(\alpha)}(z_{1:N},\phi;x) } \\
& g_i^{(\alpha)}(z_{1:N},\phi;x) = \frac{\w(z_i;x)^{1-\alpha} - f_{-i}^{(\alpha)}(z_{1:N},\phi;x)}{\sum_{j = 1}^N \w(z_j;x)^{1-\alpha}}, \quad i = 1 \ldots N \\
& f_{-i}^{(\alpha)}(z_{1:N},\phi;x) = \lambda \exp \lr{\frac{1}{N-1} \sum_{\overset{j = 1}{{j\neq i}}}^{N} \log (\w(z_j;x)^{1-\alpha})}, \quad i = 1 \ldots N 
\end{align*}
and
\begin{multline*}
  R_{1, N}^{'(\alpha)}(\phi;x) \\ = -\frac{\lambda}{1-\alpha} \exp \lr{\frac{1-\alpha}{N-1} \sum_{i=1}^N \log \w(z_i;x)}  \sum_{j = 1}^N \frac{\exp\lr{-\frac{1-\alpha}{N-1} \log \w(z_j;x)} - 1}{\sum_{\ell = 1}^N \w(z_\ell;x)^{1-\alpha}} \partial_\psi \log q_\phi(z_j). 
\end{multline*}
The proof will be concluded if we can show that 
\begin{align} \label{eq:VIMCOgmRemainder}
&  \mathbb{V}\lr{R_{1,N}^{(\alpha)}(\phi;x) + R_{1, N}^{'(\alpha)}(\phi;x)} = o \lr{\frac{1}{N}} \\
& \mathbb{V} \lr{\sum_{j = 1}^N \frac{\exp \lr{\frac{1-\alpha}{N-1} \sum_{i=1}^N \log \w(z_i;x)}}{\sum_{\ell = 1}^N \w(z_\ell;x)^{1-\alpha}} \partial_\psi \log q_\phi(z_j)} \nonumber \\
& \qquad \qquad \qquad \qquad = \lr{\frac{\exp[(1-\alpha) \PE(\log \w(z;x))]}{\PE(\w(z;x)^{1-\alpha})}}^2 \frac{\mathbb{V}(\partial_\psi \log q_\phi(z))}{N} + o \lr{\frac{1}{N}} \label{eq:ToproveVIMCOgmOne}
\end{align}
that \eqref{eq:ToproveThreeVIMCOam} holds and that
\begin{multline}
\alpha \mathbb{C}\mathrm{ov} \lr{ \sum_{j = 1}^N \normW(z_j;x) \partial_\psi \log q_\phi(z_j)   , \sum_{j = 1}^N \frac{\exp \lr{\frac{1-\alpha}{N-1} \sum_{i=1}^N \log \w(z_i;x)}}{\sum_{\ell = 1}^N \w(z_\ell;x)^{1-\alpha}} \partial_\psi \log q_\phi(z_j)} \\
 = \frac{1}{N} \mathbb{C}\mathrm{ov}\lr{U, V} + o \lr{\frac{1}{N}} \label{eq:ToproveVIMCOgmTwo}
\end{multline}
where 
\begin{align*}
 & U =  \frac{\alpha~\w(z;x)^{1-\alpha}}{\PE(\w(z;x)^{1-\alpha})} \biggl[ \partial_\psi \log q_\phi(z)   - \frac{\PE(\w(z;x)^{1-\alpha} \partial_\psi \log q_\phi(z))}{\PE(\w(z;x)^{1-\alpha})}    \biggr] \\
 & V =  \frac{\exp[(1-\alpha) \PE(\log \w(z;x))]}{\PE(\w(z;x)^{1-\alpha})}  \partial_\psi \log q_\phi(z).
\end{align*}


\begin{enumerateList}
  \item Proof of \eqref{eq:VIMCOgmRemainder}. We begin by writing that 
\begin{align*} \mathbb{V}\lr{R_{1,N}^{(\alpha)}(\phi;x) + R_{1, N}^{'(\alpha)}(\phi;x)} & \leq \PE\lr{\lr{R_{1,N}^{(\alpha)}(\phi;x) + R_{1, N}^{'(\alpha)}(\phi;x)}^2} \\
  & \leq 2 \lr{\PE(R_{1,N}^{(\alpha)}(\phi;x)^2) + \PE(R_{1,N}^{'(\alpha)}(\phi;x)^2)}. 
\end{align*}
On the one hand, \Cref{lem:LogTaylorBound} with $k=1$ yields: for all $i = 1 \ldots N$,
\begin{multline*}
 \left| \log \lr{1 - g_i^{(\alpha)}(z_{1:N},\phi;x) } + g_i^{(\alpha)}(z_{1:N},\phi;x)  \right| \\  \leq \frac{1}{2} \lr{ g_i^{(\alpha)}(z_{1:N},\phi;x)}^2 \lr{1 + \frac{1}{1 -  g_i^{(\alpha)}(z_{1:N},\phi;x)}}
\end{multline*}
and since $f_{-i}^{(\alpha)}(z_{1:N},\phi;x) \geq 0$, we deduce
\begin{multline}
 \left| \log \lr{1 - g_i^{(\alpha)}(z_{1:N},\phi;x) } + g_i^{(\alpha)}(z_{1:N},\phi;x)  \right| \\ \leq \frac{1}{2} \lr{ g_i^{(\alpha)}(z_{1:N},\phi;x) }^2 \lr{1+ \frac{\sum_{j = 1}^N \w(z_j;x)^{1-\alpha}}{\sum_{\overset{j = 1}{{j\neq i}}}^{N} \w(z_j;x)^{1-\alpha}}}. \label{eq:VIMCOgmInterBoundRemainder}
\end{multline}
On the other hand, \eqref{eq:exponBound} with $(a,b) = (-\frac{1-\alpha}{N-1} \log \w(z_j;x),0)$ implies that for all $j = 1 \ldots N$, 
\begin{multline}
 \left| \exp\lr{-\frac{1-\alpha}{N-1} \log \w(z_j;x)} - 1 \right| \\ \leq \lr{\exp\lr{-\frac{1-\alpha}{N-1} \log \w(z_j;x)} +1} \left|\frac{1-\alpha}{N-1} \log \w(z_j;x) \right|. \label{eq:VIMCOgmInterBoundRemainderTwo}
\end{multline}
We thus deduce from \eqref{eq:VIMCOgmInterBoundRemainder} and \eqref{eq:VIMCOgmInterBoundRemainderTwo} that 
\begin{align}
 (1-\alpha)^2 \mathbb{V}\lr{R_{1,N}^{(\alpha)}(\phi;x) + R_{1, N}^{'(\alpha)}(\phi;x)} \leq A_{N}^{(\alpha)}(\phi;x) + B_{N}^{(\alpha)}(\theta\phi;x) \label{eq:VIMCOgmRemainderBound}
\end{align}
where 
\begin{align*}
 2 A_{N}^{(\alpha)}(\phi;x) = \PE \lr{ \lr{     \sum_{i=1}^N \lr{ g_i^{(\alpha)}(z_{1:N},\phi;x) }^2 |\partial_\psi \log q_\phi(z_i)|  \lr{1+ \frac{\sum_{j = 1}^N \w(z_j;x)^{1-\alpha}}{\sum_{\overset{j = 1}{{j\neq i}}}^{N} \w(z_j;x)^{1-\alpha}}}}^2  }
\end{align*}
and
\begin{multline*}
  B_{N}^{(\alpha)}(\theta\phi;x) = 2\PE\left( \left[\lambda  \exp \lr{\frac{1-\alpha}{N-1} \sum_{i=1}^N \log \w(z_i;x)} \right. \right.  \\ \times \left. \left. \sum_{j = 1}^N \frac{\lr{\exp\lr{-\frac{1-\alpha}{N-1} \log \w(z_j;x)}+1}  \left|\frac{1-\alpha}{N-1} \log \w(z_j;x)\right|}{\sum_{\ell = 1}^N \w(z_\ell;x)^{1-\alpha}} |\partial_\psi \log q_\phi(z_j)|\right]^2 \right).
\end{multline*}
We now study $A_{N}^{(\alpha)}(\phi;x) $ and $B_{N}^{(\alpha)}(\phi;x)$ separately and show that they are both $o(1/N)$.
\paragraph*{Study of $A_{N}^{(\alpha)}(\phi;x) $.} 
First note that by Jensen's inequality: for all $i = 1 \ldots N$,
\begin{align}
  f_{-i}^{(\alpha)}(z_{1:N},\phi;x) \leq \frac{\lambda}{N-1} \sum_{\overset{j = 1}{j\neq i}}^N \w(z_j;x)^{1-\alpha} \leq \frac{\lambda}{N-1} \sum_{{j = 1}}^N \w(z_j;x)^{1-\alpha}. \label{eq:Jensenfminusi}
\end{align}
Since $(\w(z_i;x)^{1-\alpha} - f_{-i}^{(\alpha)}(z_{1:N},\phi;x))^2 \leq 2 \lr{\w(z_i;x)^{2(1-\alpha)} +f_{-i}^{(\alpha)}(z_{1:N},\phi;x)^2}$, we deduce
\begin{align}
& A_{N}^{(\alpha)}(\phi;x)
\nonumber \\ 
& \leq 4 \PE \lr{ \lr{    \sum_{i=1}^N \lr{\normW(z_i;x)}^2 |\partial_\psi \log q_\phi(z_i)|  \lr{1+ \frac{\sum_{j = 1}^N \w(z_j;x)^{1-\alpha}}{\sum_{\overset{j = 1}{{j\neq i}}}^{N} \w(z_j;x)^{1-\alpha}}}}^2  } \nonumber \\ 
& \quad + \frac{4 \lambda^4}{(N-1)^4} \PE \lr{ \lr{ \sum_{i=1}^N |\partial_\psi \log q_\phi(z_i)|  \lr{1+ \frac{\sum_{j = 1}^N \w(z_j;x)^{1-\alpha}}{\sum_{\overset{j = 1}{{j\neq i}}}^{N} \w(z_j;x)^{1-\alpha}}}}^2  }. \label{eq:interVIMCOgm}
\end{align}
\Cref{lem:ControlCNAlpha} with $(\ell,\eta) = (1,0)$ then implies that the first term in the right-hand side of \eqref{eq:interVIMCOgm} is $o(1/N)$. As for the second term in the right-hand side of \eqref{eq:interVIMCOgm}, writing that
\begin{align*}
1+ \frac{\sum_{j = 1}^N \w(z_j;x)^{1-\alpha}}{\sum_{\overset{j = 1}{{j\neq i}}}^{N} \w(z_j;x)^{1-\alpha}} = 2 + \frac{\w(z_i;x)^{1-\alpha}}{\sum_{\overset{j = 1}{{j\neq i}}}^{N} \w(z_j;x)^{1-\alpha}}
\end{align*}
we deduce
\begin{multline*}
  \PE \lr{ \lr{\sum_{i=1}^N |\partial_\psi \log q_\phi(z_i)|  \lr{1+ \frac{\sum_{j = 1}^N \w(z_j;x)^{1-\alpha}}{\sum_{\overset{j = 1}{{j\neq i}}}^{N} \w(z_j;x)^{1-\alpha}}}}^2  } \\ \leq  8 \PE \lr{ \lr{ \sum_{i=1}^N |\partial_\psi \log q_\phi(z_i)|}^2 } + 2 \PE \lr{ \lr{\sum_{i=1}^N   \frac{\w(z_i;x)^{1-\alpha} |\partial_\psi \log q_\phi(z_i)|}{\sum_{\overset{j = 1}{{j\neq i}}}^{N} \w(z_j;x)^{1-\alpha}}}^2  }.
\end{multline*}
Now using that $(\sum_{i=1}^N u_i)^2 \leq N \sum_{i=1}^N u_i^2$ by the Cauchy-Schwarz inequality, we get
\begin{align*}
&  \PE \lr{ \lr{ \sum_{i=1}^N |\partial_\psi \log q_\phi(z_i)|}^2 } \leq N^2 \mathbb{V}\lr{\partial_\psi \log q_\phi(z)} \\
& \PE \lr{ \lr{\sum_{i=1}^N   \frac{\w(z_i;x)^{1-\alpha} |\partial_\psi \log q_\phi(z_i)|}{\sum_{\overset{j = 1}{{j\neq i}}}^{N} \w(z_j;x)^{1-\alpha}}}^2  } \leq N   \PE \lr{ \sum_{i=1}^N  \lr{ \frac{\w(z_i;x)^{1-\alpha} |\partial_\psi \log q_\phi(z_i)|}{\sum_{\overset{j = 1}{{j\neq i}}}^{N} \w(z_j;x)^{1-\alpha}}}^2  }.
\end{align*}
Furthermore, since $z_1, \ldots z_N$ are i.i.d. and denoting $\overline{Y}_{N} = N^{-1} \sum_{i=1}^N \w(z_i;x)^{1-\alpha}$, we get
\begin{align*}
\PE \lr{ \sum_{i=1}^N  \lr{ \frac{\w(z_i;x)^{1-\alpha} |\partial_\psi \log q_\phi(z_i)|}{\sum_{\overset{j = 1}{{j\neq i}}}^{N} \w(z_j;x)^{1-\alpha}}}^2} \leq N \PE \lr{ \lr{{\w(z;x)^{1-\alpha} |\partial_\psi \log q_\phi(z)|}}^2  } \PE \lr{\overline{Y}_{N-1}^{-2}}.
\end{align*} 
Under our assumptions (i) $\PE \lr{ \lr{{\w(z;x)^{1-\alpha} |\partial_\psi \log q_\phi(z)|}}^2  } < \infty$ (ii) for all $\mu >0$ : as $N \to \infty$, $\PE((\overline{Y}_N)^{-\mu}) = O(1)$ by \Cref{lem:equivlimsupconditionGen}, we obtain by combining with the above that 
\begin{align}
A_{N}^{(\alpha)}(\phi;x) = o \lr{\frac{1}{N}}. \label{eq:ANVIMCOgm}
\end{align}

\paragraph*{Study of $B_{N}^{(\alpha)}(\phi;x)$.} Write that
\begin{multline}
 2^{-1} B_{N}^{(\alpha)}(\phi;x) \leq \PE\left( \left( \sum_{j = 1}^N \frac{f_{-j}^{(\alpha)}(z_{1:N},\phi;x) \left|\frac{1-\alpha}{N-1} \log \w(z_j;x)\right|}{\sum_{\ell = 1}^N \w(z_\ell;x)^{1-\alpha}} |\partial_\psi \log q_\phi(z_j)|\right)^2 \right) \\ + \lambda^2 \PE\left( \left(  \exp \lr{\frac{1-\alpha}{N-1} \sum_{i=1}^N \log \w(z_i;x)} \sum_{j = 1}^N \frac{\left|\frac{1-\alpha}{N-1} \log \w(z_j;x)\right|}{\sum_{\ell = 1}^N \w(z_\ell;x)^{1-\alpha}} |\partial_\psi \log q_\phi(z_j)|\right)^2 \right). \label{eq:StudyBNgm}
\end{multline}
Using \eqref{eq:Jensenfminusi} in the first term of the right-hand side above, we obtain that 
\begin{multline*}
\PE\left( \left( \sum_{j = 1}^N \frac{f_{-j}^{(\alpha)}(z_{1:N},\phi;x) \left|\frac{1-\alpha}{N-1} \log \w(z_j;x)\right|}{\sum_{\ell = 1}^N \w(z_\ell;x)^{1-\alpha}} |\partial_\psi \log q_\phi(z_j)|\right)^2 \right) \\ 
\leq \frac{\lambda^2 (1-\alpha)^2}{(N-1)^4} \PE\left( \left( \sum_{j = 1}^N \left|\log \w(z_j;x) \partial_\psi \log q_\phi(z_j)\right|\right)^2 \right)
\end{multline*}
and next using the fact that $(\sum_{i=1}^N u_i)^2 \leq N \sum_{i=1}^N u_i^2$ thanks to the Cauchy-Schwarz inequality, we deduce that 
\begin{multline*}
\PE\left( \left( \sum_{j = 1}^N \frac{f_{-j}^{(\alpha)}(z_{1:N},\phi;x) \left|\frac{1-\alpha}{N-1} \log \w(z_j;x)\right|}{\sum_{\ell = 1}^N \w(z_\ell;x)^{1-\alpha}} |\partial_\psi \log q_\phi(z_j)|\right)^2 \right) 
\\ \leq \lambda^2 (1-\alpha)^2\frac{N^2}{(N-1)^4} \PE\left(|\log \w(z;x) \partial_\psi \log q_\phi(z)|^2 \right)
\end{multline*}
and thus the first term of the right-hand side of \eqref{eq:StudyBNgm} is $o(1/N)$. 

Finally, denoting $X = (1-\alpha) \log \w(z;x)$, $\overline{X}_N = (N-1)^{-1} \sum_{i=1}^N (1-\alpha) \log \w(z_i;x)$, $Y = \w(z;x)^{1-\alpha}$, $\overline{Y}_N = N^{-1} \sum_{i=1}^N \w(z_i;x)^{1-\alpha}$, $Z =  |\log \w(z,x) \partial_\psi \log q_\phi(z)|$ and $\overline{Z}_{N} = N^{-1} \sum_{i=1}^N |\log \w(z_i;x) \partial_\psi \log q_\phi(z_i)|$, \Cref{lem:VIMCOgm} implies that the second term on the right-hand side of \eqref{eq:StudyBNgm} is $o(1/N)$ and we have thus obtained that $B_N^{(\alpha)}(\phi;x) = o (1/N)$. Pairing up this result with \eqref{eq:VIMCOgmRemainderBound} and \eqref{eq:ANVIMCOgm} yields \eqref{eq:VIMCOgmRemainder}.

\item Proof of \eqref{eq:ToproveVIMCOgmOne}. Denoting $X = (1-\alpha) \log \w(z;x)$, $\overline{X}_N = (N-1)^{-1} \sum_{i=1}^N (1-\alpha) \log \w(z_i;x)$, $Y = \w(z;x)^{1-\alpha}$, $\overline{Y}_N = N^{-1} \sum_{i=1}^N \w(z_i;x)^{1-\alpha}$, $Z = \partial_\psi \log q_\phi(z)$ and $\overline{Z}_{N} = N^{-1} \sum_{i=1}^N \partial_\psi \log q_\phi(z_i)$, \eqref{eq:ToproveVIMCOgmOne} follows from \Cref{lem:VIMCOgm} paired up with the fact that $\PE(Z) = 0$ and thus $\boldsymbol{v}^T \boldsymbol{\Sigma} \boldsymbol{v} = (\exp(\PE(X))/\PE(Y))^2\mathbb{V}(Z)$.

\item Proof of \eqref{eq:ToproveThreeVIMCOam} and \eqref{eq:ToproveVIMCOgmTwo}. Applying \eqref{eq:prop:handleRNsnrVarAnalysis} from \Cref{prop:handleRN} yields \eqref{eq:ToproveThreeVIMCOam}. Now denote $Y = \w(z;x)^{1-\alpha}$, $Z = \partial_\psi \log q_\phi(z)$, $X = YZ$, $X' = \log Y$ as well as $\overline{Y}_N = N^{-1} \sum_{i=1}^N \w(z_i;x)^{1-\alpha}$, $\overline{Z}_N = N^{-1} \sum_{i=1}^N \partial_\psi \log q_\phi(z_i)$, $\overline{X}_N = N^{-1} \sum_{i=1}^N \w(z_i;x)^{1-\alpha}\partial_\psi \log q_\phi(z_i)$ and $\overline{X}'_N = N^{-1} \sum_{i=1}^N \log (\w(z_i;x)^{1-\alpha})$. Observe next that 
\begin{multline*}
  \mathbb{C}\mathrm{ov} \lr{ \frac{\overline{X}_N}{\overline{Y}_N} , \frac{\exp \lr{\frac{N}{N-1} \overline{X}'_N} }{\overline{Y}_N} \overline{Z}_N } = \PE \lr{{\overline{X}_N \overline{Z}_N}  \frac{\exp \lr{\frac{N}{N-1} \overline{X}'_N}}{\overline{Y}_N^2}} \\ - \PE \lr{ \frac{\overline{X}_N}{\overline{Y}_N}} \PE\lr{\frac{\exp \lr{\frac{N}{N-1} \overline{X}'_N} }{\overline{Y}_N} \overline{Z}_N }.
\end{multline*}
Notice also that 
\begin{align*}
 & \PE \lr{{\overline{X}_N \overline{Z}_N}  \frac{\exp \lr{\frac{N}{N-1} \overline{X}'_N}}{\overline{Y}_N^2}}  = \frac{1}{N} \frac{\exp(\PE(X'))}{\PE(Y)^2} \lr{\PE(XZ) - \frac{\PE(X)}{\PE(Y)} \PE(ZY)} \\ & \qquad  + \PE \lr{{(\overline{X}_N-\PE(X)) \overline{Z}_N}  \lr{\frac{\exp \lr{\frac{N}{N-1} \overline{X}'_N}}{\overline{Y}_N^2} -  \frac{\exp(\PE(X'))}{\PE(Y)^2}  }} \\
  & \qquad  + \frac{\PE(X)}{\PE(Y)} \PE \lr{{\overline{Z}_N}  \frac{\exp \lr{\frac{N}{N-1} \overline{X}'_N}}{\overline{Y}_N}} \\
  & \qquad - \frac{\PE(X)}{\PE(Y)} \PE \lr{{\overline{Z}_N(\overline{Y}_N - \PE(Y))}  \lr{\frac{\exp \lr{\frac{N}{N-1} \overline{X}'_N}}{\overline{Y}_N^2} - \frac{\exp(\PE(X'))}{\PE(Y)^2}}    },
\end{align*}
so that using that $X = ZY$ 
\begin{align*}
  \mathbb{C}\mathrm{ov} \lr{ \frac{\overline{X}_N}{\overline{Y}_N} , \frac{\exp \lr{\frac{N}{N-1} \overline{X}'_N} }{\overline{Y}_N} \overline{Z}_N } = \frac{1}{N} \frac{\exp(\PE(X'))}{\PE(Y)^2} \lr{\PE(XZ) - \frac{\PE(X)^2}{\PE(Y)}} + R_N,
\end{align*}
with 
\begin{multline*}
  R_N = \PE \lr{{\lrb{(\overline{X}_N-\PE(X))  - \frac{\PE(X)}{\PE(Y)} (\overline{Y}_N - \PE(Y))} \overline{Z}_N}  \lr{\frac{\exp \lr{\frac{N}{N-1} \overline{X}'_N}}{\overline{Y}_N^2} -  \frac{\exp(\PE(X'))}{\PE(Y)^2}  }} \\
   + \lrb{ \frac{\PE(X)}{\PE(Y)} - \PE\lr{\frac{\overline{X}_N}{\overline{Y}_N}}} \PE \lr{{\overline{Z}_N}  \frac{\exp \lr{\frac{N}{N-1} \overline{X}'_N}}{\overline{Y}_N}}.
\end{multline*}
Using \eqref{eq:prop:handleRNsnr1plus1N} from \Cref{prop:handleRN}, we have that 
\begin{align*}
\PE \lr{\frac{\overline{X}_N}{ \overline{Y}_N} } = \frac{\PE(X)}{\PE(Y)}+ \frac1N\lr{\frac{\PE(X)\,\mathbb{V}\lr{Y}}{\lr{\PE(Y)}^3}-\frac{\mathbb{C}\mathrm{ov}\lr{X,Y}}{\lr{\PE(Y)}^2}}+o\lr{\frac1N}
\end{align*}
while \Cref{lem:VIMCOgm} gives 
\begin{align*}
  \PE \lr{{\overline{Z}_N}  \frac{\exp \lr{\frac{N}{N-1} \overline{X}'_N}}{\overline{Y}_N^2}} = o(1),
\end{align*}
where we have used that $\PE(Z) = 0$. Consequently, the second term in the right-hand side of $R_N$ is $o(1/N)$. The proof of \eqref{eq:ToproveVIMCOgmTwo} will be concluded if we can show that the first term of right-hand side of $R_N$ is $o(1/N)$ too, since using our notation
\begin{align*}
  \mathbb{C}\mathrm{ov}(U,V) = \alpha \frac{\exp(\PE(X'))}{\PE(Y)^2} \lr{\PE(XZ) - \frac{\PE(X)^2}{\PE(Y)}}. 
\end{align*}
Writing that
\begin{align*}
  & \PE \lr{{\lrb{(\overline{X}_N-\PE(X))  - \frac{\PE(X)}{\PE(Y)} (\overline{Y}_N - \PE(Y))} \overline{Z}_N}  \lr{\frac{\exp \lr{\frac{N}{N-1} \overline{X}'_N}}{\overline{Y}_N^2} -  \frac{\exp(\PE(X'))}{\PE(Y)^2}  }} \\
  & = \PE \lr{{\lrb{(\overline{X}_N-\PE(X))  - \frac{\PE(X)}{\PE(Y)} (\overline{Y}_N - \PE(Y))} \overline{Z}_N} {\frac{\exp \lr{\frac{N}{N-1} \overline{X}'_N}-\exp(\PE(X'))}{\overline{Y}_N^2}}} \\
  & \qquad - \frac{\exp(\PE(X'))}{\PE(Y)^2} \PE \lr{{\lrb{(\overline{X}_N-\PE(X))  - \frac{\PE(X)}{\PE(Y)} (\overline{Y}_N - \PE(Y))} \overline{Z}_N}  \lr{\frac{\overline{Y}_N^2 - \PE(Y)^2}{\overline{Y}_N^2}}}, 
\end{align*} 
we will show that both terms in the right-hand side above are $o(1/N)$. For the second term of the right-hand side above, we write that
\begin{align*}
& \PE \lr{{\lrb{(\overline{X}_N-\PE(X))  - \frac{\PE(X)}{\PE(Y)} (\overline{Y}_N - \PE(Y))} \overline{Z}_N}  \lr{\frac{\overline{Y}_N^2 - \PE(Y)^2}{\overline{Y}_N^2}}} \\
& = \PE \lr{{\lrb{(\overline{X}_N-\PE(X))  - \frac{\PE(X)}{\PE(Y)} (\overline{Y}_N - \PE(Y))} \overline{Z}_N} {\frac{\overline{Y}_N - \PE(Y)}{\overline{Y}_N}}} \\
& \qquad + \PE(Y) \PE \lr{{\lrb{(\overline{X}_N-\PE(X))  - \frac{\PE(X)}{\PE(Y)} (\overline{Y}_N - \PE(Y))} \overline{Z}_N}  {\frac{\overline{Y}_N - \PE(Y)}{\overline{Y}_N^2}}},
\end{align*}
so this is $o(1/N)$ under our assumptions. 
For the first term in the right-hand side, \eqref{eq:exponBound} implies 
\begin{multline*}
\left| \exp \lr{\frac{N}{N-1} \overline{X}'_N}-\exp(\PE(X')) \right| \\ \leq \left[\exp \lr{\frac{N}{N-1} \overline{X}'_N} + \exp(\PE(X')) \right] \left| \frac{N}{N-1} \overline{X}'_N - \PE(X') \right|.
\end{multline*}
Writing that
\begin{align*}
  \exp \lr{\frac{N}{N-1} \overline{X}'_N} & =  \exp \lr{\overline{X}'_N}  \exp \lr{\frac{1}{N-1} \overline{X}'_N} \leq \exp \lr{\overline{X}'_N} \lr{ \exp \lr{\overline{X}'_N} +1} \\
  & \leq \overline{Y}_N (\overline{Y}_N + 1)
\end{align*}
we deduce that 
\begin{align*}
\left| \exp \lr{\frac{N}{N-1} \overline{X}'_N}-\exp(\PE(X')) \right|  \leq \left[\overline{Y}_N + \overline{Y}_N^2 + \exp(\PE(X')) \right] \left| \frac{N}{N-1} \overline{X}'_N - \PE(X') \right|.
\end{align*}
so the first term of the right-hand side is $o(1/N)$ under our assumptions. 

\end{enumerateList}
\end{proof}

\begin{proof}[Proof of \eqref{eq:VarGM}]
The proof of \eqref{eq:VarGM} is a direct application of \Cref{thm:genGM} with $\lambda = 1$.
\end{proof}

\subsection{Proof of \Cref{prop:hintCV}}

\label{app:prop:hintCV}

Using that $f_{-i}^{(\alpha)}(z_{1:N},\phi;x)$ is constant equal to $\eta\geq 0$ for all $i = 1 \ldots N$ in \eqref{eq:VIMCOgenVRIWAE} yields
\begin{multline*}
  \gradVIMCO = - \sum_{j = 1}^N \normW(z_j;x) \partial_\psi \log  q_\phi(z_j) \\
  -\frac{1}{1-\alpha}  \sum_{i=1}^N {\partial_\psi \log q_\phi(z_i)  \log \lr{1 -  \normW(z_i;x) + \frac{\eta}{\sum_{j = 1}^N \w(z_j;x)^{1-\alpha}}}}
\end{multline*}
which we can rewrite as 
\begin{multline}
  \gradVIMCO = \frac{1}{1-\alpha} \sum_{j = 1}^N \frac{\alpha~\w(z_j;x)^{1-\alpha} - \eta}{\sum_{\ell = 1}^N \w(z_\ell;x)^{1-\alpha}} \partial_\psi \log q_\phi(z_j) \\
  - \frac{1}{1-\alpha}  \sum_{i=1}^N \partial_\psi \log q_\phi(z_i) \lrb{ \log \lr{1 - \frac{\w(z_i;x)^{1-\alpha} - \eta}{\sum_{j = 1}^N \w(z_j;x)^{1-\alpha}}} + \frac{\w(z_i;x)^{1-\alpha} - \eta}{\sum_{j = 1}^N \w(z_j;x)^{1-\alpha}}}. \label{eq:gradVIMCOcv}
\end{multline}
By \eqref{eq:prop:handleRNsnrVarAnalysis} from \Cref{prop:handleRN}, we have that 
\begin{align*}
\mathbb{V}\lr{\frac{1}{1-\alpha} \sum_{j = 1}^N \frac{\alpha~\w(z_j;x)^{1-\alpha} - \eta}{\sum_{\ell = 1}^N \w(z_\ell;x)^{1-\alpha}} \partial_\psi \log q_\phi(z_j)} = \frac{\vVIMCOgen(\phi;x)}{N} + o \lr{\frac{1}{N}}
\end{align*}
with $\vVIMCOgen(\phi;x)$ defined as in \Cref{prop:hintCV} and using that $\PE(\partial_\psi \log q_\phi(z)) = 0$. The proof of the asymptotic variance will be concluded if we can show that the variance of the second term in the r.h.s. of \eqref{eq:gradVIMCOcv} is $o(1/N)$. Denoting 
\begin{align*}
  B_N^{(\alpha)} = \mathbb{V} \lr{  \sum_{i=1}^N \partial_\psi \log q_\phi(z_i) \lrb{ \log \lr{1 - \frac{\w(z_i;x)^{1-\alpha} - \eta}{\sum_{j = 1}^N \w(z_j;x)^{1-\alpha}}} + \frac{\w(z_i;x)^{1-\alpha} - \eta}{\sum_{j = 1}^N \w(z_j;x)^{1-\alpha}}}},
\end{align*}
we first write that 
\begin{align*}
   | B_N^{(\alpha)}| & \leq \PE \lr{ \lr{  \sum_{i=1}^N \partial_\psi \log q_\phi(z_i) \lrb{ \log \lr{1 - \frac{\w(z_i;x)^{1-\alpha} - \eta}{\sum_{j = 1}^N \w(z_j;x)^{1-\alpha}}} + \frac{\w(z_i;x)^{1-\alpha} - \eta}{\sum_{j = 1}^N \w(z_j;x)^{1-\alpha}}}}^2 } \\
   & \leq \frac{1}{4}  \PE \lr{ \lr{  \sum_{i=1}^N \lr{\frac{\w(z_i;x)^{1-\alpha} - \eta}{\sum_{j = 1}^N \w(z_j;x)^{1-\alpha}}}^2 |\partial_\psi \log q_\phi(z_i)| \lr{1+ \frac{1}{1- \frac{\w(z_i;x)^{1-\alpha} - \eta}{\sum_{j = 1}^N \w(z_j;x)^{1-\alpha}}}}}^2 } \\
   & \leq \frac{1}{4}  \PE \lr{ \lr{  \sum_{i=1}^N \lr{\frac{\w(z_i;x)^{1-\alpha} - \eta}{\sum_{j = 1}^N \w(z_j;x)^{1-\alpha}}}^2 |\partial_\psi \log q_\phi(z_i)| \lr{1+ \frac{1}{1- \normW(z_i;x)}}}^2 }
\end{align*}
where we have used successively (i) \eqref{eq:LogNthOrderTwo} with $r = 1 -  (\w(z_i;x)^{1-\alpha}-\eta)/{\sum_{j = 1}^N \w(z_j;x)^{1-\alpha} }$ and (ii) $\w(z_i;x)^{1-\alpha} - \eta \leq \w(z_i;x)^{1-\alpha}$ since $\eta \geq 0$. Hence, 
\begin{multline*}
   | B_N^{(\alpha)}| \leq \PE \lr{ \lr{  \sum_{i=1}^N \lr{\frac{\w(z_i;x)^{1-\alpha} - \eta}{\sum_{j = 1}^N \w(z_j;x)^{1-\alpha}}}^2 |\partial_\psi \log q_\phi(z_i)|}^2 } \\
   + \PE \lr{ \lr{  \sum_{i=1}^N \lr{\frac{\w(z_i;x)^{1-\alpha} - \eta}{\sum_{j = 1}^N \w(z_j;x)^{1-\alpha}}}^2 |\partial_\psi \log q_\phi(z_i)| \lr{1+ \frac{1}{1- \normW(z_i;x)   }}}^2 }
\end{multline*}
Applying \eqref{eq:prop:handleRNsnr1} from \Cref{prop:handleRN} (paired up with \eqref{eq:prop:handleRNsnrVarAnalysis} from \Cref{prop:handleRN}) shows that the first term in the right-hand side above is $o(1/N)$, while \Cref{lem:ControlCNAlpha} with $\ell = 1$ shows that the second term in the right-hand side above is $o\lr{1/N}$. The proof of the asymptotic variance is thus concluded.

It follows from the expression of $\vVIMCOgen(\phi;x)$ that this quantity is minimized for 
$$
  \eta_\psi^{(\alpha, \star)} = \frac{\mathbb{C}\mathrm{ov} \lr{\alpha \w(z;x)^{1-\alpha} \biggl[ \partial_\psi \log q_\phi(z) - \PE\lr{\frac{\w(z;x)^{1-\alpha}}{\PE(\w(z;x)^{1-\alpha})} \partial_\psi \log q_\phi(z)}    \biggr]   , \partial_\psi \log q_\phi(z)}  }{\mathbb{V}(\partial_\psi \log q_\phi(z))}.
$$
Next observe that 
\begin{align*}
    \eta_\psi^{(\alpha, \star)} & = \alpha \PE(\w(z;x)^{1-\alpha}) \frac{\PE\lr{\frac{\w(z;x)^{1-\alpha}}{\PE(\w(z;x)^{1-\alpha})} (\partial_\psi \log q_\phi(z))^2 } - \PE\lr{\frac{\w(z;x)^{1-\alpha}}{\PE(\w(z;x)^{1-\alpha})} \partial_\psi \log q_\phi(z) }^2}{\mathbb{V}(\partial_\psi \log q_\phi(z))} \\
    & = \alpha \PE(\w(z;x)^{1-\alpha}) \frac{\mathbb{V}(\partial_\psi \log q_\phi(\tilde{z}_\alpha))}{\mathbb{V}(\partial_\psi \log q_\phi(z))}
\end{align*}
where $\tilde{z}_\alpha \sim \tilde{q}_{\phi,\alpha}$ with $\tilde{q}_{\phi,\alpha} \propto q_\phi \w(\cdot;x)^{1-\alpha}$

\subsection{Proof of \Cref{thm:optimalVarReduction}}
\label{app:subsec:thm:optimalVarReduction}

Since $f_{-i}^{(\alpha)}(z_{1:N},\phi;x) = 0$, we first write that
\begin{align*}
   \gradVIMCO[N][\psi][0] = \frac{1}{2} \sum_{j = 1}^N \lr{ \normW[\theta][\phi][1](z_j;x)}^2 \partial_\psi \log q_\phi(z_j) + R_{2,N}^{(0)}(\phi;x),
\end{align*}
where 
\begin{multline}
 R_{2,N}^{(0)}(\phi;x) =  - \frac{1}{1-\alpha}  \sum_{i=1}^N \partial_\psi \log q_\phi(z_i) \left[ \log \lr{1 - \normW[\theta][\phi][1](z_i;x)} + \normW[\theta][\phi][1](z_i;x) \right. \\ \left. + \frac{1}{2} (\normW[\theta][\phi][1](z_i;x))^2 \right]. \label{eq:reusezv}
\end{multline}
 Under our assumptions, we can apply \eqref{eq:prop:handleRNsnrVarAnalysis} from \Cref{prop:handleRN} and we have that
\begin{align}
  \mathbb{V}\lr{\sum_{j = 1}^N \lr{\normW[\theta][\phi][1](z_j;x)}^2 \frac{1}{2} \partial_\psi \log q_\phi(z_j)}  = \frac{\vVIMCOstar[0](\phi;x)}{N^3} + o\lr{\frac{1}{N^3}}. \label{eq:reusezvbis}
\end{align}
Denoting 
\begin{align*}
B_{N}^{(\alpha)}: = \mathbb{V} \left(\sum_{i=1}^N \partial_\psi \log q_\phi(z_i) \left[\log \lr{1 - \normW[\theta][\phi][1](z_i;x)}  \right. \right. + \normW[\theta][\phi][1](z_i;x)
\left. \left. + \frac12\lr{\normW[\theta][\phi][1](z_i;x)}^2\right] \right),
\end{align*}
we will get the desired result if we can show that $B_{N}^{(\alpha)} = o \lr{1/{N^3}}$ with $\alpha = 0$. Since
\begin{align*}
  B_{N}^{(\alpha)} \leq \PE \left( \left(\sum_{i=1}^N \partial_\psi \log q_\phi(z_i) \left[\log \lr{1 - \normW[\theta][\phi][1](z_i;x)} + \normW[\theta][\phi][1](z_i;x) \right. \right. \right.
\left. \left. \left. + \frac{1}{2} \lr{\normW[\theta][\phi][1](z_i;x)}^2\right] \right)^2 \right),
\end{align*}
setting $r = 1 -  \normW[\theta][\phi][1](z_i;x)$, we have that $r> 0$ and \eqref{eq:LogNthOrderTwo} from \Cref{lem:LogTaylorBound} with $k = 2$ implies 
\begin{align*}
  |B_{N}^{(\alpha)}| \leq \frac{1}{9} \PE \lr{ \lr{\sum_{i=1}^N \lr{\normW[\theta][\phi][1](z_i;x)}^3 |\partial_\psi \log q_\phi(z_i)| \lrb{1+ \frac{1}{1 - \normW[\theta][\phi][1](z_i;x)}} }^2  }.
\end{align*}
As a result,  
\begin{multline}
    |B_{N}^{(\alpha)}| \leq \PE \lr{ \lr{\sum_{i=1}^N \lr{\normW[\theta][\phi][1](z_i;x)}^3 |\partial_\psi \log q_\phi(z_i)|}^2} \\ + \PE\lr{ \lr{\sum_{i=1}^N \lr{\normW[\theta][\phi][1](z_i;x)}^3 |\partial_\psi \log q_\phi(z_i)| \frac{1}{1 - \normW[\theta][\phi][1](z_i;x)} }^2  }. \label{eq:interREINF1overNCubeAlpha}
\end{multline}
Denoting $X= w(z;x)^{3 } |\partial_\psi \log q_\phi(z)|$, $\overline{X}_N = 1/N \sum_{i=1}^N \w(z_i;x)^{3 } |\partial_\psi \log q_\phi(z_i)|$, $Y = \w(z;x) $ and $\overline{Y}_N = 1/N \sum_{i=1}^N \w(z_i;x) $, the first term in the right-hand side of \eqref{eq:interREINF1overNCubeAlpha} becomes 
\begin{align*}
\frac{1}{N^4} \PE \lr{\lr{\frac{\overline{X}_N}{\overline{Y}_N^3}}^2} 
\end{align*}
Now write that
\begin{align*}
\PE\lr{\lr{\frac{\overline X_N}{\overline Y_N^3}}^2} = {\PE(X)^2} { \mathbb{E}\lr{\frac{1}{\overline Y_N^6}}} + 2\PE(X)\PE \lr{\frac{\overline X_N-\PE(X)}{\overline Y_N^6}} + \PE \lr{\frac{(\overline X_N-\PE(X))^2}{\overline Y_N^6}}.
\end{align*}
Under our assumptions and applying \Cref{lem:petrov}, there exists $p>1$ such that 
\begin{align}
& \lim_{N \to \infty }\PE(|\overline{X}_N-\PE(X)|^{p}) = 0 \\
&  \lim_{N \to \infty }\PE(|\overline{X}_N-\PE(X)|^{2p}) = 0
\end{align}
and for all $\mu >0$ : as $N \to \infty$, $\PE((\overline{Y}_N)^{-\mu}) = O(1)$ by \Cref{lem:equivlimsupconditionGen}. Hence, 
\begin{align*}
  \PE\lr{\lr{\frac{\overline X_N}{\overline Y_N^3}}^2} = O(1) 
\end{align*}
so that the first term in the right-hand side of \eqref{eq:interREINF1overNCubeAlpha} is $o(1/N^3)$. As for the second term in the right-hand side of \eqref{eq:interREINF1overNCubeAlpha} it is $o(1/N^3)$ by applying \Cref{lem:ControlCNAlpha} with $(\ell,\eta) = (2,0)$. The proof is then concluded.

\subsection{Asymptotic $\sqrt{N}$ SNR scaling for VIMCO-zvAM($\alpha$)}
\label{app:thmzv}

\begin{thm} \label{thm:sAM} Assume~\ref{hyp:inverseW}, \ref{hyp:momentW} and \ref{hyp:momentGradTwo} with $h > 8$, $h' > 4$. Then, as $N \to \infty$, \looseness=-1
\begin{align} \label{eq:VarsAM}
 & \mathbb{V}\lr{\gradVIMCOzvam} = \frac{1}{N} \vVIMCOgen(\phi;x) + o \lr{\frac{1}{N}}, 
\end{align}
with $\eta = \alpha \PE\lr{\w(z;x)^{1-\alpha}}$. Further assume that $\alpha = 0$ and $h>12$. Then, as $N \to \infty$,
\begin{align} \label{eq:VarsAM0}
 & \mathbb{V}\lr{\gradVIMCOzvam[N][\psi][0]} = \frac{1}{N^3} \vVIMCOgzvAM[0](\phi;x) + o \lr{\frac{1}{N^3}} 
 \end{align}
 \begin{multline*} \mbox{where} \quad \vVIMCOgzvAM[0](\phi;x) \\ = \mathbb{V}\lr{\frac{ \w(z;x)^{2} \partial_\psi \log q_\phi(z)}{2 \PE(\w(z;x))^2}  - \w(z;x) \frac{\PE(\w(z;x)^{2} \partial_\psi \log q_\phi(z))}{\PE(\w(z;x))^3} - \frac{1}{2} \partial_\psi \log q_\phi(z)}.
 \end{multline*}
\end{thm}

\begin{proof}
We first write that for all $\alpha \in [0,1)$,
\begin{align*}
& \gradVIMCOzvam \\ & = \frac{\alpha}{1-\alpha} \sum_{j = 1}^N \normW(z_j;x) \partial_\psi \log q_\phi(z_j) - \frac{1}{1-\alpha} \log\lr{1 + \frac{c_{N}^{(\mathrm{zv}, \alpha)}}{N-1}} \sum_{i=1}^N \partial_\psi \log q_\phi(z_i) \\ & \quad - \frac{1}{1-\alpha} \sum_{i=1}^N \partial_\psi \log q_\phi(z_i) \lrb{\log \lr{1 - \normW(z_i;x)} + \normW(z_i;x)},
\end{align*}
and we observe that
\begin{multline}
\mathbb{V}\left(\frac{1}{1-\alpha} \log\lr{1 + \frac{c_{N}^{(\mathrm{zv}, \alpha)}}{N-1}} \sum_{i=1}^N \partial_\psi \log q_\phi(z_i) \right) \\ = \frac{N}{(1-\alpha)^2} \lrb{\frac{\alpha}{N} + \frac{1}{2N^2} + o\lr{\frac{1}{N^2}}}^2 \mathbb{V}\lr{\partial_\psi \log q_\phi(z)} . \label{eq:devLim}
\end{multline}
The desired result \eqref{eq:VarsAM} follows by combining the above with \Cref{lem:ToProveTwoVIMCOam}, which yields \looseness=-1
\begin{align}
\mathbb{V}\lr{\sum_{i=1}^N \partial_\psi \log q_\phi(z_i) \lrb{\log \lr{1 - \normW(z_i;x)} + \normW(z_i;x)}} = o \lr{\frac{1}{N}},
\end{align}
and \Cref{prop:ToProveTwoVIMCOsam}, which yields 
\begin{multline*}
  \mathbb{V}\lr{\sum_{j = 1}^N \normW(z_j;x) \partial_\psi \log q_\phi(z_j)} \\ = \frac{1}{N} \mathbb{V} \lr{\frac{\w(z;x)^{1-\alpha}}{\PE(\w(z;x)^{1-\alpha})} \lrb{ \partial_\psi \log q_\phi(z)  - \frac{\PE(\w(z;x)^{1-\alpha} \partial_\psi \log q_\phi(z))}{\PE(\w(z;x)^{1-\alpha})} } } + o \lr{\frac{1}{N}}
\end{multline*}
as well as
\begin{multline*}
\log\lr{1 + \frac{c_{N}^{(\mathrm{zv}, \alpha)}}{N-1}} \mathbb{C}\mathrm{ov} \lr{\sum_{j = 1}^N \normW(z_j;x) \partial_\psi \log q_\phi(z_j)   ,    \sum_{i=1}^N \partial_\psi \log q_\phi(z_i) } \\
 =  \frac{\alpha}{N} \mathbb{C}\mathrm{ov}  \left(\frac{\w(z;x)^{1-\alpha}}{\PE(\w(z;x)^{1-\alpha})} \biggl[ \partial_\psi \log q_\phi(z)  \right. \\ \left. - \frac{\PE(\w(z;x)^{1-\alpha} \partial_\psi \log q_\phi(z))}{\PE(\w(z;x)^{1-\alpha})}    \biggr], \partial_\psi \log q_\phi(z) \right) + o\lr{\frac{1}{N}} .
\end{multline*}
As for the specific case $\alpha = 0$, we write
\begin{align*}
   \gradVIMCOzvam[N][\psi][0] & =  - \log\lr{1 + \frac{c_{N}^{(\mathrm{zv}, 0)}}{N-1}} \sum_{i=1}^N \partial_\psi \log q_\phi(z_i) \\ & + \frac{1}{2} \sum_{j = 1}^N \lr{ \normW[\theta][\phi][1](z_j;x)}^2 \partial_\psi \log q_\phi(z_j) + R_{2,N}^{(0)}(\phi;x),
\end{align*}
where $R_{2,N}^{(0)}(\phi;x)$ is defined as in \eqref{eq:reusezv}. Now using \eqref{eq:reusezvbis}, the fact that $\mathbb{V}(R_{2,N}^{(0)}(\phi;x)) = o(1/N^3)$ by the proof of \Cref{thm:optimalVarReduction} and \eqref{eq:devLim} with $\alpha = 0$, the proof will be concluded if we can show that 
\begin{multline}
\mathbb{C}\mathrm{ov} \lr{\sum_{j = 1}^N \left(\normW[\theta][\phi][1](z_j;x)\right)^2 \partial_\psi \log q_\phi(z_j)   ,    \sum_{i=1}^N \partial_\psi \log q_\phi(z_i) } \\
= \frac{1}{N} \mathbb{C}\mathrm{ov}\lr{ \frac{ X}{\PE(Y)^2}\right. \left.  - 2Y \frac{\PE(X)}{\PE(Y)^3} , Z } + o \lr{\frac{1}{N}}. \label{eq:FinalToProvezvAM}
\end{multline}
where we have denoted $X = \w(z;x)^2 \partial_\psi \log q_\phi(z)$, $Y = \w(z;x)$, $Z = \partial_\psi \log q_\phi(z)$. Further denote
\begin{align*}
 & \overline{X}_N = N^{-1} \sum_{j = 1}^N \w(z_j;x)^2 \partial_\psi \log q_\phi(z_j) \\
 & \overline{Y}_N = N^{-1} \sum_{j = 1}^N \w(z_j;x) \\
 & \overline{Z}_N = N^{-1} \sum_{j = 1}^N \partial_\psi \log q_\phi(z_j).
\end{align*}
Since $\PE(Z) = 0$, we have that 
\begin{align*}
  \mathbb{C}\mathrm{ov} \lr{\frac{\overline{X}_N}{\overline{Y}_N^2}, \overline{Z}_N} = \PE\lr{\frac{\overline{X}_N \overline{Z}_N}{(\overline{Y}_N)^2}}.
\end{align*}
Now write that 
\begin{align*}
  \frac{1}{\overline{Y}_N^2} = \frac{1}{\PE(Y)^2} - \frac{2(\overline{Y}_N - \PE(Y))}{\PE(Y)^3} + \frac{(\overline{Y}_N -\PE(Y))^2(3 \PE(Y) + 2(\overline{Y}_N - \PE(Y)))}{\PE(Y)^3 \overline{Y}_N^2}
\end{align*}
so that 
\begin{align*}
\PE\lr{\frac{\overline{X}_N \overline{Z}_N}{(\overline{Y}_N)^2}} & = \frac{1}{N} \lrb{ \frac{\mathbb{C}\mathrm{ov}(X,Z)}{\PE(Y)^2} - \frac{2 \PE(X) \mathbb{C}\mathrm{ov}(Y,Z)}{\PE(Y)^3}} - \frac{2\PE \lr{(\overline{X}_N - \PE(X))(\overline{Y}_N - \PE(Y))\overline{Z}_N}}{\PE(Y)^3} \\
& \quad + \frac{1}{\PE(Y)^3} \PE \lr{\frac{\overline{X}_N \overline{Z}_N \lrb{(\overline{Y}_N -\PE(Y))^2(3 \PE(Y) + 2(\overline{Y}_N - \PE(Y)))}}{\overline{Y}_N^2}}.
\end{align*}
Under our assumptions, and using the Cauchy-Schwarz inequality twice, we then have that 
\begin{align*}
  \PE \lr{(\overline{X}_N - \PE(X))(\overline{Y}_N - \PE(Y))\overline{Z}_N}
 \leq \lr{\frac{\mathbb{V}(X)}{N}}^{\frac{1}{2}} \PE(\overline{Z}_N^4)^\frac{1}{4} \PE \lr{{(\overline{Y}_N- \PE(Y))^4}{}}^{\frac{1}{4}},
\end{align*}
which is $o(1/N)$ using \eqref{eq:petrov-centered} from \Cref{lem:petrov} twice. Furthermore, 
\begin{align*}
  \PE \lr{\frac{(\overline{X}_N - \PE(X))\overline{Z}_N(\overline{Y}_N - \PE(Y))^2}{\overline{Y}_N^2}} \leq \lr{\frac{\mathbb{V}(X)}{N}}^{\frac{1}{2}} \PE(\overline{Z}_N^4)^\frac{1}{4} \PE \lr{ \frac{(\overline{Y}_N- \PE(Y))^8}{\overline{Y}_N^8}}^{\frac{1}{4}},
\end{align*}
which is $o(1/N)$ using \eqref{eq:petrov-centered} from \Cref{lem:petrov} and \Cref{lem:equivlimsupconditionGen} (as in the proof of \eqref{eq:ToProveOneVIMCOamRewritten}). Similarly, 
\begin{align*}
&  \PE \lr{\frac{(\overline{X}_N - \PE(X))\overline{Z}_N(\overline{Y}_N - \PE(Y))^3}{\overline{Y}_N^2}} = o \lr{\frac{1}{N}} \\
&   \PE \lr{\frac{\overline{Z}_N \lrb{(\overline{Y}_N -\PE(Y))^2(3 \PE(Y) + 2(\overline{Y}_N - \PE(Y)))}}{\overline{Y}_N^2}}
\end{align*}
so that \eqref{eq:FinalToProvezvAM} holds and the proof is concluded.
\end{proof}

\section{Practical considerations}

\label{app:sec:practical:considerations}

Consider a statistical model with joint distribution $p_\theta(x, z)$ parameterized by $\theta = (\theta_1, \ldots, \theta_a) \in \Theta \subseteq \rset^a$, where $x$ is an observation and $z$ is a latent variable valued in a latent space $\mathsf{Z}$. Assume that sampling from the posterior density $p_\theta(z|x)$ is impossible and that the goal is to find the optimal $\theta \in \Theta$ that maximizes the marginal log likelihood $\ell(\theta;x)$. Letting $\psi$ be a one-dimensional component of $\{ \theta_1, \ldots, \theta_a, \phi_1, \ldots, \phi_b \}$, one can straightforwardly adapt the VIMCO gradient estimator framework introduced in \eqref{eq:VIMCOgenVRIWAE} to learn the parameters $(\theta, \phi)$ via SGA by noticing that the dependency in $\theta$ is embedded in $\w(z;x) = {p_\theta(x,z)}/{q_\phi(z)}$.

The choice of $(N,\alpha)$ then controls two distinct mechanisms captured by \Cref{thm:REINFORCEGradientExpectation,thm:optimalVarReduction,prop:hintCV,thm:VarianceReduction} as well as Theorems 1 and 2 from \cite{daudel2024learningimportanceweightedvariational}. Firstly, increasing $N$ reduces the bias of the VR-IWAE bound relative to the marginal log-likelihood and the variance of its gradient estimator \eqref{eq:VIMCOgenVRIWAE}. While these gains can be accompanied by a decay of the SNR as $N$ increases, this effect is not inherent to $N$ itself (nor $\alpha$): selecting $f_{-i}^{(\alpha)}$ appropriately leads to SNRs that scale as $\sqrt{N}$. Secondly, the parameter $\alpha$ governs a bias-variance/signal trade-off. Although small values of $\alpha$, and in particular $\alpha = 0$, are preferable from a bias standpoint, they may lead to a large variance/low signal in the early stages of the optimization procedure when the variational approximation is inaccurate. \looseness=-1

A natural strategy is therefore to (i) select $f_{-i}^{(\alpha)}$ in the gradient estimator \eqref{eq:VIMCOgenVRIWAE} such that its SNR scales at a fast $\sqrt{N}$ rate and (ii) initialize the resulting algorithm to account for a potentially high mismatch between the posterior density and its variational approximation. We elaborate on the aspects (i) and (ii) below.

\subsection{Selecting $f_{-i}^{(\alpha)}$ in the VIMCO gradient estimator \eqref{eq:VIMCOgenVRIWAE}}
 The VIMCO-zv$\star$($\alpha$) estimator is designed so that its per-sample baseline converges a.s. to the optimal deterministic per-sample baseline  $\eta_\psi^{(\alpha, \star)}$ and it also has zero variance at optimality. We now address the issue of the computational cost to evaluate this estimator. VIMCO-zv$\star$($\alpha$) uses a $z$-samples of size $N$, so that a direct way to bring down the computational cost is to use a subsample of size $N_0< N$. While this is not the focus of our paper, future work could seek to lower the computational budget used to estimate $\eta_\psi^{(\alpha, \star)}$ while also retaining the desirable theoretical properties of the VIMCO-zv$\star$($\alpha$) gradient estimator. Another option to reduce the computational cost when $\alpha \in (0,1)$, which is the one we employ, is to use VIMCO-zvAM($\alpha$). This gradient estimator is cheap to compute, although its per-sample baseline is not designed to converge a.s. to $\eta_\psi^{(\alpha, \star)}$. Note that, in the spirit of control variates, $\eta_\psi^{(\alpha, \star)}$ could also be done using the $z$-sample generated at the previous iteration of the VI optimization procedure. \looseness=-1

\subsection{Annealing schedule}\label{subsec:Annealing_schedule_for_alpha}

A significant mismatch between the posterior and the variational approximation often exists in the early stages of optimization. Annealing strategies help mitigate the noise in VIMCO gradient estimators. 
Our analysis identifies $\alpha$ as a critical lever for managing the bias-variance/signal trade-off in importance-weighted VI. A high value of $\alpha$ helps stabilize the weights $\{\w(z_j;x)^{1-\alpha},j=1...N\}$ and increase the SNR of the VIMCO gradient estimator, whereas a small value of $\alpha$ achieves a tighter lower bound. To balance these competing demands while always targeting the true posterior density, we can use a dynamic annealing schedule for selecting $\alpha$, where $\alpha$ plays the role of tempering parameter as in the annealed importance sampling literature \citep{neal2001annealed}. One way to do so is to initialize $\alpha$ near $1$ and to adaptively reduce it based on the effective sample size \looseness=-1
\begin{equation}
\text{ESS}_{1-\alpha} = \frac{\lr{\sum_{j=1}^N \w(z_j;x)^{1-\alpha}}^2}{\sum_{j=1}^N \w(z_j;x)^{2(1-\alpha)}}=\frac{1}{\sum_{j=1}^N\normW(z_j;x)^2},
\end{equation}
as $\text{ESS}_{1-\alpha}$ permits to measure the efficiency of the weights $\{\w(z_j;x)^{1-\alpha},j=1...N\}$. In practice, we initialize $\alpha$ at either $0.99$ or $0.9$ in the VIMCO($\alpha$) gradient estimators, and progressively reduce it whenever $\text{ESS}_{1-\alpha}$ exceeds a prescribed threshold, such as $0.5N$ \citep[see also][]{finke2019importance}. This annealing approach is expected to accelerate training and to achieve the tightest VR-IWAE bound; reaching $\alpha=0$ by the end of the procedure indicates that the variational approximation has become sufficiently accurate. Another simpler possibility is to use a geometric interpolation for $\alpha$ \citep{lievinNEURIPS2020}. \cite{lievinNEURIPS2020} indeed further extended their high-ESS OVIS estimator by adapting their control variate construction to encompass weight tempering \citep[Appendix E]{lievinNEURIPS2020}. We note that their generalized high-ESS version of OVIS does not enjoy zero variance at posterior optimality: using our notation, they consider the following gradient estimator
\begin{multline*}
  \psi\text{-}g_N^{(\alpha, \mathrm{OVIS})}(\phi;x) = \sum_{j = 1}^N \normW(z_j;x) \partial_\psi \log  \w(z_j;x) \\
  -\frac{1}{1-\alpha}  \sum_{i=1}^N {\partial_\psi \log q_{\phi}(z_i)  \log \lr{1 - \normW(z_i;x) + \frac{f_{-i}^{(\alpha)}(z_{1:N},\phi;x)}{\sum_{j = 1}^N \w(z_j;x)^{1-\alpha}}}}
\end{multline*}
with $f_{-i}^{(\alpha)}(z_{1:N},\phi;x) = \lrb{\lr{\frac{N}{N-1}}^\alpha -1}\sum_{j = 1, j \neq i}^N \w(z_j;x)^{1-\alpha}$, so that, at posterior optimality (that is, when $q_\phi = p(\cdot|x)$): for all $\alpha \in [0,1)$, 
\begin{align*}
\mathbb{V} \lr{    \psi\text{-}g_N^{(\alpha, \mathrm{OVIS})}(\phi;x)} &= \frac{1}{N} \lrb{1 + N \log\lr{1- \frac{1}{N}}}^2 \mathbb{V}(\partial_\psi \log  q_\phi(z)).
\end{align*}

Developing theoretically grounded annealing schedules for $\alpha$ \cite[analogous to those established in the Sequential Monte Carlo literature, see, e.g.][]{del2012adaptive} is a compelling direction of research in importance weighted VI, which is left for future work.


 \looseness=-1

 \section{Limiting case $\alpha \to 1$ in zv gradient estimators}
\label{app:secAlphaTo1}

As $\alpha \to 1$, we have that 
\begin{align}
  & \w(z;x)^{1-\alpha} = 1 + (1-\alpha) \log  \w(z;x) + o((1-\alpha)), \nonumber \\
  & c_{N}^{(\mathrm{zv}, \alpha)} = 1 - (1-\alpha) + o ((1-\alpha)), \nonumber \\
  & - \frac{1}{1-\alpha} \log\lrb{1 - (1-\alpha)\frac{\beta}{N} + o((1-\alpha))} = \frac{\beta}{N} + o(1). \label{eq:betaApprox}
\end{align}
Hence, \eqref{eq:fOptimOne} and \eqref{eq:fCheap} read respectively
\begin{align}
 & f_{-i}^{(\alpha)}(z_{1:N},\phi;x) = 1 + (1-\alpha) \lrb{ \frac{b_{-i,1}^{(1)}(z_{1:N}, \phi; x)}{b_{-i,0}^{(1)}(z_{1:N}, \phi; x)} - 1} + o((1-\alpha)) \\
& f_{-i}^{(\alpha)}(z_{1:N},\phi;x)  = 1 + (1-\alpha) \lrb{\frac{1}{N-1} \sum_{j=1, j \neq i}^N \log  \w(z_j;x)-1} + o((1-\alpha))
\end{align}
where we have set: for all $k \in \{0,1\}$,
$$
b_{-i,k}^{(1)}(z_{1:N}, \phi; x) = {\frac{1}{N-1} \sum_{j=1, j \neq i}^N (\log \w(z_j;x))^k \lr{\partial_\psi \log q_\phi(z_j) - \frac{1}{N-1} \sum_{j=1, j \neq i}^N  \partial_\psi \log q_\phi(z_j) }^2}
$$
Plugging the above into \eqref{eq:VIMCOgenVRIWAE} and using \eqref{eq:betaApprox} with
\begin{align*}
& \beta = \log \w(z_i;x) - \frac{b_{-i,1}^{(1)}}{b_{-i,0}^{(1)}} + 1 \\
& \beta = \log \w(z_i;x) - \frac{1}{N-1} \sum_{j=1, j \neq i}^N \log  \w(z_j;x) + 1 
\end{align*}
yields the limiting results as $\alpha \to 1$ for the zero-variance gradient estimators VIMCO-zv$\star$($\alpha$) and VIMCO-zvAM($\alpha$) respectively and we define the extension of these gradient estimators as $\alpha \to 1$ by 
\begin{multline*}
\gradVIMCOstar[N][\psi][1] := \frac{1}{N} \sum_{i=1}^N \partial_\psi \log \w(z_i;x) \\
 + \frac{1}{N} \sum_{i=1}^N  \partial_\psi \log q_\phi(z_i;x) \lrb{ \log  \w(z_i;x) - \frac{b_{-i,1}^{(1)}(z_{1:N}, \phi; x)}{b_{-i,0}^{(1)}(z_{1:N}, \phi; x)} + 1}
\end{multline*}
and
\begin{multline*}
\gradVIMCOzvam[N][\psi][1] := \frac{1}{N} \sum_{i=1}^N \partial_\psi \log \w(z_i;x) \\
+ \frac{1}{N} \sum_{i=1}^N  \partial_\psi \log q_\phi(z_i;x) \lrb{ \log  \w(z_i;x) - \frac{1}{N-1} \sum_{j=1, j \neq i}^N \log  \w(z_j;x) + 1}.
\end{multline*}
 Since $\psi \in \{ \phi_1, \ldots, \phi_b \}$, these expressions further simplify to \begin{align*}
& \gradVIMCOstar[N][\psi][1] := \frac{1}{N} \sum_{i=1}^N  \partial_\psi \log q_\phi(z_i;x) \lrb{ \log  \w(z_i;x) - \frac{b_{-i,1}^{(1)}(z_{1:N}, \phi; x)}{b_{-i,0}^{(1)}(z_{1:N}, \phi; x)}} \\
& \gradVIMCOzvam[N][\psi][1] := \frac{1}{N} \sum_{i=1}^N  \partial_\psi \log q_\phi(z_i;x) \lrb{ \log  \w(z_i;x) - \frac{1}{N-1} \sum_{j=1, j \neq i}^N \log  \w(z_j;x)}. 
\end{align*}
Observe in particular that the ratio ${b_{-i,1}^{(1)}(z_
{1:N}, \phi; x)}/{b_{-i,0}^{(1)}(z_{1:N}, \phi; x)}$ converges a.s. to the well-known optimal control variate baseline of the ELBO gradient estimator identified in \cite{ranganath2014black}. Moreover, VIMCO-zvAM($1$) corresponds to the REINFORCE leave-one-out gradient estimator for the ELBO from \cite{salimans2014using,kool2019buy,NEURIPS20209c22c0b5}.


\section{Deferred proofs and empirical results for \Cref{sec:NumEx}}

\subsection{Gaussian example}
\label{app:subsub:ex:Gaussian}
 
\subsubsection{Closed-form expressions for the quantities of interest}

We first recall some known results taken from \cite{daudel2024learningimportanceweightedvariational}.

\begin{ex}[Known results from \cite{daudel2024learningimportanceweightedvariational}] \label{ex:GaussPrev} Let $\theta, \phi \in \rset^d$. Set $p(z|x) = \mathcal{N}(z;\theta, \boldsymbol{I}_d)$ and $q_\phi(z) = \mathcal{N}(z; \phi, \boldsymbol{I}_d)$, where $\boldsymbol{I}_d$ is the $d$-dimensional identity matrix. Denoting the Euclidean norm of a finite dimensional vector $x$ with real entries by $\lrn{x}$ and its associated inner product by $\langle \cdot,\cdot \rangle$. Then: for all $k = 1 \ldots d$, 
$\partial_{\phi_k} \mathrm{VR}^{(\alpha)}(\phi; x) = - \alpha(\phi_k -\theta_k)$ and $ \partial_{\phi_k} [\gammaA(\phi; x)^2] = 2 {(1-\alpha)(\phi_k- \theta_k) \exp \lr{(1-\alpha)^2 \|\phi-\theta\|^2}}$.
In addition, it holds that
  \begin{align}
    \log \lr{\frac{p(z|x)}{q_\phi(z)}} =
    -\frac{\|\theta-\phi\|^2}{2} -\|\theta-\phi \| S, \quad S =
    \frac{\langle z - \phi, \phi- \theta \rangle}{\|\theta-\phi \|},
    \quad z \in \rset^d. \label{eq:logWGaussianExample}
  \end{align}
\end{ex}
By \eqref{eq:logWGaussianExample}, we then have that
  \begin{align}
    \log \lr{\frac{\w(z;x)}{\PE(\w(z;x))}} =
    -\frac{\|\theta-\phi\|^2}{2} -\|\theta-\phi \| S, \quad S =
    \frac{\langle z - \phi, \phi- \theta \rangle}{\|\theta-\phi \|},
    \quad z \in \rset^d \label{eq:logWGaussianExampleBis}
  \end{align}
from which we deduce: for all $\alpha \in \rset$, 
\begin{align}
  \log \lr{ \frac{\w(z;x)^{1-\alpha}}{\PE(\w(z;x)^{1-\alpha})}  } = -(1-\alpha) \langle z - \phi, \phi- \theta \rangle - \frac{1}{2} (1-\alpha)^2 \|\theta-\phi \|^2. \label{eq:logWnorm}
\end{align}
For all $k = 1 \ldots d$, we now let 
$z^{(k)}$ denote the $k$-th element of $z \in \rset^d$ so that \begin{align}
&{ \frac{\w(z;x)^{1-\alpha}}{\PE(\w(z;x)^{1-\alpha})}} \Biggl[ \partial_{\phi_k} \log q_\phi(z)
  - \frac{\PE(\w(z;x)^{1-\alpha} \partial_{\phi_k} \log q_\phi(z)   )}{\PE(\w(z;x)^{1-\alpha})} \Biggr] \nonumber \\
  & \quad = - \lrb{\phi_k - z^{(k)} + (\alpha-1) (\phi_k - \theta_k)} \frac{\w(z;x)^{1-\alpha}}{\PE(\w(z;x)^{1-\alpha})},  \label{eq:usefulrewritinggradW}
\end{align}
which follows from \eqref{eq:logWnorm} paired up with the fact that for $u \sim \mathcal{N}(0,1)$ and all $t \in \rset$,
$\PE(u e^{u t}) = t e^{t^2/2}$ as well as
\begin{align}
  \partial_{\phi_k} \log q_\phi(z) = z^{(k)} - \phi_k.\label{eq:partial:phiGaussian}
\end{align}
Furthermore, since $\langle z - \phi, \phi- \theta \rangle$ is a Gaussian variable if $z \sim q_\phi$, it admits exponential moments of any exponent and we have that \ref{hyp:inverseW} holds. We further get that $\log \w(z;x) = c_0(\theta,\phi) - \langle z - \phi, \phi- \theta \rangle$ for some continuously differentiable function $c_0$ defined on $\Theta\times\Phi$. Conditions \ref{hyp:momentW} and \ref{hyp:momentGradTwo} follow using that $z \sim q_\phi$ with $h \geq 8$ and $h' \geq 4$. We now compute $\vVIMCOgen[\alpha][\phi_k](\phi;x)$ for all $\eta > 0$. 
First note that \eqref{eq:partial:phiGaussian} implies $\mathbb{V} \lr{\partial_{\phi_k} \log q_\phi(z)} = 1$ and next we will show that
\begin{multline}
  \mathbb{V} \left(\frac{\w(z;x)^{1-\alpha}}{\PE(\w(z;x)^{1-\alpha})} \left[ \partial_{\phi_k} \log q_\phi(z) - \frac{\PE(\w(z;x)^{1-\alpha} \partial_{\phi_k} \log q_\phi(z))}{\PE(\w(z;x)^{1-\alpha})} \right] \right) \\
   = e^{(1-\alpha)^2 \| \theta - \phi \|^2} \lr{1 + \lr{1-\alpha}^2 (\phi_k -\theta_k)^2} \label{eq:VarOneGaussExample}
\end{multline}
and 
\begin{align}
\PE \left(\frac{\w(z;x)^{1-\alpha}}{\PE(\w(z;x)^{1-\alpha})} \left[ \partial_{\phi_k} \log q_\phi(z) - \frac{\PE(\w(z;x)^{1-\alpha} \partial_{\phi_k} \log q_\phi(z))}{\PE(\w(z;x)^{1-\alpha})} \right] {\partial_{\phi_k} \log q_\phi(z)} \right) 
= 1. \label{eq:CovarGaussExample}
\end{align}

\begin{enumerateList}
  \item Proof of \eqref{eq:VarOneGaussExample}. Using \eqref{eq:usefulrewritinggradW}, we get
\begin{multline*}
B := \mathbb{V} \left(\frac{\w(z;x)^{1-\alpha}}{\PE(\w(z;x)^{1-\alpha})} \left[ \partial_{\phi_k} \log q_\phi(z) - \frac{\PE(\w(z;x)^{1-\alpha} \partial_{\phi_k} \log q_\phi(z))}{\PE(\w(z;x)^{1-\alpha})} \right] \right) \\ = \PE \lr{ \left(\lrb{\phi_k - z^{(k)} + (\alpha-1)(\phi_k - \theta_k) }  \frac{\w(z;x)^{1-\alpha}}{\PE(\w(z;x)^{1-\alpha})}  \right)^2}.
\end{multline*}
As a result, 
\begin{align*} 
B & = e^{-(1-\alpha)^2 \| \theta - \phi \|^2} \PE \left(\lrb{z^{(k)} -\phi_k  -  (\alpha-1)(\phi_k - \theta_k) }^2  e^{-2 (1-\alpha) \langle z - \phi, \phi- \theta \rangle}  \right) \\
& = e^{-(1-\alpha)^2 \| \theta - \phi \|^2} e^{2(1-\alpha)^2 \| \theta - \phi \|^2} \lr{1 + \lr{1-\alpha}^2 (\phi_k -\theta_k)^2}
\end{align*}
where we have used \eqref{eq:logWnorm} and that for $u \sim \mathcal{N}(0,1)$ and all $t \in \rset$,
\begin{align*}
\PE (e^{t u}) = e^{t^2/2}, \quad  \PE(u e^{u t}) = t e^{t^2/2} \quad \mbox{and} \quad \PE(u^2 e^{u t}) = (1+t^2) e^{t^2/2},
\end{align*}
and which concludes the proof of \eqref{eq:VarOneGaussExample}. 
\item Proof of \eqref{eq:CovarGaussExample}. Using \eqref{eq:usefulrewritinggradW} and \eqref{eq:partial:phiGaussian}, we get that
\begin{multline*}
C := \PE \left(\frac{\w(z;x)^{1-\alpha}}{\PE(\w(z;x)^{1-\alpha})} \left[ \partial_{\phi_k} \log q_\phi(z) - \frac{\PE(\w(z;x)^{1-\alpha} \partial_{\phi_k} \log q_\phi(z))}{\PE(\w(z;x)^{1-\alpha})} \right] {\partial_{\phi_k} \log q_\phi(z)} \right) \\ 
= \PE \lr{(\phi_k - z^{(k)} + (\alpha-1) (\phi_k - \theta_k))(\phi_k - z^{(k)}) \frac{\w(z;x)^{1-\alpha}}{\PE(\w(z;x)^{1-\alpha})}}.
\end{multline*}
Hence 
\begin{align*}
  C & = e^{-\frac{1}{2} (1-\alpha)^2 \|\theta-\phi\|^2} \PE \lr{\lr{( z^{(k)} - \phi_k)^2 - (\alpha-1) (\phi_k - \theta_k)(z^{(k)}-\phi_k)} e^{-(1-\alpha) \langle z-\phi, \phi-\theta \rangle}} \\
  & = 1
\end{align*}
which is exactly \eqref{eq:CovarGaussExample}.
\end{enumerateList}
Using \eqref{eq:VarOneGaussExample}, \eqref{eq:CovarGaussExample} and $\mathbb{V} \lr{\partial_{\phi_k} \log q_\phi(z)} = 1$, we thus get: for all $\eta \geq 0$,
\begin{multline}
  \vVIMCOgen[\alpha][\phi_k](\phi;x) = \frac{\alpha^2}{(1-\alpha)^2} e^{(1-\alpha)^2 \| \theta - \phi \|^2}  \lr{1 + \lr{1-\alpha}^2 (\phi_k -\theta_k)^2} + \\
  \frac{1}{(1-\alpha)^2} \frac{\eta}{\PE(\w(z;x)^{1-\alpha})} \lr{-2\alpha  + \frac{\eta}{\PE(\w(z;x)^{1-\alpha})}}\label{eq:Agamma}
\end{multline}
Furthermore, we have that using \eqref{eq:logWGaussianExampleBis} that
  \begin{align}
\frac{\exp[(1-\alpha)\PE(\log\w(z;x))]}{\PE(\w(z;x)^{1-\alpha})} & = \frac{\exp\lrb{(1-\alpha)\PE\lr{\log\frac{\w(z;x)}{\PE(\w(z;x))}}}}{\PE\lr{\lr{\frac{\w(z;x)}{\PE(\w(z;x))}}^{1-\alpha}}} \nonumber \\
  & = \frac{\exp\lrb{-\frac{1}{2} (1-\alpha) \| \theta-\phi\|^2}}{\exp\lrb{-\frac{1}{2} (1-\alpha) \| \theta-\phi\|^2} \exp \lrb{\frac{1}{2} (1-\alpha)^2 \| \theta-\phi\|^2}} \nonumber \\
  & = \frac{1}{\exp \lrb{\frac{1}{2} (1-\alpha)^2 \| \theta-\phi\|^2}} \label{eq:factorGM}
  \end{align} 
and we can deduce 
\begin{itemize}[wide=0pt, labelindent=\parindent]
  \item Setting $\alpha = 0$ and $\eta = \PE(\w(z;x))$ in \eqref{eq:Agamma} yields: as $N \to \infty$,
   \begin{align*}
 & \mathbb{V}\lr{\gradVIMCOam} = \frac{1}{N} + o \lr{\frac{1}{N}}. 
\end{align*}
  \item  Setting $\alpha = 0$ and $\eta = \exp[\PE(\log\w(z;x))]$ in \eqref{eq:Agamma} and using \eqref{eq:factorGM} yields: as $N \to \infty$,
   \begin{align*}
   \mathbb{V} \lr{\gradVIMCOgm} = \frac{\exp \lrb{- \| \theta-\phi\|^2}}{N}  + o \lr{\frac{1}{N}}. 
\end{align*}
\item Setting $\eta = \eta^{(\alpha,\star)}$, we have
  \begin{align*}
    \frac{\eta^{(\alpha,\star)}}{\PE(\w(z;x)^{1-\alpha})} & = \alpha \frac{\PE \left(\frac{\w(z;x)^{1-\alpha}}{\PE(\w(z;x)^{1-\alpha})} \left[ \partial_{\phi_k} \log q_\phi(z) - \frac{\PE(\w(z;x)^{1-\alpha} \partial_{\phi_k} \log q_\phi(z))}{\PE(\w(z;x)^{1-\alpha})} \right] {\partial_{\phi_k} \log q_\phi(z)} \right)}{\mathbb{V} \lr{\partial_{\phi_k} \log q_\phi(z)}} \\
    & = \alpha 
  \end{align*}
 where we have used \eqref{eq:CovarGaussExample} and $\mathbb{V} \lr{\partial_{\phi_k} \log q_\phi(z)} = 1$.  Plugging this in \eqref{eq:Agamma} yields $\vVIMCOgen$.
\item Setting $\eta = \alpha \PE(\w(z;x)^{1-\alpha})$ and $\eta = \alpha \exp[(1-\alpha)\PE(\log\w(z;x))]$ yield respectively 
\begin{align*}
      \frac{\eta}{\PE(\w(z;x)^{1-\alpha})} = \alpha \quad \mbox{and} \quad     \frac{\eta}{\PE(\w(z;x)^{1-\alpha})} = {\exp \lrb{-\frac{1}{2} (1-\alpha)^2 \| \theta-\phi\|^2}} \alpha
\end{align*}
where we have also used \eqref{eq:factorGM}. 
\end{itemize}
To conclude, it remains to compute $\vVIMCOstar[0]$. Denoting 
\begin{align*}
  A(\phi;x) =  \frac{ \w(z;x)^2 \partial_{\phi_k} \log q_\phi(z)}{2 \PE(\w(z;x))^2}
\end{align*}
it follows from \eqref{eq:partial:phiGaussian} that 
\begin{align*} 
\vVIMCOstar[0][\phi_k](\phi;x) 
 = \PE\lr{ \left(A(\phi;x) - 2\frac{\w(z;x)}{\PE(\w(z;x))} \PE(A(\phi;x)) \right)^2} - \PE\lr{A(\phi;x)}^2
\end{align*}
where $A(\phi;x) = \frac{1}{2} (z^{(k)} - \phi_k) \lr{ \frac{\w(z;x)}{\PE(\w(z;x))}}^2$.
Consequently
\begin{align*} 
& \vVIMCOstar[0][\phi_k](\phi;x)  = \frac{1}{4} \PE \lr{(z^{(k)} - \phi_k)^2 \lr{ \frac{\w(z;x)}{\PE(\w(z;x))}}^4} \\
& \qquad - \PE \lr{(z^{(k)} - \phi_k) \lr{ \frac{\w(z;x)}{\PE(\w(z;x))}}^3} \PE \lr{(z^{(k)} - \phi_k) \lr{ \frac{\w(z;x)}{\PE(\w(z;x))}}^2} \\
& \qquad + \lrb{\PE \lr{\lr{ \frac{\w(z;x)}{\PE(\w(z;x))}}^2} -\frac{1}{4} } \PE \lr{(z^{(k)} - \phi_k) \lr{ \frac{\w(z;x)}{\PE(\w(z;x))}}^2}^2
\end{align*}
so that using \eqref{eq:logWnorm} with $\alpha = 0$
\begin{multline*}
 \vVIMCOstar[0][\phi_k](\phi;x)  = \lr{\frac{1}{4}+4(\phi_k-\theta_k)^2} e^{6\|\phi-\theta \|^2} - 6(\phi_k-\theta_k)^2 e^{4 \| \theta-\phi\|^2} \\
+ \lrb{e^{ \| \theta-\phi\|^2} -\frac{1}{4}} 4(\phi_k - \theta_k)^2 e^{2 \| \theta-\phi\|^2}
\end{multline*}
and we also deduce
\begin{align*}
  \vVIMCOgzvAM[0][\phi_k](\phi;x) =  \vVIMCOstar[0][\phi_k](\phi;x) - \frac{1}{2} e^{\| \theta-\phi\|^2} + \frac{1}{4}.
\end{align*}

%

\subsubsection{Gaussian example near optimality}
\begin{figure}[t]
\begin{center}
    \includegraphics[width=\textwidth]{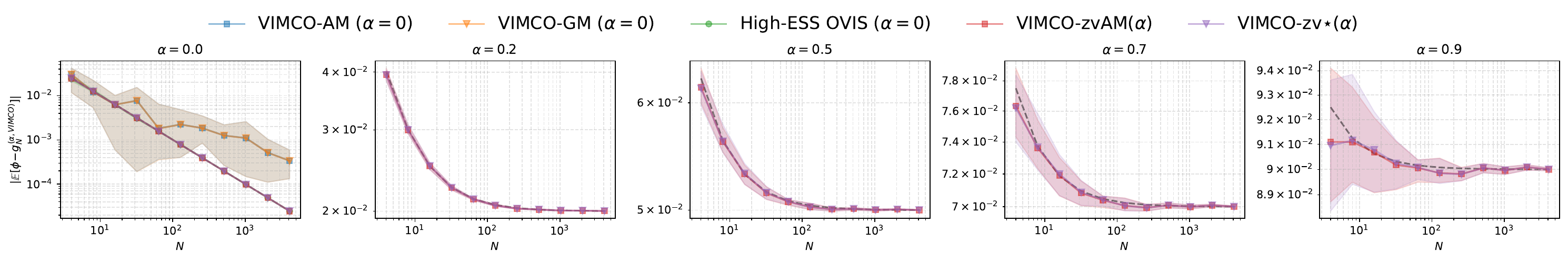}
  \includegraphics[width=\textwidth]{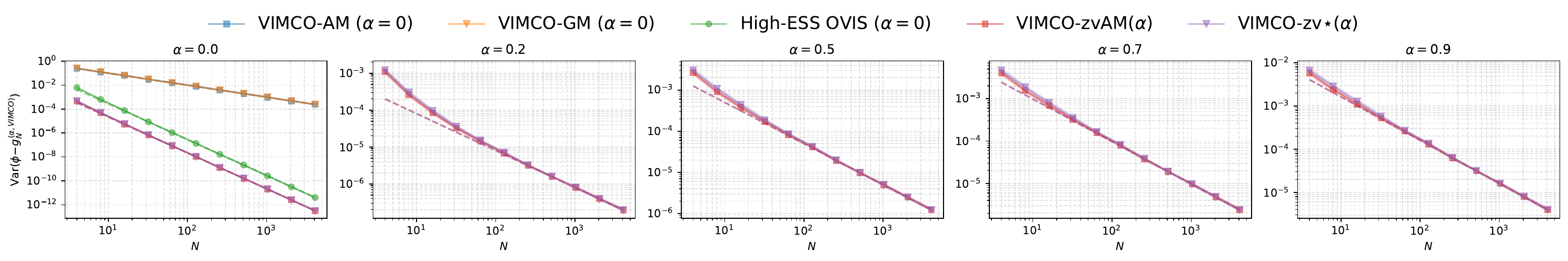}
    \caption{Expectation and variance as a function of $N$ for different values of $\alpha$ when $\phi=0.1$ in the Gaussian example. The solid lines correspond to the absolute mean and variance estimates for different VIMCO gradient estimators, while the dashed lines correspond to the theoretical predictions. The shaded area represents the 5\% to 95\% percentile range across 10 independent runs.}
    \label{fig:Gaussian_closeMeanVar}
\end{center}
\end{figure}

\Cref{fig:Gaussian_closeMeanVar} indicates that VIMCO-AM and VIMCO-GM fail to yield reliable gradient estimates near optimality, whereas High-ESS OVIS, VIMCO-zvAM($0$) and VIMCO-zv$\star$($0$) produce stable and accurate estimates that closely match the theoretical prediction of a $1/N$ convergence rate. As for the variances, VIMCO-AM and VIMCO-GM both exhibit a $1/N$ rate, while High-ESS OVIS, VIMCO-zvAM($0$) and VIMCO-zv$\star$($0$) achieve the faster rate of $1/N^3$. Consequently, relative to High-ESS OVIS, VIMCO-zvAM($0$) and VIMCO-zv$\star$($0$), the gradients obtained from VIMCO-AM and VIMCO-GM are substantially noisier when the variational approximation is close to the target distribution.


\subsection{Sigmoid Belief Network (SBN)}

\label{app:SBN}

\Cref{fig:SBNperf10} plots the marginal log-likelihood and ELBO estimates throughout training for all the methods considered in the SBN example when $N = 10$. Similarly to the case $N = 50$, larger values initially facilitate optimization and lead to higher ELBO values, but at the cost of optimizing a looser importance-weighted objective, so that large values of $\alpha$ may not yield the best marginal log-likelihood.

\begin{figure}
\begin{center}
\includegraphics[width=\textwidth]{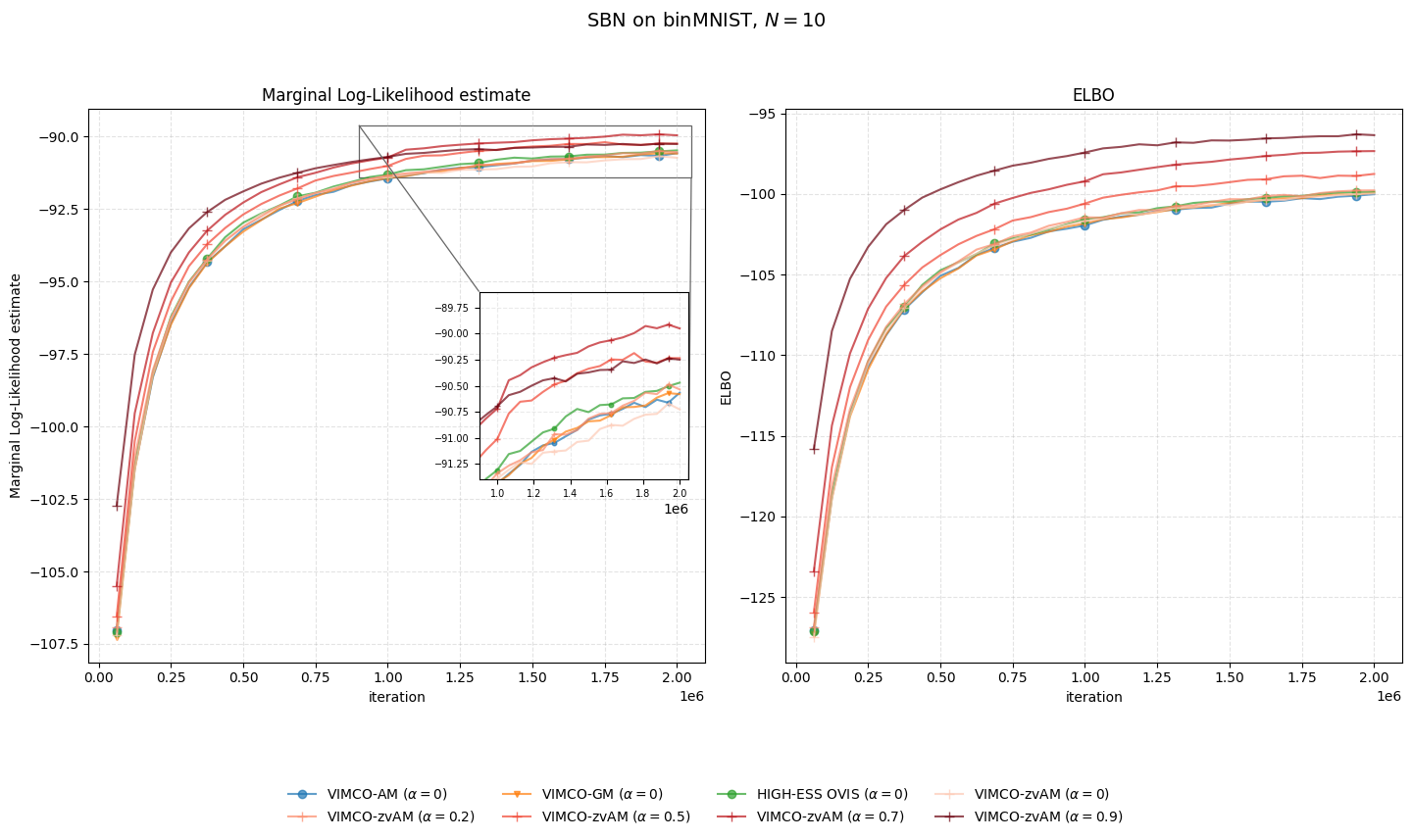}
\end{center}
\caption{Performance of VIMCO gradient estimators in the SBN example averaged over 3 independent runs and expressed in terms of marginal log-likelihood (left) and ELBO (right).}
\label{fig:SBNperf10}
\end{figure}

\subsection{Variational Bayesian Inference for State-Space models}
\label{app:VBISSM}
 The training procedures for the ELBO and VR-IWAE use the same algorithmic settings including the stopping criterion with fixed learning rate searched over a grid, and differ only in the choice of the bound objective and gradient estimator. For these VI methods, the likelihood is estimated using a particle filter with 300 particles.

The pseudo-marginal MCMC method is based on the random-walk Metropolis-Hastings algorithm, with the likelihood estimated unbiasedly using a particle filter with 500 particles. 
As pseudo-marginal MCMC is sensitive to the variance of the likelihood estimate, it is important to use a larger number of particles compared to the VI methods. 
The MCMC algorithm is run for 100,000 iterations. Figure \ref{fig:PMCMC traceplot} presents the trace plots of the resulting Markov chain, which indicate that the chain has converged.

\begin{figure}
    \centering
    \includegraphics[width = \textwidth]{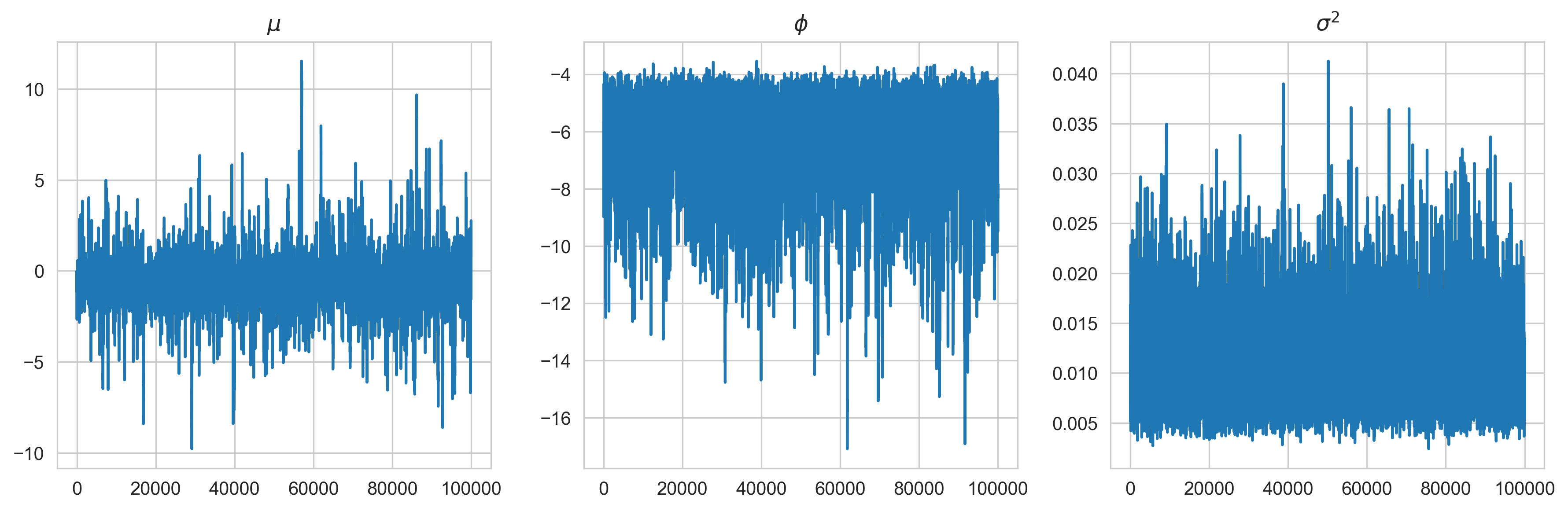}
    \caption{Trace plot of MCMC}
    \label{fig:PMCMC traceplot}
\end{figure}

%
%
%
%

\section{Technical results}

We recall below as lemmas several existing technical results that are used in our proofs.

  \begin{lem}[Lemma 2 of \cite{daudel2024learningimportanceweightedvariational}] \label{lem:additional-tech-lemma-bis}
    Let $\eta_0< \mu_0$ be two real exponents. Then, for all
    $a,b\geq0$, we have
    $$
    \lrav{a-b}^{\mu_0}\leq
    \lrav{a-b}^{\eta_0}\;\lr{a^{{\mu_0}-\eta_0}+ b^{{\mu_0}-\eta_0}}\;.
    $$
  \end{lem}

    \begin{lem}[Proposition 5 of \cite{daudel2024learningimportanceweightedvariational}]
    \label{lem:equivlimsupconditionGen}
    Set $\overline{Z}_N = N^{-1} \sum_{i=1}^N Z_i$ for all
    $N \in \mathbb{N}^\star$, where $Z_1, \ldots, Z_N$ are positive
    i.i.d. random variables. Then the following assertions hold.
    \begin{enumerate}[label=(\roman*)]
  \item\label{item:lem:equivlimsupconditionGen2} For all $\mu>0$, we have
  \begin{align} \label{ass:limsupZequiv}
    \limsup_{N\to\infty  } \PE\lr{(1/\overline{Z}_N)^\mu} <
    \infty\Longleftrightarrow\exists N\geq1\;,\,\PE((1/\overline{Z}_{N})^\mu) < \infty.
  \end{align}
  Furthermore, if the assertions of the equivalence~(\ref{ass:limsupZequiv}) hold for some $\mu>0$ and the distribution of $Z_1$ does not reduce to a Dirac measure, there exists $N_0\geq1$ such that $\PE\lr{\lr{1/\overline{Z}_N}^\mu}=\infty$ for $1\leq N < N_0$ and the sequence $\lr{\PE\lr{\lr{1/\overline{Z}_N}^\mu}}_{N \geq N_0}$ is strictly decreasing in $(0,\infty)$. 

  \item \label{item:lem:equivlimsupconditionGen1} For all $\eta>0$, we have
  \begin{align}  \label{eq:simple-cond-NGrad}
    \sup_{t>0} \lr{t^\eta\; \PE\lr{\rme^{-t\,Z_1}}}<\infty \Longleftrightarrow \sup_{u>0}\lr{u^{-\eta}\;\PP\lr{Z_1\leq u} }<\infty.
  \end{align}

\item\label{item:lem:equivlimsupconditionGen4} 
  Suppose that the assertions of the equivalence \eqref{eq:simple-cond-NGrad} hold for some $\eta>0$. Then the assertions of the equivalence~(\ref{ass:limsupZequiv}) hold
  for all $\mu>0$.
\item\label{item:lem:equivlimsupconditionGen3} Suppose that the
  assertions of the equivalence~(\ref{ass:limsupZequiv}) hold for some
  $\mu>1$. Then the assertions of the equivalence \eqref{eq:simple-cond-NGrad} hold for some $\eta>0$.
\item \label{item:lem:equivlimsupconditionGen6} Suppose 
that the assertions of the equivalence \eqref{eq:simple-cond-NGrad} hold for some $\eta>0$.  
Then, for all $\mu>0$, we have
  \begin{align} \label{ass:limZnegpower}
    \lim_{N\to\infty  } \PE\lr{(1/\overline{Z}_N)^\mu} =
    \begin{cases}
      0 &\text{ if $\PE(Z_1)=\infty$,}\\
    \lr{\PE(Z_1)}^{-\mu} &\text{ otherwise.}
    \end{cases}
  \end{align}
\end{enumerate}
  \end{lem}

\begin{lem}[Proposition 6 (i), (iv) \& (vi) of \cite{daudel2024learningimportanceweightedvariational}] \label{prop:handleRN} Let $\lr{X,Y}$ be an
  $\mathbb{R}\times\mathbb{R}_+^*$-valued random vector. Let
  $\lr{X_1,Y_1}, \lr{X_2,Y_2},$ $\ldots$ be a sequence of i.i.d. random
  vectors with the same distributions as $\lr{X,Y}$.  Set
  $\overline{X}_N = N^{-1} \sum_{i=1}^N X_i$ and
  $\overline{Y}_N = N^{-1} \sum_{i=1}^N Y_i$ for all
  $N \in \mathbb{N}^\star$. Assume that there exist $N\geq1$ and
  $\mu>0$ such that $\PE((\overline{Y}_{N})^{-\mu})<\infty$ and
  that $\PE(Y^{h})$ and $\PE(\lrav{X}^{\tilde{h}})$ are finite for some
  positive real exponents $h$ and $\tilde{h}$.  Then the following assertions
  hold.
  \begin{enumerate}[label=(\roman*)]
    \item\label{item:eq:prop:handleRNsnr1} If $h\geq1$ and $\tilde{h}>1$, we   have, as $N\to\infty$, 
  \begin{align}    \label{eq:prop:handleRNsnr1}
      &    \PE\lr{\frac{\overline{X}_{N}}{(\overline{Y}_{N})^\ell}} = \frac{\PE(X)}{(\PE(Y))^\ell}+ o\lr{1}, \quad \ell \in \{1, 2\}.
    \end{align}
      \item\label{item:eq:prop:handleRN1} If $h>2$ and  $\tilde{h}>1$ with $2/h+1/\tilde{h}<1$, we have, as $N\to\infty$,
    \begin{align}
      \label{eq:prop:handleRN1}
      &    \PE\lr{\frac{(\overline{Y}_N - \PE(Y))^2\;\lrav{\overline{X}_N - \PE(X)}}{\overline{Y}_N}} = o\lr{\frac{1}{N}}\;.
    \end{align}
  Furthermore, if $h,\tilde{h}\geq2$ with  $1/h+1/\tilde{h}<1$, then (\ref{eq:prop:handleRN1}) holds again as $N\to\infty$.
  \item\label{item:eq:prop:handleRNsnr1plus1N} Suppose that $h>2$ and
    $\tilde{h}>1$. If $\tilde{h}<2$, suppose moreover that  $2/h+1/\tilde{h}<1$. Then, we have, as $N\to\infty$,  
    \begin{align}
      \label{eq:prop:handleRNsnr1plus1N}
      &    \PE\lr{\frac{\overline{X}_{N}}{\overline{Y}_{N}}}
        = \frac{\PE(X)}{\PE(Y)}+ \frac1N\lr{\frac{\PE(X)\,\mathbb{V}\lr{Y}}{\lr{\PE(Y)}^3}-\frac{\mathbb{C}\mathrm{ov}\lr{X,Y}}{\lr{\PE(Y)}^2}}+o\lr{\frac1N}\;.
    \end{align}
    \item\label{item:eq:prop:handleRNsnrVarAnalysis}  If $h,\tilde{h}>2$, we have, as $N\to\infty$,  
  \begin{align}
    \label{eq:prop:handleRNsnrVarAnalysis}
    \mathbb{V}\lr{\frac{\overline{X}_N}{(\overline{Y}_{N})^\ell}} =
    \frac{1}{N} \frac{\mathbb{V}\lr{X - \ell Y
    \frac{\PE(X)}{\PE(Y)}}}{\lr{\PE(Y)}^{2 \ell}} +o\lr{\frac1N}, \quad \ell \in \{1, 2\}.
  \end{align}
  \end{enumerate}
\end{lem}

  \begin{lem}[Claim (66) in Lemma 3 of \cite{daudel2024learningimportanceweightedvariational}] \label{lem:petrov}
    Let $Z$ be a real valued random variable, let $Z_1, \ldots,
    Z_N$ be i.i.d. copies of $Z$ and denote  $\overline{Z}_N = N^{-1} \sum_{i=1}^N
    Z_i$ for all $N \in \mathbb{N}^\star$. Then, for all real
    $p\geq1$, 
    \begin{align} \label{eq:petrov-centered}
&\PE(|Z|^{2p}) <
    \infty \Longrightarrow     \sup_{N \in \mathbb{N}^\star} \lr{N^{p}\; \PE\lr{\lrav{\overline{Z}_N-\PE\lr{Z}}^{2p}}} < \infty. 
    \end{align}
\end{lem}

\printbibliography

\end{document}